\documentclass[twoside, 11pt]{article}

\usepackage[preprint]{jmlr2e}

\newcommand{\bibpath}{bib_cont_submodular}

\usepackage[usenames]{xcolor}
\definecolor{cite_color}{RGB}{0, 0, 255}
\definecolor{link_color}{RGB}{153, 0,0}  %
\definecolor{url_color}{RGB}{153, 102,  0}
\definecolor{emp_color}{RGB}{0,0,255}
\graphicspath{{./figs/}}

\usepackage{epsfig}
\usepackage{amsmath}
\usepackage{amssymb}
\usepackage{mathrsfs}
\usepackage{graphicx}
\usepackage{subfig}
\usepackage{multirow}
\usepackage{booktabs}
\usepackage{wrapfig}
\usepackage{enumerate}
\usepackage{bm}
\usepackage{tabularx}

\usepackage{natbib}
\bibpunct{(}{)}{;}{a}{,}{,}

\usepackage{mathtools}

\usepackage[vlined, ruled, linesnumbered]{algorithm2e}

\SetCommentSty{mycommfont}
\SetKwInOut{Input}{input}
\SetKwInOut{Output}{output}

\allowdisplaybreaks

\usepackage{thmtools}
\usepackage{thm-restate}

\let \oldtextcircled \textcircled
\renewcommand{\textcircled}[1]{\oldtextcircled{\footnotesize #1}}

\usepackage{enumitem}
\setlist[itemize]{leftmargin=9mm}

\newcommand{\appendixtitle}[1]{
	\begin{center}
		\LARGE \bf #1
	\end{center}
}

\def\E{{\mathbb E}}
\def\X{{\cal X}}

\def\M{{\cal M}}

\def\optcont{\ensuremath{\x^*}}

\def \c{\mathbf{c}}
\def \v{\mathbf{v}}

\def \r{\mathbf{r}}
\def \a{\mathbf{a}}
\def \b{\mathbf{b}}
\def \d{\mathbf{d}}
\def \x{\mathbf{x}}
\def \y{\mathbf{y}}
\def \s{\mathbf{s}}

\def \u{\mathbf{u}}

\def \z{\mathbf{z}}
\def \h{\mathbf{h}}
\def \bh{\mathbf{h}}
\def \bu{\mathbf{u}}
\def \p{\mathbf{p}}
\def \q{\mathbf{q}}

\def \BA{\mathbf{A}}

\def \BI{\mathbf{I}}
\def \BC{\mathbf{C}}
\def \BD{\mathbf{D}}

\def \BH{\mathbf{H}}
\def \BL{\mathbf{L}}

\def \BR{\mathbf{R}}

\def \BU{\mathbf{U}}

\def \BW{\mathbf{W}}

\def \bmA{\mathbf{A}}
\def \bmH{\mathbf{H}}
\def \bmL{\mathbf{L}}

\def \bmI{\mathbf{I}}

\def \bmtheta{\bm{\theta}}

\def \P{{\cal{P}}}
\def \Q{{\cal{Q}}}

\def \R{{\mathbb{R}}}
\def \trans{\top}
\newcommand{\tr}[1]{\text{tr}(#1)}
\newcommand{\pare}[1]{{#1}}  %

\def \D{{\cal{D}}}

\def \E{{\mathbb{E}}} %

\newcommand{\maxcut}{{\textsc{MaxCut}}}

\newcommand{\argmin}{{\arg\min}}
\newcommand{\diag}{{\text{diag}}}

\newcommand{\argmax}{{\arg\max}}
\newcommand{\algname}[1]{{\texttt{#1}}}

\newcommand{\spt}[1]{{\texttt{supp}}(#1)}
\newcommand{\dtp}[2]{\langle #1, #2\rangle}%
\newcommand{\de}[1]{\text{det}\left(#1\right)}%
\newcommand{\fracpartial}[2]{\frac{\partial #1}{\partial #2}}
\newcommand{\fracppartial}[3]{\frac{\partial^2 #1}{\partial #2 \partial #3}}

\newcommand{\bas}{\mathbf{e}} %
\def \chara{\mathbf{e}} %

\def \cg{L}
\def\A{{\mathscr{A}}}

\newcommand{\groundset}{\ensuremath{\mathcal{V}}}

\newcommand{\NEWDR}{\texttt{weak DR}}

\newcommand{\sete}[3]{\mathbf #1|_{#2} (#3)}

\usepackage{cleveref} %
 
\crefname{section}{Section}{Sections}
\crefname{theorem}{Theorem}{Theorems}
\crefname{lemma}{Lemma}{Lemmas}
\crefname{equation}{Equation}{Equations}
\crefname{proposition}{Proposition}{Propositions}
\Crefname{proposition}{Proposition}{Propositions}
\crefname{claim}{Claim}{Claims}
\crefname{appendix}{Appendix}{Appendices}
\crefname{algorithm}{Algorithm}{Algorithms}
\crefname{figure}{Figure}{Figures}
\crefname{table}{Table}{Tables}
\crefname{remark}{Remark}{Remarks}
\crefname{definition}{Definition}{Definitions}
\crefname{equation}{Equation}{Equations}
\crefname{corollary}{Corollary}{Corollaries}
\Crefname{corollary}{Corollary}{Corollaries}

\newtheorem{theorem}{Theorem}
\newtheorem{lemma}[theorem]{Lemma} 
\newtheorem{proposition}[theorem]{Proposition} 
\newtheorem{remark}[theorem]{Remark}
\newtheorem{corollary}[theorem]{Corollary}
\newtheorem{definition}[theorem]{Definition}

\newtheorem{observation}[theorem]{Observation}

\DeclarePairedDelimiter\floor{\lfloor}{\rfloor}

\newcommand{\parti}{\text{Z}} %
\newcommand{\zero}{\mathbf{0}} %
\newcommand{\one}{\mathbf{1}} %
\newcommand{\ele}{v} %
\newcommand{\multi}{f_{\text{mt}}} %
\newcommand{\bigo}[1]{\mathcal {O}\! \left(#1\right)}

\newcommand{\epe}[2][]{\underset{#1}{\mathbb E}\left[#2\right]}  %

\def \bu{\mathbf{u}}

\def \m{\mathbf{m}}

\def \flweights {R} %

\newcommand{\stepsize}{step size\xspace}
\newcommand{\stepsizes}{step sizes\xspace}

\newcommand{\pga}{\algname{PGA}\xspace}
\newcommand{\nonconvexfw}{\algname{Non-convex FW}\xspace}
\newcommand{\submodularfw}{\algname{Submodular FW}\xspace}
\newcommand{\shrunkenfw}{\algname{Shrunken FW}\xspace}
\newcommand{\twophasefw}{\algname{Two-Phase FW}\xspace}
\newcommand{\twophase}{\algname{Two-Phase}\xspace}

\newcommand{\irsuper}{{IR-supermodular}\xspace}
\newcommand{\ir}{{IR}\xspace}

\def \bmA{\mathbf{A}}

\def \bmI{\mathbf{I}}

\def \bmL{\mathbf{L}}

\def \bmtheta{\bm{\theta}}

\newcommand{\geqco}{\ensuremath{\geq}}
\newcommand{\leqco}{\ensuremath{\leq}}

\newcommand{\ith}[1]{\ensuremath{#1^\text{th}}}
\newcommand{\nondec}{nondecreasing\xspace}

\newcommand{\set}[1]{\{#1\}} %
\newcommand{\lovasz}{Lov{\'a}sz \xspace}

\jmlrheading{1}{2020}{1-48}{4/00}{10/00}{xx}{Bian and Buhmann and Krause}

\ShortHeadings{Continuous Submodular Function Maximization}{Bian and Buhmann and Krause}
\firstpageno{1}

\begin{document}
\title{Continuous Submodular Function Maximization
}

\author{\name Yatao Bian\thanks{Corresponding author. Most of the work was conducted while Y. Bian was at ETH Zurich. Y. Bian's ORCID id is: \href{https://orcid.org/0000-0002-2368-4084}{orcid.org/0000-0002-2368-4084}} \email yatao.bian@gmail.com \\
\addr Tencent AI Lab\\
Shenzhen, China 518057
\AND
\name Joachim M. Buhmann \email jbuhmann@inf.ethz.ch\\
\addr Department of Computer Science\\
ETH Zurich\\
8092 Zurich, Switzerland
\AND
\name Andreas Krause \email krausea@ethz.ch \\
\addr Department of Computer Science\\
ETH Zurich\\
8092 Zurich, Switzerland
}

\editor{xx and xx}

\maketitle

\begin{abstract}%

Continuous submodular functions are a category of generally non-convex/non-concave functions with a wide spectrum of applications. The celebrated property of this class of functions -- continuous submodularity
-- enables both exact minimization and approximate maximization in polynomial time.
Continuous submodularity is obtained by generalizing the notion of submodularity from discrete domains  to continuous domains. It intuitively captures a repulsive effect amongst different dimensions of the defined multivariate function.

In this paper, we systematically study continuous submodularity and  a class of non-convex optimization problems: {\em continuous submodular function maximization}.
We start by a thorough characterization of the class of continuous submodular functions, and show that continuous submodularity is equivalent to a weak version of the diminishing returns (DR) property. Thus we also derive  a subclass of continuous submodular functions, termed {\em continuous DR-submodular functions}, which enjoys the full DR property. Then we present operations that preserve continuous (DR-)submodularity, thus yielding general rules for composing new submodular functions. We establish intriguing properties for the problem of
constrained DR-submodular maximization, such as the {\em local-global
relation}, which captures the relationship of locally (approximate) stationary points and global optima. We identify several applications of continuous submodular optimization, ranging from
influence  maximization with general marketing strategies, MAP inference
for DPPs to mean field inference for probabilistic log-submodular
models. For these applications, continuous submodularity formalizes valuable domain knowledge relevant for optimizing
this class of objectives.
We present inapproximability results and provable algorithms for two problem settings: constrained monotone DR-submodular maximization and constrained non-monotone DR-submodular maximization.  Finally, we extensively evaluate the effectiveness of the proposed algorithms on different problem instances, such as influence maximization with marketing strategies and revenue maximization with continuous assignments.

\end{abstract}


\section{Introduction}

Submodularity  is  a  structural   property  usually  associated  with
\emph{set  functions}, with  important  implications for  optimization
\citep{nemhauser1978analysis}.  The general  setup
requires a ground set $\groundset$ containing $n$ items, which could
be, for instance, the set of all features in a given supervised learning problem \citep{das2011submodular}, or the set of
all users in the influence maximization problem \citep{kempe2003maximizing}. Usually, we have an objective
function that maps a subset of $\groundset$ to a real value:
$F(X): 2^\groundset \rightarrow \R_+$. This function often quantifies utility,
coverage, relevance, diversity etc.
Equivalently, one can express any subset $X$ as a binary vector
$\x\in \{0, 1\}^n$. Hereby, for component $i$ of $\x,\;x_i=1$ means that item $i$
is inside $X$, otherwise item $i$ is outside of $X$. This binary
representation associates the powerset of $\groundset$
with all vertices of an $n$-dimensional hypercube.  Because of this, we
also call submodularity of set functions ``submodularity over binary
domains'' or ``binary submodularity''.

Over binary domains, there are two  well-known definitions of
submodularity: the lattice definition and the diminishing
returns (DR) definition.

\begin{definition}[Lattice  definition]
	A set function $F: 2^\groundset \mapsto \R_+$ is {\em submodular} iff
	$\forall X, Y \subseteq \groundset$, it holds:
	\begin{align}
	F(X) + F(Y) \geq F(X\cup Y) + F(X \cap Y).
	\end{align}
\end{definition}
One can easily show that it is equivalent to the following DR definition:
\begin{definition}[DR definition]
	A set function $F(X): 2^\groundset \mapsto \R_+$ is {\em submodular} iff
	$\forall A \subseteq B \subseteq \groundset$ and
	$\forall v \in \groundset \setminus B$, it holds:
	\begin{equation}
	F(A \cup \set{v}) - F(A) \geq F(B\cup \set{v}) - F(B).
	\end{equation}
\end{definition}

Optimizing submodular set functions has found numerous applications in
machine learning, including variable selection
\citep{DBLP:conf/uai/KrauseG05}, dictionary learning
\citep{krause2010submodular,das2011submodular}, sparsity inducing
regularizers \citep{bach2010structured}, summarization
\citep{gomes2010budgeted,lin2011class,mirzasoleiman2013distributed} and variational
inference \citep{djolonga2014map}. Submodular set functions can be
efficiently minimized \citep{iwata2001combinatorial}, and there are
strong guarantees for approximate maximization
\citep{nemhauser1978analysis,krause2012submodular}.

Even though submodularity is most widely considered in the discrete
setting, the notion can be generalized to {\em arbitrary lattices}
\citep{fujishige2005submodular}.
Of particular interest  are lattices over real vectors, which can be used to define submodularity over continuous domains \citep{topkis1978minimizing,bach2015submodular,bian2017guaranteed}. But one may wonder: \textit{why do we need continuous submodularity?}

In summary, there are two
motivations for studying continuous submodularity: {\em i)} It is an
important modeling ingredient for many real-world applications; {\em ii)} It
captures a subclass of well-behaved non-convex optimization problems,
which admits guaranteed  optimization with algorithms running in
polynomial time. In the following, we will informally illustrate these two aspects.

\paragraph{Natural Prior Knowledge for Modeling.}

In order to illustrate the first motivation, let us consider a stylized
scenario. Suppose you got stuck in the desert one day, and became
extremely thirsty. After two days of exploration you found a bottle of
water. What is even better is that you also found a bottle of soda.

We will use a two-dimensional function
$f([x_1; x_2])$ to quantize the ``happiness'' gained by having $x_1$
quantity of water and $x_2$ quantity of soda. Let
$\delta = [50 \text{ml water} ; 50 \text{ml soda}]$. Now it is natural
to see that the following inequality shall hold:
$f([1ml; 1ml] + \delta) - f([1ml; 1ml]) \geq f([100ml;
100ml] + \delta) - f([100ml; 100ml])$.
The LHS of the inequality measures the marginal
gain of happiness by having $\delta$ more [water, soda] based on a
\emph{small} context ([1ml; 1ml]), while the RHS means the marginal
gain based on a \emph{large} context ([100ml; 100ml]), this is a typical example of the well-known diminishing returns (DR) phenomenon, which will formally defined in \cref{sec_dr_dr_submodular}.
The DR  property models the context sensitive
expectation that adding one more unit of resource
contributes more in the small context than in a large context.

This example illustrates that diminishing returns effects naturally occur in continuous domains, not only discrete ones. While related to concavity, we will see that continuous submodularity yields complementary means of modeling diminishing returns effects over continuous domains.
Real-world examples comprise user preferences in recommender
systems, customer satisfaction, influence in social advertisements
etc.

\paragraph{Non-Convex Structure enabling Provable Optimization.}

\looseness -1 Non-convex optimization is a core challenge in machine
learning, and arises in numerous learning tasks from training
deep neural networks \citep{bottou2018optimization} to latent variable models
\citep{anandkumar2014tensor}.  A fundamental problem in non-convex
optimization is to reach a stationary point assuming smoothness of the
objective for unconstrained optimization
\citep{sra2012scalable,li2015accelerated,reddi2016fast,allen2016variance}
or constrained optimization problems
\citep{ghadimi2016mini,lacoste2016convergence}.
However, without further assumptions, a stationary point may in general be of arbitrary poor objective value.  It thus remains a challenging
problem to understand which classes of non-convex objectives can be
tractably optimized.

In pursuit of solving this challenging problem, we show that
continuous submodularity provides a natural structure for
provable non-convex optimization.
It arises in various important non-convex objectives. Let us look at
a simple example by considering a classical quadratic program (QP):
$f(\x) = \frac{1}{2}\x^\trans \BH \x + \bh^\trans \x + c$. When $\BH$
is symmetric, we know that the Hessian matrix is $\nabla^2 f =
\BH$.
Let us consider a specific two dimensional example, where
$\BH = [-1, -2; -2, -1]$. One can verify that its eigenvalues are
$[1; -3]$. So it is an indefinite quadratic program, which is
neither convex, nor concave. However, it will soon be clear that $f$ is a
DR-submodular function (see definitions in
\cref{sec_defs_continuous_submodurity}). In this paper, we propose polynomial-time
solvers for optimizing such objectives with strong approximation guarantees.
Further examples of submodular objectives include the  \lovasz \citep{lovasz1983submodular} and
multilinear extensions \citep{calinescu2007maximizing} of submodular
set functions, or to the softmax extensions \citep{gillenwater2012near} for
DPP (determinantal point process) MAP inference.

\paragraph{Organization of the Paper.}

We will present a brief background of submodular optimization, the classical Frank-Wolfe algorithm and existing structures for non-convex optimization  in \cref{sec_background}.  In \cref{sec_defs_continuous_submodurity} we  give a thorough
characterization of the class of continuous submodular and
DR-submodular\footnote{A DR-submodular function is a submodular
	function with the additional diminishing returns (DR) property, which
	will be formally defined in \cref{sec_defs_continuous_submodurity}.}
functions. \cref{sec_ops_preserving_submodularity} presents general composition rules that preserve continuous (DR-)submodularity, along with exemplary applications of these rules, such as for designing deep submodular functions.
\cref{sec_structures} discusses intriguing properties for the problem of
constrained DR-submodular maximization in both monotone and non-monotone settings, such as the local-global
relation.
In \cref{sec_app} we illustrate representative
applications of continuous submodular optimization.
In the next two sections we discuss hardness results and algorithmic
techniques for constrained DR-submodular maximization in different
settings: \cref{sec_mono_dr_fun} illustrates how to maximize
monotone continuous DR-submodular functions,
and \cref{sec_algs}
provides techniques for maximizing non-monotone DR-submodular functions
with a down-closed convex constraint.
We present experimental results on three representative problems in \cref{sec_exp}.
Lastly, \cref{sec_discu} discusses and concludes the paper.

\section{Background and Related Work}
\label{sec_background}

We give a brief introduction of the background of submodular optimization in this section.

\paragraph{Notation.}
Throughout this work we assume
$ \groundset =\{\ele_1, \ele_2,..., \ele_n\}$ being the ground set of
$n$ elements, and $\chara_i\in\R^n$ is the characteristic vector for
element $\ele_i$ (also the standard $i^\text{th}$ basis vector).
We use boldface letters $\x\in \R^\groundset$ and $\x\in \R^n$
interchangebly to indicate an $n$-dimensional vector, where $x_i$ is
the $i^\text{th}$ entry of $\x$. We use a boldface capital letter
$\BA\in\R^{m\times n}$ to denote an $m$ by $n$ matrix and use $A_{ij}$
to denote its ${ij}^\text{th}$ entry.
By default, $f(\cdot)$ is used to denote a continuous function, and
$F(\cdot)$ to represent a set function.
For a differentiable function $f(\cdot)$, $\nabla f(\cdot)$ denotes its gradient, and for a twice differentiable function $f(\cdot)$, $\nabla^2 f(\cdot)$ denotes its Hessian.
$[n]:= \{1,...,n\}$ for an
integer $n \geq 1$.  $\|\cdot\|$ means the Euclidean norm by default.
Given two vectors $\x,\y$, $\x\leqco \y$ means
$x_i\leq y_i, \forall i$.  $\x\vee \y$ and $\x \wedge \y $ denote
coordinate-wise maximum and coordinate-wise minimum, respectively.
$\sete{x}{i}{k}$ is the operation of setting the
$i^\text{th}$ element of $\x$ to $k$, while keeping all other elements
unchanged, i.e., $\sete{x}{i}{k}=\x-x_i \bas_i + k\bas_i$.

\subsection{Submodularity over Discrete Domains}

As a discrete analogue of convexity, submodularity
provides computationally effective structure so that many discrete
problems with this property can be  efficiently solved or approximated.
Of particular interest is a $(1-1/e)$-approximation for maximizing a
monotone submodular set function subject to a cardinality, a matroid,
or a knapsack constraint \citep{nemhauser1978analysis,
	DBLP:conf/stoc/Vondrak08, sviridenko2004note}.  For maximizing  non-monotone
submodular functions, a 0.325-approximation under cardinality and
matroid constraints \citep{gharan2011submodular}, and a
0.2-approximation under a knapsack constraint have been shown
\citep{lee2009non}.  Another result pertains to unconstrained maximization of
non-monotone submodular set functions, for which
\citet{buchbinder2012tight} propose the deterministic double greedy
algorithm with a 1/3 approximation guarantee, and the randomized double
greedy algorithm that  achieves the tight 1/2 approximation guarantee.

Although most commonly associated with set functions, in many
practical scenarios, it is natural to consider generalizations of
submodular set functions, including \textit{bisubmodular} functions,
\textit{$k$-submodular} functions, \textit{tree-submodular} functions,
\textit{adaptive submodular} functions, as well as submodular
functions defined over integer lattices.

\citet{golovin2011adaptive} introduce the notion of adaptive submodularity to
generalize submodular set functions to adaptive policies.
\citet{kolmogorov2011submodularity} studies tree-submodular functions
and presents a polynomial-time algorithm for minimizing them.
For distributive lattices, it is well-known that the combinatorial
polynomial-time algorithms for minimizing a submodular set function
can be adopted to minimize a submodular function over a bounded
integer lattice \citep{fujishige2005submodular}.

 Approximation algorithms for maximizing
bisubmodular functions and $k$-submodular functions have  been
proposed by \citet{singh2012bisubmodular,ward2014maximizing}.
Recently, maximizing a submodular function over integer lattices has
attracted considerable attention. In particular,
\citet{soma2014optimal} develop a $(1-1/e)$-approximation algorithm
for maximizing a monotone DR-submodular integer function under a
knapsack constraint. For non-monotone submodular functions over the
bounded integer lattice, \citet{gottschalk2015submodular} provide a
1/3-approximation algorithm.
Recently, \citet{soma2018maximizing} present a continuous non-smooth extension
for maximizing monotone integer submodular functions.

\subsection{Submodularity over Continuous Domains}

Even though submodularity is most widely considered in the discrete
realm, the notion can be generalized to arbitrary lattices
\citep{fujishige2005submodular}.
\citet{DBLP:journals/mor/Wolsey82} considers maximizing a special
class of continuous submodular functions subject to one knapsack
constraint, in the context of solving location problems. That class of
functions are additionally required to be monotone, piecewise linear
and \textit{concave}.
\citet{calinescu2007maximizing} and \citet{DBLP:conf/stoc/Vondrak08}
discuss a subclass of continuous submodular functions, which is termed
smooth submodular functions\footnote{A function
	$f: [0,1]^n \rightarrow \R$ is smooth submodular if it has second
	partial derivatives everywhere and all entries of its Hessian matrix
	are non-positive.}, to describe the multilinear extension of a
submodular set function.  They propose the continuous greedy
algorithm, which has a $(1-1/e)$ approximation guarantee for maximizing
a smooth submodular function under a down-closed polytope
constraint. \citet{bach2015submodular} considers the problem of
{\em minimizing continuous submodular functions}, and proves that
efficient techniques from convex optimization may be used for
minimization \citep{fujishige2005submodular}.

\citet{ene2016reduction} provide an approach for reducing
integer DR-submodular function maximization problems to  submodular
set function maximization problem. This approach suggests a way to approximately optimize
continuous submodular functions over \textit{simple} continuous
constraints: Discretize the continuous function and constraint to be
an integer instance, and then optimize it using the
reduction. However, for monotone DR-submodular function maximization,
this method can not handle the general continuous constraints
discussed in this work, i.e., arbitrary down-closed convex sets. Moreover,
for general submodular function maximization, this method cannot be
applied, since the reduction needs the additional diminishing returns
property.  Therefore we focus on explicitly continuous methods in this work.

Recently, \citet{niazadeh2018optimal}\footnote{Appeared later than when the paper \citet{bian2019optimalmeanfield} was released.} present optimal algorithms for
non-monotone submodular maximization with a box constraint.
Continuous submodular maximization is also well studied in the
stochastic setting \citep{karimi2017stochastic,hassani2017gradient,mokhtari2018stochastic},
online setting \citep{chen2018online}, bandit setting
\citep{durr2019non} and decentralized setting
\citep{mokhtari2018decentralized}.

\subsection{Classical Frank-Wolfe Algorithm}

Since the workhorse algorithms for continuous  DR-submodular maximization are
Frank-Wolfe style algorithms, we give a brief introduction of
classical Frank-Wolfe algorithms in this section.
The Frank-Wolfe algorithm \citep{frank1956algorithm} (also known as
Conditional Gradient algorithm or the Projection-Free algorithm) is
one of the
classical algorithms for constrained convex
optimization. It has received renewed interest in recent years due to its
projection free nature and its ability to exploit structured
constraints \citep{jaggi13}.

The Frank-Wolfe algorithm solves the following constrained
optimization problem:
\begin{align}
\min_{\x\in \R^n, \; \x\in \D} f(\x),
\end{align}
where $f$ is differentiable with $L$-Lipschitz gradients and the
constraint $\D$ is convex and compact.

A sketch of the Frank-Wolfe algorithm is presented in
\cref{alg_classical_fw}.  It needs an initializer $\x^\pare{0}\in \D$.
Then it runs for $T$ iterations. In each iteration it does the following: in
Step \labelcref{alg_fw_lmo} it solves a linear minimization problem
whose objective is defined by the current gradient $\nabla f(\x^{t})$.
This step is often called the linear minimization/maximization oracle
(LMO). In Step
\labelcref{alg_fw_stepsize} a \stepsize $\gamma$ is chosen. Then it
updates the solution $\x$ to be a convex combination of the current
solution and the LMO output $\s$.

\begin{algorithm}[ t]
	\caption{Classical Frank-Wolfe algorithm for constrained convex
		optimization \citep{frank1956algorithm}}\label{alg_classical_fw}

	\KwIn{$\min_{\x\in \R^n, \x\in \D} f(\x)$; $\x^\pare{0}\in \D$ }
	\For{$t=0\dots T$}{
		{Compute
			$\s^\pare{t} := \argmin_{\s\in \D} \left\langle \s, \nabla
			f(\x^{t}) \right\rangle$ \label{alg_fw_lmo} \tcp*{LMO}}

		{Choose \stepsize $\gamma \in (0, 1]$\;  \label{alg_fw_stepsize}}

		{	Update $\x^\pare{t+1}:= (1-\gamma)\x^{t}+\gamma\s^\pare{t}$\;}
	}
	\KwOut{$\x^\pare{T}$\;}
\end{algorithm}

There are several popular rules to choose the \stepsize in Step
\labelcref{alg_fw_stepsize}. For a short summary: i)
$\gamma_t := \frac2{t+2}$, which is often called the ``oblivious''
rule since it does not depend on any information of the optimization
problem; ii)
$\gamma_t = \min\{1, \frac{g_\pare{t}}{L \|\s^\pare{t} -
	\x^\pare{t}\|} \}$,
where
$g_\pare{t} := -\dtp{\nabla f(\x^{t})}{\s^\pare{t} - \x^\pare{t}} $ is
the so-called Frank-Wolfe gap, which is an upper bound of the
suboptimality if $f$ is convex; iii) Line search rule:
$\gamma_t := \argmin_{\gamma\in [0, 1]} f(\x^\pare{t} + \gamma
(\s^\pare{t} - \x^\pare{t}) )$.

\paragraph{Frank-Wolfe Algorithm for Non-Convex Optimization.}

Recently, Frank-Wolfe algorithms have been extended for smooth
non-convex optimization problems with
constraints. \citet{lacoste2016convergence} analyzes the Frank-Wolfe
method for general constrained non-convex optimization problems, where
he uses the Frank-Wolfe gap as the non-stationarity measure.
\citet{reddi2016stochastic} study Frank-Wolfe methods for non-convex
stochastic and finite-sum optimization problems. They also used the
Frank-Wolfe gap as the non-stationarity measure.

\subsection{Structures for Non-Convex Optimization}

Optimizing  non-convex continuous functions has received considerable
interest in the last decades.
There are two widespread structures for non-convex optimization: {\em quasi-convexity} and {\em geodesic convexity}, both of them are based on relaxations of the classical convexity definition.

\paragraph{Quasi-Convexity.}

A function $f: \D \mapsto \R$ defined on a convex subset $\D$ of a
real vector space is {\em quasi-convex} if for all $\x, \y\in \D$ and
$\lambda\in [0, 1]$ it holds,
\begin{align}
f(\lambda \x + (1- \lambda)\y) \leq \max\{ f(\x), f(\y) \}.
\end{align}
Quasi-convex optimization problems appear in different areas, such as
industrial organization \citep{wolfstetter1999topics} and computer
vision \citep{ke2007quasiconvex}.
Quasi-convex optimization problems can be solved by a series of convex
feasibility problems \citep{boyd2004convex}.
\citet{hazan2015beyond} study stochastic quasi-convex optimization,
where they proved that a stochastic version of the normalized gradient
descent can converge to a global minimium for quasi-convex functions
that are locally Lipschitz.

\paragraph{Geodesic Convexity.}

Geodesic convex functions are a class of generally
non-convex functions in Euclidean space. However, they
still enjoy the nice property that local optimality implies global optimality.
\citet{sra2016geometric} provide an introduction to
geodesic convex  optimization with machine learning applications.
Recently, \citet{vishnoi2018geodesic}
study various aspects of
geodesic convex optimization.

\begin{definition}[Geodesically convex functions]
	Let $(\M, g)$ be a Riemannian manifold and $K\subseteq \M$ be a totally convex set with respect to $g$. A function $f: K \rightarrow \R$ is a geodesically convex function with respect to $g$ if $\forall \p, \q \in K$, and for all geodesic $\gamma_{\p \q}: [0, 1]\rightarrow K$ that joins $\p$ to $\q$, it holds,
	\begin{align}
	\forall t\in [0, 1], f(\gamma_{\p \q}(t)) \leq (1- t) f(\p) + t f(\q).
	\end{align}
\end{definition}

Various applications with non-convex objectives in Euclidean space can be resolved with geodesic convex optimization methods, such as Gaussian mixture models \citep{hosseini2015matrix},
metric learning \citep{zadeh2016geometric} and matrix square root \citep{sra2015matrix}.
By deriving explicit expressions for the smooth manifold structure, such as  inner products, gradients, vector transport and Hessian, various optimization methods have been developed. \citet{jeuris2012survey} present conjugate gradient, BFGS and trust-region methods. \citet{qi2010riemannian} propose the Riemannian BFGS (RBFGS) algorithm for general retraction and vector transport. \citet{ring2012optimization} prove its local superlinear rate of convergence. \citet{sra2015conic} present a limited memory version of RBFGS.

\paragraph{Other Non-convex Structures.}

Tensor methods have been used in various non-convex problems,
e.g.,
learning latent variable models \citep{anandkumar2014tensor} and training neural networks \citep{janzamin2015beating}.
A fundamental problem in non-convex optimization is to reach a stationary point assuming the smoothness of the objective \citep{sra2012scalable,li2015accelerated,reddi2016fast,allen2016variance}.
With extra  assumptions, certain global
convergence results can be obtained. For example, for functions with Lipschitz continuous
Hessians, the regularized Newton scheme of \cite{nesterov2006cubic}
achieves global convergence results for functions
with an additional star-convexity property or with an additional gradient-dominance
property \citep{polyak1963gradient}.  \cite{hazan2015graduated} introduce
the family of $\sigma$-nice functions and propose a graduated optimization-based algorithm, that
provably converges to a global optimum for this family of non-convex functions.
However, it is typically difficult to verify whether these assumptions hold in real-world problems.

\subsection{Our Contributions}

To the best of our knowledge, this work is the \emph{first}\footnote{This journal paper is partially based on the previous conference papers \cite{bian2017guaranteed}, \cite{biannips2017nonmonotone} also the thesis \cite{bian2019provable}.} to systematically study continuous submodularity and its maximization algorithms.
Our main contributions are:

\paragraph{Thorough characterizations of submodularity.}
By lifting the notion of submodularity to continuous domains, we
identify a subclass of tractable non-convex optimization problems:
continuous submodular optimization. We provide a thorough
characterization of continuous submodularity, which results in
$0^{\text{th}}$ order, $1^{\text{st}}$ order and $2^{\text{nd}}$
order definitions.

\paragraph{Continuous submodularity preserving operations.} We study general principles for maintaining continuous (DR-)submodularity. These enable: {\em i)} Convenient ways of recognizing new continuous submodular objectives; {\em ii)} Generic rules for designing new continuous or discrete submodular objectives, such as deep submodular functions.

\paragraph{Properties of constrained DR-submodular maximization.}
We discover intriguing properties of the general constrained DR-submodular maximization problem, such as the local-global relation (in
\cref{local_global}), which  relates (approximately) stationary
points and the global optimum, thus allowing to incorporate
progress in the area of non-convex optimization research.

\paragraph{Provable algorithms for DR-submodular maximization.}
We establish hardness results and propose provable algorithms for
constrained DR-submodular maximization in two settings: {\em i)} Maximizing
monotone functions with down-closed convex constraints; {\em ii)}
Maximizing non-monotone functions with down-closed convex
constraints.

\paragraph{Applications with (DR)-submodular objectives.}
We formulate representative applications with (DR)-submodular objectives from various areas,  such as machine learning, data mining and combinatorial optimization.

\paragraph{Extensive experimental evaluations.}
We present representative applications with the studied
continuous submodular objectives, and extensively evaluate the
proposed algorithms on these applications.

\section{Characterizations of  Continuous Submodular  Functions}
\label{sec_defs_continuous_submodurity}

Continuous submodular functions are defined on subsets of $\R^n$:
$\X = \prod_{i=1}^n \X_i$, where each $\X_i$ is a compact subset of
$\mathbb{R}$ \citep{topkis1978minimizing, bach2015submodular}.  A
function $f: \X \rightarrow \R$ is submodular \textit{iff} for all
$(\x, \y)\in \X \times \X$,

\begin{align}\label{eq1}
f(\x) + f(\y) \geq f(\x \vee \y) + f(\x \wedge \y),  \quad
(\emph{submodularity})
\end{align}

\noindent where $\wedge$ and $\vee$ are the coordinate-wise minimum and maximum
operations, respectively.  Specifically, $\X_i$ could be a finite set,
such as $\{0, 1 \}$ (in which case $f(\cdot)$ is called a \textit{set}
function), or $\{0, ..., k_i-1 \}$ (called \textit{integer} function),
where the notion of continuity is vacuous; $\X_i$ can also be an
interval, which is referred to as a continuous domain. In this
section, we consider the interval by default, but it is worth noting
that the properties introduced in this section can be applied to
$\X_i$ being a general compact subset of $\R$.

When twice-differentiable, $f(\cdot)$ is submodular iff all
off-diagonal entries of its Hessian matrix are non-positive\footnote{Notice
	that an equivalent definition of (\ref{eq1}) is that
	$\forall \x\in \X$, $\forall i \neq j$ and $a_i, a_j\geq 0$
	s.t. $x_i +a_i\in \X_i, x_j+a_j\in \X_j$, it holds
	$f(\x+a_i\bas_i) + f(\x+a_j\bas_j) \geq f(\x) + f(\x+a_i\bas_i +
	a_j\bas_j)$.
	With $a_i$ and $a_j$ approaching zero, one gets (\ref{eq2}).}
\citep{topkis1978minimizing,bach2015submodular},
\begin{equation}\label{eq2}
\forall \x\in \X, \;\; \frac{\partial^2 f(\x)}{\partial x_i \partial x_j}
\leq 0, \;\; \forall i \neq j.
\end{equation}

The class of continuous submodular functions contains a subset of both
convex and concave functions, and shares some useful properties with
them (illustrated in \cref{fig_sub}).  Examples include submodular and
convex functions of the form $\phi_{ij}(x_i - x_j)$ for $\phi_{ij}$
convex;
submodular and concave functions of the form $\x
\mapsto g(\sum_{i=1}^{n} \lambda_i x_i)$ for
$g$
concave and $\lambda_i$
non-negative.  Lastly, indefinite quadratic functions of the form
$f(\x)
= \frac{1}{2} \x^\trans \BH \x + \bh^\trans \x +
c$ with all off-diagonal entries of
$\bmH$
non-positive are examples of submodular but non-convex/non-concave
functions.
Interestingly, characterizations of continuous submodular functions
are in correspondence to those of convex functions, which are
summarized in \cref{tab_comparison}.

\begin{table}[t]
	\begin{center}
		\caption{Comparison of definitions of continuous  submodular and convex functions}
		\label{tab_comparison}

		\begin{tabularx}{\textwidth}{|l|X|X|}
			\hline
			Definitions &   Continuous submodular  function $f(\cdot)$ & Convex function $g(\cdot)$, $\forall \lambda\in [0,1]$ \\
			\hline \hline
			$0^{\text{th}}$ order  &  $f(\x) + f(\y) \geq f( \x \vee \y) + f(\x \wedge \y)$ &  $\lambda g(\x) + (1-\lambda)g(\y) \geq g(\lambda \x + (1-\lambda) \y)$   \\
			\hline
			$1^{\text{st}}$ order & {\texttt{weak DR} property (\cref{def_weakdr}), or $\nabla f(\cdot)$ {is a \textcolor{link_color}{weak antitone} mapping} (\cref{lemma_weak_antitone})} & $g(\y) \geq g(\x) +   \dtp{\nabla g(\x)}{\y-\x}$   \\
			\hline
			$2^{\text{nd}}$ order & $\frac{\partial^2 f(\x)}{\partial x_i \partial x_j} \leq 0$, \textcolor{link_color}{$\forall i \neq j$}   & $\nabla^2 g(\x)  \succeq 0$ (symmetric positive semidefinite)   \\
			\hline
		\end{tabularx}
	\end{center}
\end{table}

\subsection{The DR Property and DR-Submodular Functions}
\label{sec_dr_dr_submodular}

The Diminishing Returns (DR) property  was introduced
when studying set and integer functions.
We generalize the {DR} property to general functions defined
over $\X$. It will soon be clear that the {DR} property
defines a subclass of submodular functions.  All of the proofs can be
found in \cref{app_proof}.

\begin{figure}
	\centering
	\includegraphics[width=0.5\textwidth]{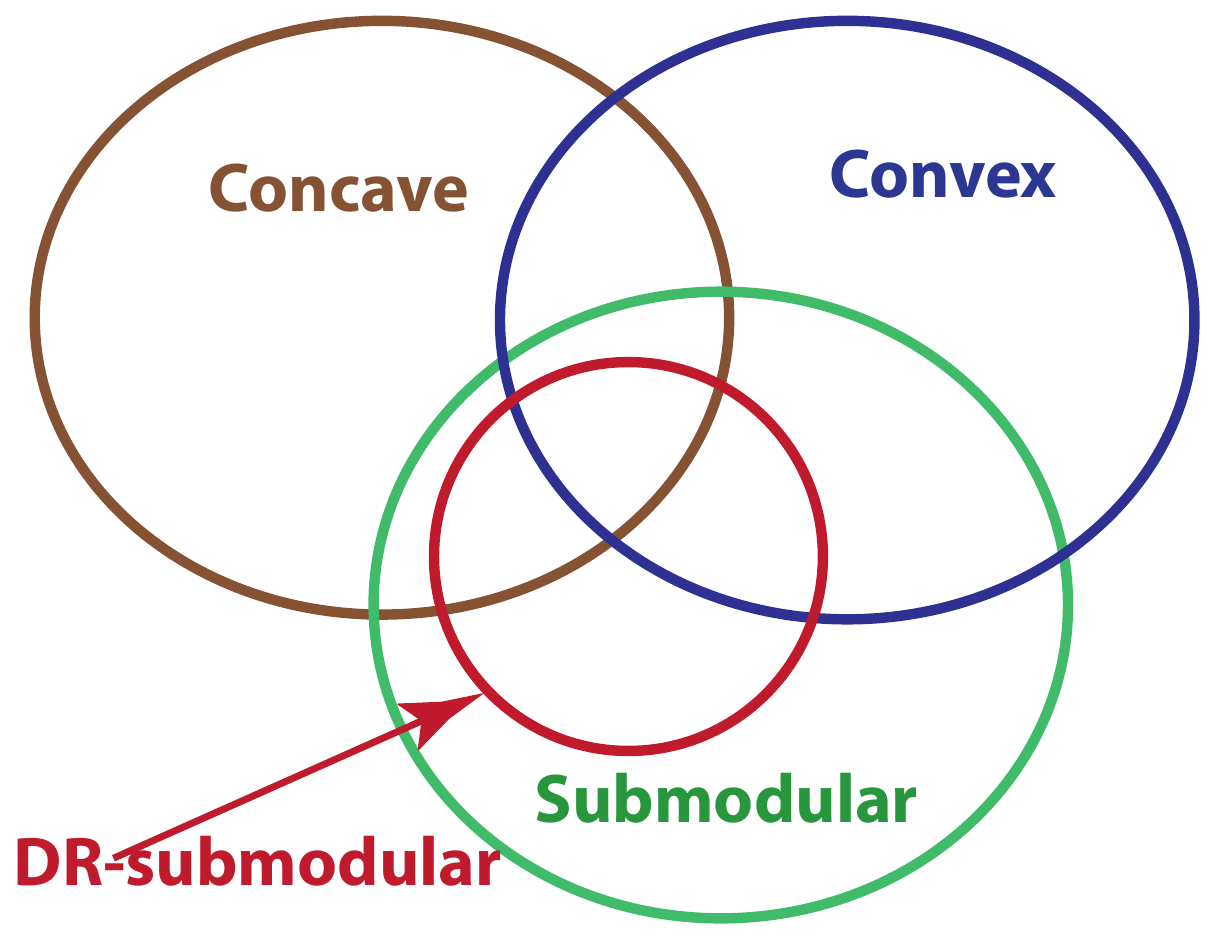}
	\caption{Venn diagram for concavity, convexity, submodularity and DR-submodularity.}
	\label{fig_sub}
\end{figure}

\begin{definition}[{DR}/IR property,   DR-submodular/\irsuper functions]\label{def_dr}
	A function $f(\cdot)$ defined over $\X$ satisfies the
	\emph{diminishing returns (DR)} property if
	$\forall \a \leqco {\b} \in \X$, $\forall i\in [n]$,
	$\forall k\in \R_+$ such that $(k\chara_i+ \a)$ and $(k\bas_i + \b)$
	are still in $\X$, it holds,
	\begin{equation}
	f(k\chara_i+ \a) - f(\a) \geq f(k\chara_i + \b) - f(\b).
	\end{equation}

	This function $f(\cdot)$ is called a DR-submodular\footnote{Note that the
		DR property implies submodularity and thus the name
		``DR-submodular'' contains redundant information about submodularity
		of a function, but we keep this terminology to be consistent with
		previous literature on integer submodular functions.} function.
	If $- f(\cdot)$ is DR-submodular, we call $f(\cdot)$ an
	\textcolor{link_color}{\irsuper} function, where \ir stands for ``Increasing Returns''.
\end{definition}

One immediate observation is that for a differentiable DR-submodular
function $f(\cdot)$, we have that $\forall \a\leqco \b\in \X$,
$\nabla f(\a)\geqco \nabla f(\b)$, i.e., the gradient
$\nabla f(\cdot)$ is an \emph{antitone} mapping from $\R^n$ to
$\R^n$. This observation can be formalized below:

\begin{lemma}[Antitone mapping]\label{lemma_dr_antitone}
	If $f(\cdot)$ is continuously differentiable, then $f(\cdot)$ is
	DR-submodular iff $\nabla f(\cdot)$ is an \emph{antitone mapping}
	from $\R^n$ to $\R^n$, i.e., $\forall \a\leqco \b\in \X$,
	$\nabla f(\a)\geqco \nabla f(\b)$.
\end{lemma}

Recently, the {DR} property is explored by \citet{NIPS2016_6073} to
achieve the worst-case competitive ratio for an online concave
maximization problem.  The {DR} property is also closely related to a
sufficient condition on a concave function $g(\cdot)$ \citep[Section
5.2]{bilmes2017deep}, to ensure submodularity of the corresponding set
function generated by giving $g(\cdot)$ boolean input vectors.

\subsection{The Weak DR Property and Its Equivalence to Submodularity}

It is well known that for set functions, the {DR} property is
equivalent to submodularity, while for integer functions,
submodularity does not in general imply the {DR} property
\citep{soma2014optimal,DBLP:conf/nips/SomaY15,soma2015maximizing}. However,
it was unclear whether there exists a diminishing-return-style
characterization that is equivalent to submodularity of
integer functions. In this work we give a positive answer to
this question by proposing the \textit{weak diminishing returns}
(\texttt{weak DR}) property for general functions defined over $\X$,
and prove that \texttt{weak DR} gives a sufficient and necessary
condition for a general function to be submodular.

\begin{definition}[{Weak DR} property]\label{def_weakdr}
	A function $f(\cdot)$ defined over $\X$ has the \textit{weak
		diminishing returns} property \emph{(weak DR)} if
	$\forall \a\leqco \b\in \X$,
	$\color{blue}\forall i\in \groundset \text{ such that } a_{i} =
	b_{i}$,
	$\forall k\in \R_+$ such that $(k\chara_i+\a)$ and $(k\chara_i+\b)$
	are still in $\X$, it holds,
	\begin{equation} \label{def_supp_dr2}
	f(k\chara_i+\a) - f(\a) \geq
	f(k\chara_i + \b) - f(\b).
	\end{equation}
\end{definition}

The following proposition shows that for all set functions, as well as
integer and continuous functions, submodularity is equivalent
to the \NEWDR\ property. All the proofs can be found in \cref{app_proof}.
\begin{proposition}[\texttt{submodularity}) $\Leftrightarrow$
	(\NEWDR]\label{lemma_support_dr} A function $f(\cdot)$ defined over
	$\X$ is submodular \emph{iff} it satisfies the \textit{weak DR}
	property.
\end{proposition}

Given \cref{lemma_support_dr},
one can treat \NEWDR\ as the first order definition of submodularity:
Notice that for a continuously differentiable  function $f(\cdot)$ with the
\texttt{weak DR} property,   we have that $\forall  \a\leqco \b\in \X$, $\forall i\in \groundset \text{ s.t. } a_{i} = b_{i}$, it holds   $\nabla_i f(\a)\geq \nabla_i f(\b)$,  i.e., $\nabla f(\cdot)$ is a \textit{weak} antitone mapping.
Formally,

\begin{lemma}[Weak antitone mapping]\label{lemma_weak_antitone}
	If $f(\cdot)$ is continuously differentiable, then $f(\cdot)$ is
	submodular iff $\nabla f(\cdot)$ is a weak \emph{antitone mapping}
	from $\R^n$ to $\R^n$, i.e., $\forall \a\leqco \b\in \X$,
	$\forall i\in \groundset \text{ s.t. } a_{i} = b_{i}$,
	$\nabla_i f(\a)\geq \nabla_i f(\b)$.
\end{lemma}

Now we show that the \texttt{DR} property is stronger than the
\texttt{weak DR} property, and the class of DR-submodular functions is
a proper subset of that of submodular functions, as indicated by
\cref{fig_sub}.

\begin{proposition}[\texttt{submodular/weak DR}) +
	(\texttt{coordinate-wise concave}) $\Leftrightarrow$
	(\texttt{DR}]\label{lemma_dr} A function $f(\cdot)$ defined over
	$\X$ satisfies the \texttt{DR} property iff $f(\cdot)$ is submodular
	and coordinate-wise concave, where the \texttt{coordinate-wise
		concave} property is defined as: $\forall \x\in \X$,
	$\forall i\in \groundset$, $\forall k, l\in \R_+$ s.t.
	$(k\chara_i+\x), (l\chara_i + \x), ((k+l)\chara_i + \x)$ are still
	in $\X$, it holds,
	\begin{align}\label{def_coordinatewise_concave}
	f(k\chara_i+\x) - f(\x) \geq f((k+l)\chara_i + \x) - f(l\chara_i + \x),
	\end{align}
	or equivalently (if twice differentiable)
	$\frac{\partial^2 f(\x)}{\partial x_i^2} \leq 0, \forall i\in
	\groundset$.
\end{proposition}
Proposition \ref{lemma_dr} shows that a twice differentiable function $f(\cdot)$
is DR-submodular iff $\forall \x\in \X,  \frac{\partial^2 f(\x)}{\partial x_i \partial x_j}
\leq 0, \forall i, j\in \groundset$, which does not necessarily imply the concavity of $f(\cdot)$.
Given Proposition \ref{lemma_dr}, we also have the characterizations
of continuous DR-submodular functions, which are summarized in
\cref{tab_dr}.
\begin{table}[t]
	\begin{center}
		\caption{Summarization of definitions of continuous DR-submodular
			functions}
		\label{tab_dr}
		\begin{tabularx}{\textwidth}{|l|X|X|}
			\hline Definitions & Continuous DR-submodular function
			$f(\cdot)$, $\forall \x, \y\in \X$\\
			\hline \hline $0^{\text{th}}$ order &
			$f(\x) + f(\y) \geq f( \x \vee \y) +
			f(\x \wedge \y)$, and $f(\cdot)$ is coordinate-wise concave (see  \labelcref{def_coordinatewise_concave}) \\
			\hline $1^{\text{st}}$ order & {\texttt{DR} property
				(\cref{def_dr}), or $\nabla f(\cdot)$ {is an \textcolor{link_color}{antitone} mapping} (\cref{lemma_dr_antitone})}  \\
			\hline $2^{\text{nd}}$ order &
			$\frac{\partial^2 f(\x)}{\partial x_i \partial x_j} \leq 0$,
			\textcolor{link_color}{$\forall i , j$} (all entries of the
			Hessian matrix  being non-positive)  \\
			\hline
		\end{tabularx}
	\end{center}
\end{table}

\subsection{A Simple Visualization}

\cref{fig_2d_softmax} shows the contour of a 2-D continuous submodular
function
$[x_1; x_2]\mapsto 0.7 (x_1 - x_2)^2 + e^{-4(2x_1 - \frac{5}{3})^2} +
0.6e^{-4(2x_1 - \frac{1}{3})^2}+ e^{-4(2x_2 - \frac{5}{3})^2}+
e^{-4(2x_2 - \frac{1}{3})^2}$ and a 2-D DR-submodular function
\begin{align}
\x \mapsto \log\de{\diag(\x)(\bmL-\bmI) +\bmI }, \x\in [0,1]^2,
\end{align}
where $\bmL = [2.25, 3; 3, 4.25]$.
We can see that both of them are neither convex, nor concave. Notice
that along each  coordinate, continuous submodular functions
may behave arbitrarily. In contrast, DR-submodular functions are always concave along any single coordinate.

\begin{figure}[htbp]
	\centering
	\includegraphics[width=0.49\textwidth]{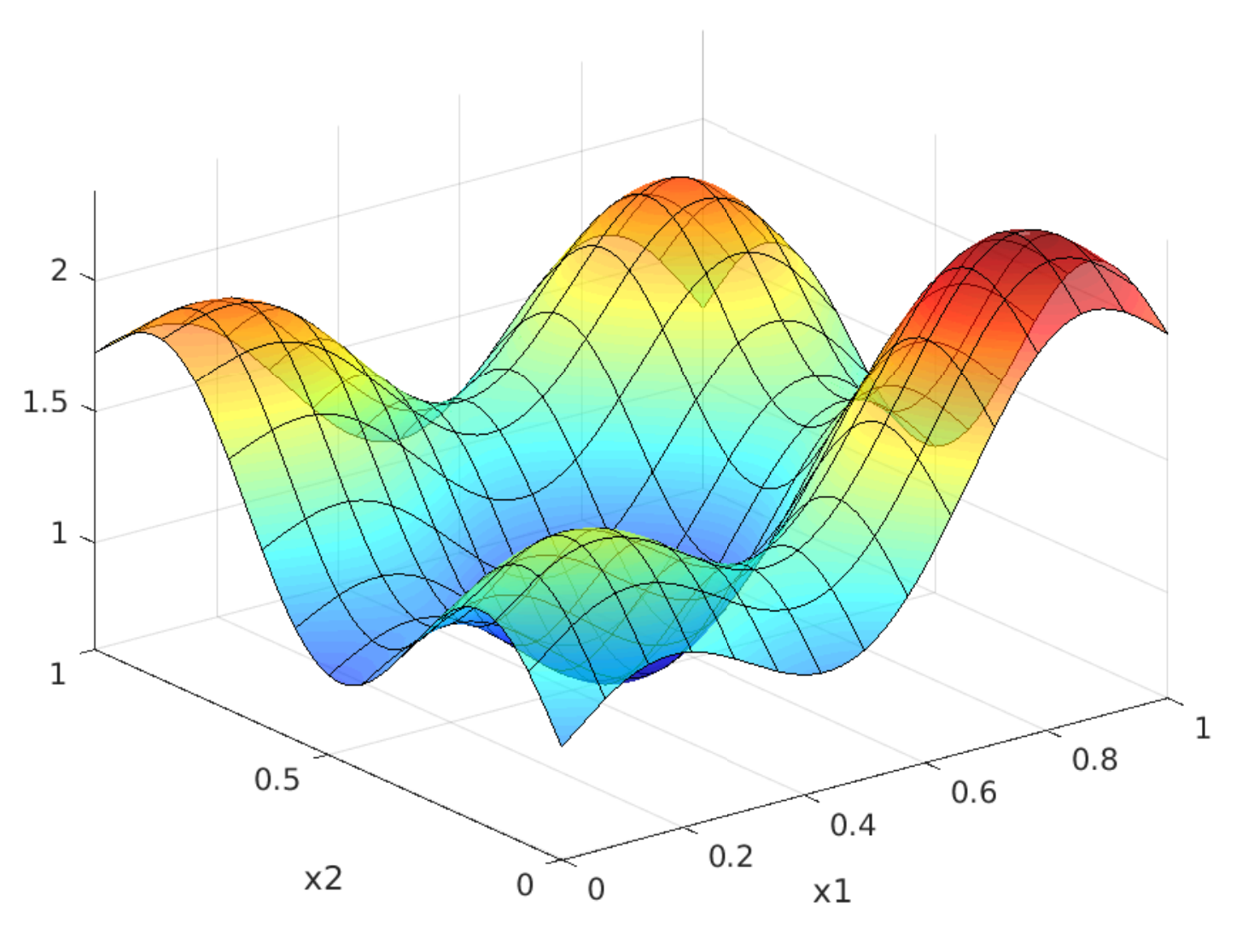}
	\includegraphics[width=0.49\textwidth]{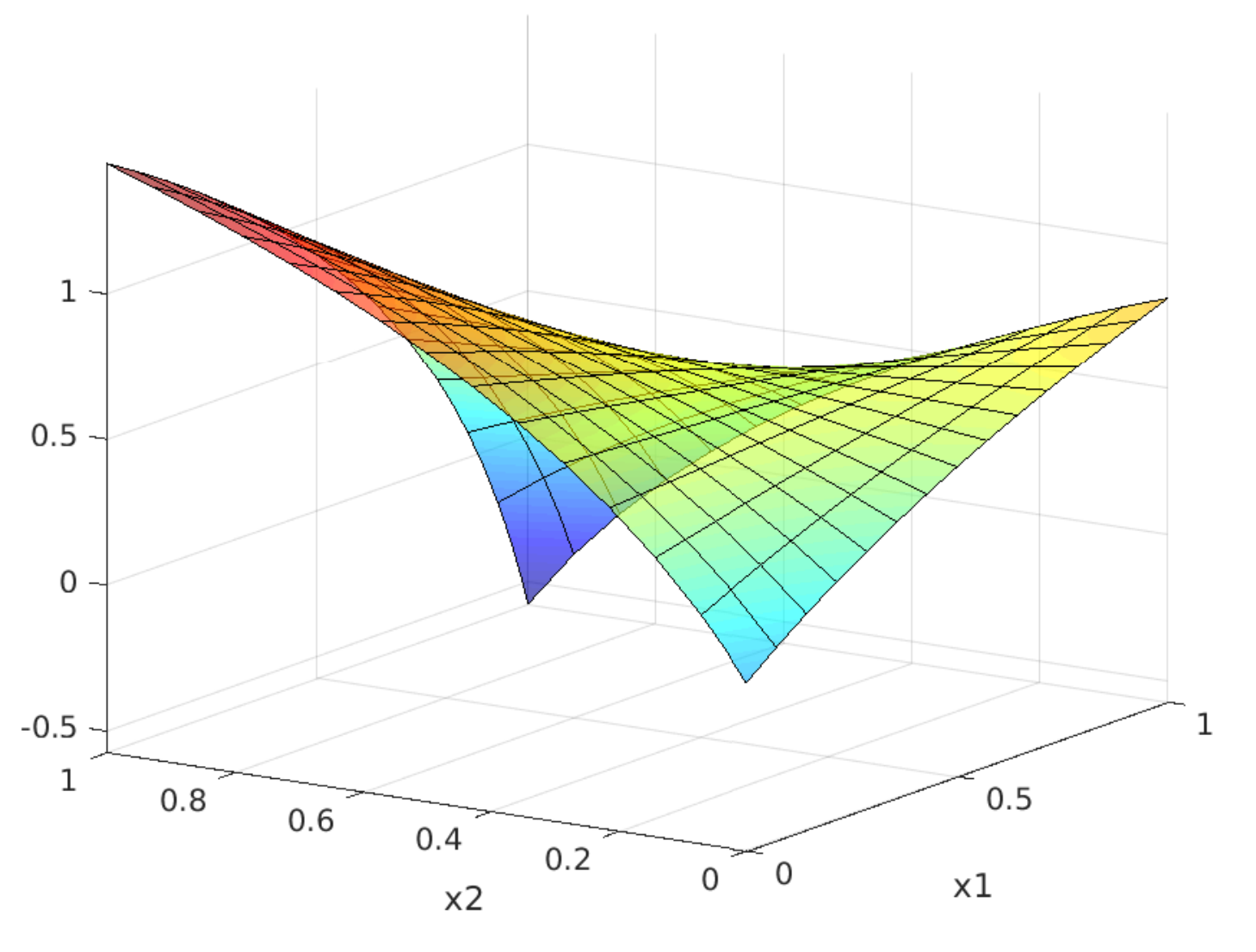}
	\caption{Left: A 2-D continuous submodular function:
		$[x_1; x_2]\mapsto 0.7 (x_1 - x_2)^2 + e^{-4(2x_1 -
			\frac{5}{3})^2} + 0.6e^{-4(2x_1 - \frac{1}{3})^2}+
		e^{-4(2x_2 - \frac{5}{3})^2}+ e^{-4(2x_2 - \frac{1}{3})^2}$.
		Right: A 2-D softmax extension, which is continuous
		DR-submodular.
		$\x \mapsto \log\de{\diag(\x)(\bmL-\bmI) +\bmI }, \x\in
		[0,1]^2$, where $\bmL = [2.25, 3; 3, 4.25]$.}
	\label{fig_2d_softmax}
\end{figure}

\section{Operations that Preserve Continuous (DR-)Submodularity}
\label{sec_ops_preserving_submodularity}

Continuous submodularity is preserved under various operations,
e.g.,
the sum of two continuous submodular  functions is submodular, non-negative combinations of continuous submodular functions are still submodular, and
a continuous submodular function multiplied by a positive scalar is still submodular.
In this section, we will study some general
submodularity preserving operations from the perspective of function composition. Then we will look at some exemplary applications resulting from these rules.

\begin{observation}[\cite{bach2015submodular}]\label{lemma_restriction}
Let $f$ be a DR-submodular function over $\X = \prod_{i=1}^n \X_i$.
Let $\tilde{f}$ be the function defined by restricting $f$ on a product of subsets of $\X_i$. Then $\tilde{f}$ is DR-submodular.
\end{observation}
\Cref{lemma_restriction} will be useful when we try to obtain discrete submodular functions by discretizing continuous submodular functions.

\subsection{Function Composition}

Suppose there are  two functions $h: \R^m \rightarrow \R^n$ and $f: \R^n \rightarrow \R$.
Consider the composed function $g(\x) :=f(h(\x)) = (f \circ h)(\x)$. We are interested in what properties are needed from $f$ and $h$ such that the composed function $g$ is DR-submodular.

Here $h$ is a multivariate vector-valued  function. Equivalently, we can express $h$ as $n$ multivariate functions $h^k: \R^m \rightarrow \R,
k=1,...,n$.
We use  $\nabla h$ to denote the $n\times m$ Jacobian matrix of $h$. Let $\y = h(\x)$, so $y_k = h^k(\x)$.

For a vector-valued function, we define its (DR)-submodularity as,
\begin{definition}[(DR-)submodularity for vector-valued functions]
Let $h: \R^m \rightarrow \R^n$ be a multivariate vector-valued  function, and  $h^k: \R^m \rightarrow \R$ be the \ith{k} entry of the output, $k=1,...,n$. Then we say $h$ is (DR-)submodular iff $h^k$ is (DR-)submodular, $\forall k\in [n]$.
\end{definition}

Assume for simplicity that both $f$ and $h$ are twice differentiable. Applying the chain rule twice, one can
verify that

\begin{align}
\nabla^2 g(\x) = \nabla h (\x)^\trans  \nabla^2 f(\y) \nabla h (\x) + \sum_{k=1}^n \fracpartial{f(\y)}{y_k}  \nabla^2 h^k(\x),
\end{align}
where the product above is the standard matrix multiplication.
After some manipulation, one can see that the \ith{(i,j)} entry of $\nabla^2 g(\x)$ is,

\begin{align}\label{eq_ijth}
\fracppartial{g(\x)}{x_i}{x_j}
= \sum_{s,  t = 1}^n  \fracppartial{f(\y)}{y_s}{y_t}  \fracpartial{h^s(\x)}{x_i} \fracpartial{h^t(\x)}{x_j}
+ \sum_{k=1}^n \fracpartial{f(\y)}{y_k}  \fracppartial{h^k(\x)}{x_i}{x_j}.
\end{align}

Maintaining DR-submodularity (or IR-supermodularity) means
maintaining the sign of $\fracppartial{g(\x)}{x_i}{x_j}$.
From \cref{eq_ijth}, one can see that if we want $\fracppartial{g(\x)}{x_i}{x_j}$ to be non-positive, $h$ must in general be monotone.  $h$ could be either nondecreasing or nonincreasing, in both cases we have $\fracpartial{h^s(\x)}{x_i} \fracpartial{h^t(\x)}{x_j} \geq 0$.

\begin{theorem}[DR-submodularity preserving conditions on function composition]
\label{prop_composition}
Suppose  $h: \R^m \rightarrow \R^n$ is monotone (nondecreasing or nonincreasing),
$f: \R^n \rightarrow \R$.
The following statements about the  composed function $g(\x) :=f(h(\x)) = (f \circ h)(\x)$ hold:

\begin{enumerate}
	\item  If $f$ is DR-submodular, nondecreasing, and $h$ is DR-submodular, then $g$ is DR-submodular;

	\item  If $f$ is DR-submodular, nonincreasing, and $h$ is \irsuper, then $g$ is DR-submodular;

	\item  If $f$ is \irsuper, nondecreasing, and $h$ is \irsuper, then $g$ is \irsuper;

	\item  If $f$ is \irsuper, nonincreasing, and $h$ is DR-submodular, then $g$ is \irsuper.
\end{enumerate}
\end{theorem}

If $f$ and $h$ are both twice differentiable, \cref{prop_composition} can be directly proved by examining the  \ith{(i,j)} entry of $\nabla^2 g(\x)$  in
\cref{eq_ijth}.
\emph{Furthermore, the above conclusions can also be rigorously proved when the functions are non-differentiable.} Bellow we give an exemplar proof  of statement 1 in \cref{prop_composition}. The other proofs are omitted due to high  similarity.

\paragraph{Proof of statement 1 in  \cref{prop_composition} when the functions are non-differentiable}

\begin{proof}[Proof of \cref{prop_composition} when the functions are non-differentiable]

To prove the DR-submodularity of $g$, it suffices to show that:
\begin{align}\label{dr_of_g}
\forall \x \leqco \y, \forall i\in [m],  \forall k \geq 0, g(\x + k\bas_i) - g(\x)\geq g(\y +k\bas_i) - g(\y).
\end{align}

Due to DR-submodularity of $h$,
\begin{align}\label{eq88}
h(\x+k\bas_i) - h(\x) \geqco h(\y + k\bas_i) - h(
\y)
\end{align}

I) Let us consider the case when $h$ is nondecreasing.
It holds,
\begin{align}\label{eq9}
	h(\x) \leqco h(\y)
\end{align}

Then,
\begin{align}
&g(\x + k\bas_i) - g(\x)\\
& = f[h(\x+k\bas_i)] - f[h(\x)] \\
& \geq   f[h(\x) + h(\y+k\bas_i)  - h(\y)    ] - f[h(\x)]  \quad  \text{\labelcref{eq88} and $f$ is non-decreasing}   \\
& \geq   f[ h(\y+k\bas_i)  ] - f[h(\y)]  \quad \text{\labelcref{eq9} and $f$ is DR-submodular}   \\
& =
g(\y +k\bas_i) - g(\y).
\end{align}

Thus we prove \cref{dr_of_g}, i.e., the DR-submodularity of $g$.

II) Let us consider the case when $h$ is nonincreasing.
It holds,
\begin{align}\label{eq117}
h(\x + k\bas_i) \geqco h(\y + k\bas_i)
\end{align}

Thus,
\begin{align}
&  g(\y) - g(\y +k\bas_i)\\
& =  f[h(\y)]  - f[ h(\y+k\bas_i)  ]\\
& \geq   f[ h(\y+k\bas_i) + h(\x)   - h(\x+k\bas_i )    ] - f[h(\y  + k\bas_i)]  \quad  \text{\labelcref{eq88} \& $f$ is nondecreasing}   \\
& \geq f[h(\x)]  -  f[h(\x+k\bas_i)] \quad \text{\labelcref{eq117} and $f$ is DR-submodular} \\
& =  g(\x) - g(\x + k\bas_i).
\end{align}
\end{proof}

By examining the  \ith{(i,j)} entry of $\nabla^2 g(\x)$  in
\cref{eq_ijth}, we can also prove the following conclusion:

\begin{lemma}\label{lemma_separable_reparameterization}
Suppose $h$ is monotone (nondecreasing or nonincreasing). In addition, assume $h$ is separable, that is, $m=n$ and $h^k(\x) = h^k(x_k), k = 1,...,n$.  Then $f(h(\x))$ maintains submodularity (supermodularity) of $f$.
\end{lemma}
The detailed proof can be found in \cref{app_proof_separable_reparameterization}. It is worth noting that under the same setting as in \cref{lemma_separable_reparameterization}, $f(h(\x))$ might not  maintain DR-submodularity (IR-supermodularity) of $f$.

\subsection{Examples of Function Composition}

\paragraph{Design deep submodular (set or integer) functions (DSFs).}

Using conclusions in this section, we obtain a general way of composing  discrete DSFs: i) We make a continuous deep  submodular function $f: \X \rightarrow \R$ utilizing the composition rules; ii) By restricting $f$ to the binary lattice $\{0, 1\}^n$, we obtain a deep submodular set function. Similarly, by restricting $f$ on the integer lattice $\{0, 1, 2, ..., k\}^n$, we get  a deep submodular integer function. This step is ensured by the restriction rule (\cref{lemma_restriction}).

For a specific example, we can easily prove that the DSFs composed by nesting SCMMs with concave functions \citep{bilmes2017deep} are binary submodular.

\begin{figure}
	\centering
	\includegraphics[width=0.4\textwidth]{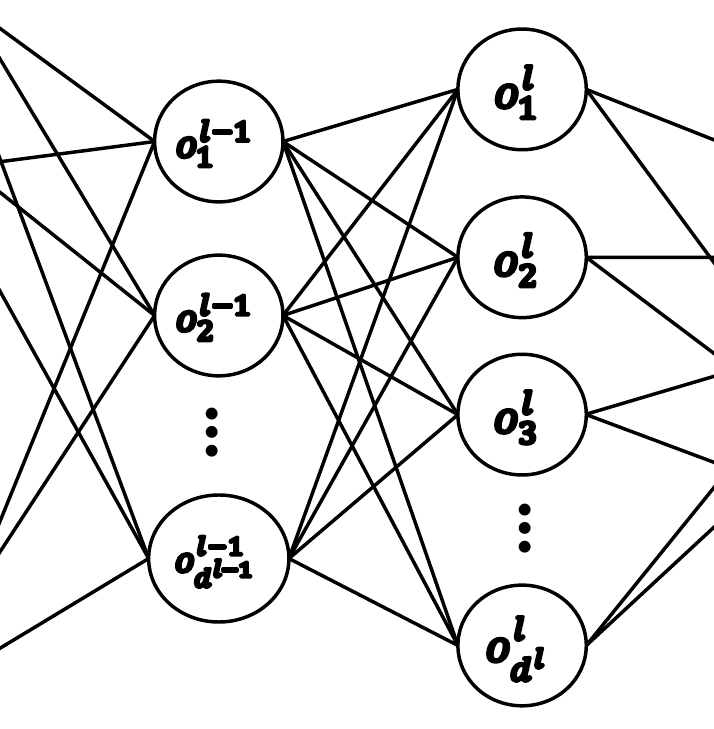}
	\caption{Layers $l-1$ and $l$ of the DSF.}
	\label{fig_dsf}
\end{figure}

Firstly, we can prove that the continuous function
composed by nesting SCMMs with concave functions \citep{bilmes2017deep} are continuous submodular. Let the original input vector be $\x\in \R^n$, which serves as the input vector of the $0^\text{th}$ layer. As shown by \cref{fig_dsf}, let the output of the \ith{i} neuron in the \ith{l} layer be $o^l_i$.  So
\begin{align}\label{eq23}
o^l_i = \sigma ( W^l_{1,1} o^{l-1}_1 +  W^l_{1,2} o^{l-1}_2 + ... +  W^l_{1,d^{l-1}} o^{l-1}_{d^{l-1} } ).
\end{align}
where we assume there are $d^l$ neurons in layer $l$, and use $\BW^l \in \R_+^{d^l \times d^{l-1}}$ to denote the weight matrix between layer $l-1$ and layer $l$.  $\sigma$ is the activation function which is concave and \nondec in the positive orthant.

Now let us proceed by induction. When $l=0$,
$o^0_i(\x)$ is DR-submodular and \nondec since
$o^0_i(\x) = x_i$. Let us assume for the \ith{(l-1)} layer that
$o^{l-1}_i(\x)$ is DR-submodular and \nondec wrt. $\x$. It is left to prove that $o^{l}_i(\x)$ is  DR-submodular and \nondec wrt. $\x$. According to \cref{eq23}, $o^l_i (\x) = \sigma ( W^l_{1,1} o^{l-1}_1(\x) +  W^l_{1,2} o^{l-1}_2(\x) + ... +  W^l_{1,d^{l-1}} o^{l-1}_{d^{l-1} }(\x) )$. Since $\BW$ is non-negative, the function $h(\x) =  W^l_{1,1} o^{l-1}_1(\x) +  W^l_{1,2} o^{l-1}_2(\x) + ... +  W^l_{1,d^{l-1}} o^{l-1}_{d^{l-1} }(\x) $ is DR-submodular and \nondec. $\sigma$ is also DR-submodular and \nondec. According to \cref{prop_composition},  $o^l_i (\x)$ is also DR-submodular wrt $\x$. Thus we finish the induction.

Given that  continuous DSFs are continuous submodular, by means of the restriction operation, we obtain binary or integer deep submodular functions.

However, the above principles offer more general ways of designing DSFs, other than nesting SCMMs with concave activation functions (as proposed by \citet{bilmes2017deep}). As long as the resultant continuous map is DR-submodular, the discrete function obtained by restriction will be DR-submodular.
With the principles proved in this section, one can immediately recognize that the following applications enjoy continuous submodular objectives (more details will be discussed in the corresponding sections).

\paragraph{Influence Maximization with Marketing Strategies.}

One can easily see that the objective in \cref{influence_general_marketing} is the composition of a nondecreasing multilinear extension and a monotone activation function. So it is DR-submodular according to \cref{prop_composition}.

\paragraph{Revenue Maximization with Continuous Assignments.}
One can also verify that the revenue maximization problem in \cref{revenue_max_maxcut} is the composition of a non-monotone DR-submodular multilinear extension and a separable monotone function, so it is still DR-submodular according to \cref{lemma_separable_reparameterization}.

\section{Properties of  Constrained  DR-Submodular  Maximization}
\label{sec_structures}

In this section, we  first formulate the constrained DR-submodular
maximization problem, and then establish several properties of it. In particular, we show properties related to    concavity  of the objective along certain directions, and establish  the
relation between  locally
stationary points and the global
optimum (thus called ``local-global relation'').
These properties will be used to derive guarantees for the algorithms in the following sections. All omitted proofs are in \cref{app_proofs_struc_algs}.

\subsection{The Constrained (DR-)Submodular Maximization Problem}

The general setup of constrained
continuous
submodular function  maximization is,
\begin{align}\label{setup}
\max_{\x \in \P\subseteq \X} f(\x),     \tag{P}
\end{align}
where $f: \X \rightarrow \R$ is continuous submodular or
DR-submodular, $\X = [\underline{\u}, \bar \u]$
\citep{bian2017guaranteed}.  One can assume $f$ is non-negative over
$\X$, since otherwise one just needs to find a lower bound for the
minimum function value of $f$ over $\X$ (because box-constrained
submodular minimization can be solved to arbitrary precision in
polynomial time \citep{bach2015submodular}). Let the lower bound be
$f_{\text{min}}$, then working on a new function
$f'(\x):= f(\x) - f_{\text{min}}$ will not change the solution
structure of the original problem \labelcref{setup}.

The constraint set $\P\subseteq \X$ is assumed to be a
\emph{down-closed} {convex} set, since without this property one
cannot reach any constant factor approximation guarantee of the
problem \labelcref{setup} \citep{vondrak2013symmetry}. Formally,
down-closedness of a convex set is defined as follows:
\begin{definition}[Down-closedness]
	A down-closed convex set is a convex set $\P$ associated with a
	lower bound $\underline{\bu}\in \P$, such that:
	\begin{enumerate}
		\item  $\forall \y\in \P$, $\underline{\bu} \leqco \y$;

		\item $\forall \y\in\P$, $\x\in \R^n$,
		$\underline{\bu} \leqco \x\leqco \y$ implies that $\x\in \P$.
	\end{enumerate}
\end{definition}

Without loss of generality, we assume $\P$ lies in the positive
orthant and has the lower bound $\zero$.  Otherwise we can always
define a new set
$\P' = \{\x \;|\; \x = \y - \underline{\bu}, \y\in \P \}$ in the
positive orthant, and a corresponding continuous submodular function
$f'(\x) := f(\x + \underline{\bu})$, and all properties of the
function are still preserved.

The diameter of $\P$ is $D:= \max_{\x,\y\in\P}\|\x-\y\|$, and it holds
that $D \leq \|\bar \u \|$.  We use $\x^*$ to denote the
global maximum of \labelcref{setup}.
In some applications we know that $f$ satisfies the monotonicity property:

\begin{definition}[Monotonicity]
	A function $f(\cdot)$  is monotone nondecreasing if,
	\begin{align}
	\forall \a \leqco \b, f(\a) \leq f(\b).
	\end{align}
	In the sequel, by ``monotonicity'', we mean monotone nondecreasing by default.
\end{definition}

We also assume that $f$ has
Lipschitz gradients,

\begin{definition}[Lipschitz gradients]
	A differentiable function $f(\cdot)$ has $L$-Lipschitz gradients if
	for all $\x,\y \in \X$ it holds that,
	\begin{align}\label{eq_smooth}
	\| \nabla f(\x)- \nabla f(\y) \| \leq L \|\x - \y\|.
	\end{align}
	According to \citet[Lemma 1.2.3]{nesterov2013introductory}, if
	$f(\cdot)$ has $L$-Lipschitz gradients, then
	\begin{align}\label{eq_quad_lower_bound}
	|f(\x + \v) - f(\x)  - \dtp{\nabla f(\x)}{\v}| \leq \frac{L}{2} \|\v\|^2.
	\end{align}
\end{definition}

For Frank-Wolfe style algorithms, the notion of curvature usually
gives a tighter bound than just using the Lipschitz gradients.
\begin{definition}[Curvature of a continuously differentiable function]
	The curvature of a differentiable function $f(\cdot)$ w.r.t. a
	constraint set $\P$ is,
	\begin{align}
	C_f(\P) : = \sup_{\x, \v\in \P, \gamma\in (0,  1],  \y = \x +
		\gamma (\v - \x )}\frac{2}{\gamma^2}\left[f(\y) - f(\x) - {(\y -
		\x)^\trans}{\nabla f(\x)}\right].
	\end{align}
\end{definition}

If a differentiable function $f(\cdot)$ has $L$-Lipschitz gradients,
one can easily show that $C_f(\P) \leq LD^2$, given \citet[Lemma
1.2.3]{nesterov2013introductory}.

\subsection{Properties Along Non-negative/Non-positive Directions}

Though in general a DR-submodular function $f$ is neither convex, nor
concave, it is \emph{concave along some directions}:

\begin{proposition}[\cite{bian2017guaranteed}]\label{prop_concave}
	A continuous {DR}-submodular function $f(\cdot)$ is concave along
	any non-negative direction $\v \geqco \zero$, and any non-positive
	direction $\v \leqco \zero$.
\end{proposition}

Notice that DR-submodularity is a stronger condition than concavity
along directions $\v \in \pm \R_+^n$: for instance, a concave function
is concave along any direction, but it may not be a DR-submodular
function.

\paragraph{Strong DR-submodularity.}
DR-submodular objectives may be strongly concave along directions
$\v \in \pm \R_+^n$, e.g., for DR-submodular quadratic functions.  We
will show that such additional structure may be exploited to obtain
stronger guarantees for the local-global relation.

\begin{definition}[Strong DR-submodularity] \label{eq_strongly_dr} A
	function $f$ is $\mu$-strongly DR-submodular ($\mu\geq 0$) if for
	all $\x\in \X$ and $\v \in \pm \R_+^n$, it holds that,
	\begin{align}\label{eq_strong_dr} f(\x+\v) \leq f(\x) +
	\dtp{\nabla f(\x)}{\v} - \frac{\mu }{2}\|\v\|^2.
	\end{align}
\end{definition}

\subsection{Local-Global Relation: Relation Between Approximately Stationary Points and Global Optimum}\label{subsec_local_global}

We know that for unconstrained  optimization problems,  $\|\nabla f(\x)\|$ is often used as  the non-stationarity measure of the point $\x$. What should be
a proper non-stationarity measure of a general
constrained optimization problem? We advocate the non-stationarity measure proposed by
\citet{lacoste2016convergence} and \citet{reddi2016stochastic}, which  can
be calculated for free within  {Frank-Wolfe}-style algorithms (e.g.,  \cref{nonconvex_fw}).

\paragraph{Non-stationarity measure.}

For any constraint set $\Q\subseteq \X$, the non-stationarity
of a point $\x\in \Q$ is,
\begin{align}\label{non_stationary}
g_{\Q}(\x) := \max_{\v\in\Q}\dtp{\v - \x}{\nabla f(\x)} \qquad \text{(non-stationarity)}.
\end{align}
It always holds that $g_{\Q}(\x)\geq 0$, and $\x$ is defined to be  a
stationary point in $\Q$ iff $g_{\Q}(\x)=0$, so \labelcref{non_stationary} is a natural generalization of
the non-stationarity measure for unconstrainted optimization problems.

We start with the following proposition %
 involving the non-stationarity measure.
\begin{proposition}[\cite{biannips2017nonmonotone}]\label{lemma_3_1}
If $f$ is $\mu$-strongly DR-submodular, then for any two points $\x$,  $\y$ in $\X$, it holds:
\begin{align}\label{non_stationarity}
(\y-\x)^{\trans}\nabla f(\x) \geq f(\x\vee\y) + f(\x\wedge \y) - 2f(\x) + \frac{\mu}{2}\|\x -\y\|^2.
\end{align}
\end{proposition}
 Proposition \ref{lemma_3_1} implies that if $\x$ is  stationary in $\P$ (i.e., $g_{\P}(\x)=0$), then $2f(\x) \geq f(\x\vee\y) + f(\x\wedge \y)  + \frac{\mu}{2}\|\x -\y\|^2$, which gives an implicit relation
between $\x$ and  $\y$.

As the following statements show, $g_{\Q}(\x)$
plays an important role in characterizing the local-global
relation in both monotone and non-monotone setting.

\subsubsection{Local-Global Relation in the Monotone Setting}

\begin{corollary}[Local-Global Relation: \emph{Monotone Setting}]\label{coro_1half}
	Let $\x$ be a point in $\P$ with non-stationarity $g_{\P}(\x)$.  If
	$f$ is monotone nondecreasing and $\mu$-strongly DR-submodular, then
	it holds that,
	\begin{align}
	f(\x) \geq \frac{1}{2}\left[f(\x^*) -g_{\P}(\x) \right ]  +   \frac{\mu}{4}\|\x -\optcont\|^2.
	\end{align}
\end{corollary}

\Cref{coro_1half} indicates that any stationary point is a 1/2
approximation, which is also found by \citet{hassani2017gradient} (with
$\mu = 0$). Furthermore, if $f$ is $\mu$-strongly DR-submodular, the
quality of $\x$ will be improved considerably: if $\x$ is close to $\optcont$,
it should be close to being optimal since $f$ is smooth; if $\x$ is
far away from $\optcont$, the term $\frac{\mu}{4}\|\x -\optcont\|^2$
will boost the approximation bound significantly.
We provide here a very succinct proof based on Proposition \ref{lemma_3_1}.

\begin{proof}[Proof of \cref{coro_1half}]
	Let $\y = \optcont$ in \cref{lemma_3_1}, one can easily obtain
	\begin{align}
	f(\x) \geq \frac{1}{2}\left[f(\x^* \vee \x) + f(\optcont \wedge
	\x) -g_{\P}(\x) \right ]  +   \frac{\mu}{4}\|\x -\x^*\|^2.
	\end{align}

	Because of monotonicity and $\x^* \vee \x \geqco \x^*$, we know that
	$f(\x^* \vee \x)\geq f(\x^*)$. From non-negativity,
	$f(\optcont \wedge \x) \geq 0$.  Then we reach the conclusion.
\end{proof}

\subsubsection{Local-Global Relation in the Non-Monotone Setting}

\begin{proposition}[Local-Global Relation:
	\emph{Non-Monotone Setting}]\label{local_global}
	Let $\x$ be a point in $\P$ with non-stationarity $g_{\P}(\x)$, and
	${\Q} :=\P \cap  \{\y | \y\leqco \bar \u - \x\}$.
	Let $\z$ be a point in $\Q$ with non-stationarity $g_{\Q}(\z)$.  It
	holds that,
	\begin{flalign}
	&\max\{f(\x), f(\z) \} \geq \\\notag
	&\frac{1}{4}\left[f(\x^*)
	-g_{\P}(\x) -g_{\Q}(\z)\right ] + \frac{\mu}{8}\left(\|\x
	-\x^*\|^2 + \|\z -\z^*\|^2\right ),
	\end{flalign}
	where $\z^*:= \x\vee \x^* -\x$.
\end{proposition}
\cref{fig_local_global_nonmotone} provides a two-dimensional visualization of \cref{local_global}.
Notice that the smaller constraint $\Q$ is generated after the first stationary point $\x$ is calculated.
\begin{figure}[htbp]
	\centering
	\includegraphics[width=0.6\textwidth]{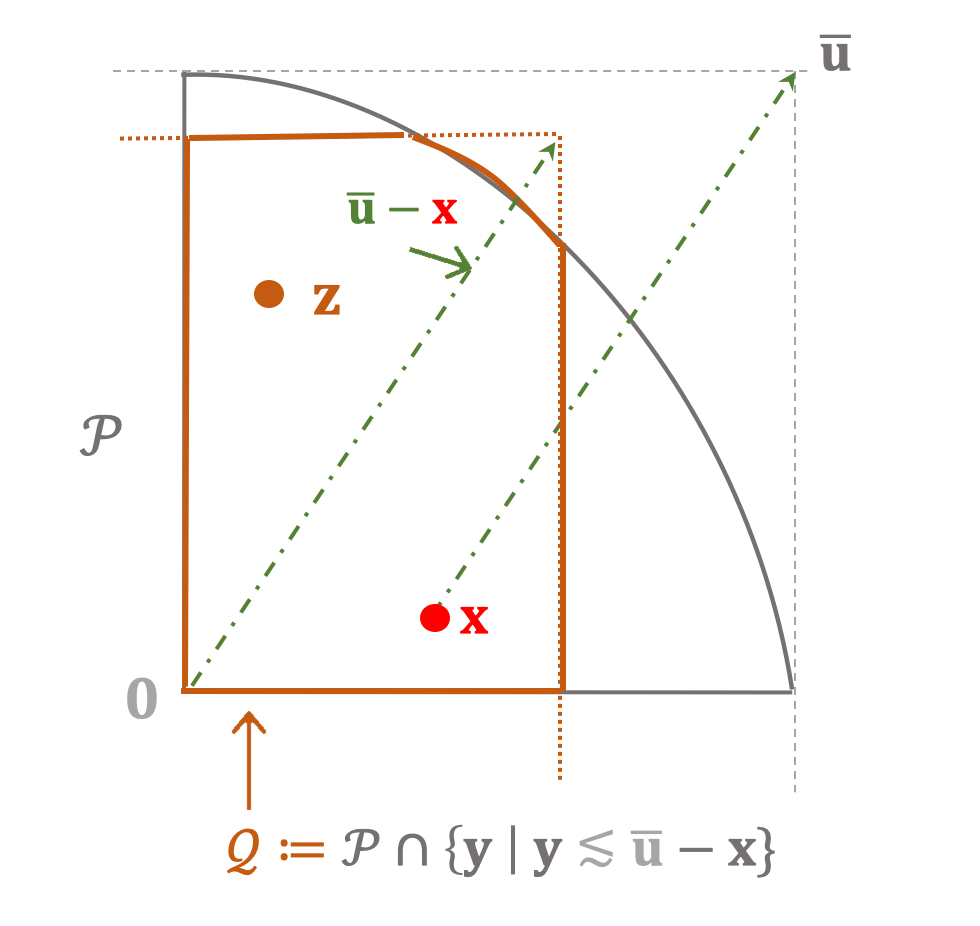}
	\caption{Visualization of the local-global relation in the non-monotone setting.}
	\label{fig_local_global_nonmotone}
\end{figure}

\paragraph{Proof sketch of Proposition \ref{local_global}:}
The proof uses \cref{lemma_3_1}, the non-stationarity
in \labelcref{non_stationary} and a key observation in the following
\namecref{claim_key}.  The detailed proof is deferred to
\cref{app_claim_proof}.
\begin{restatable}[]{claim}{keyclaim}
	\label{claim_key}
	Under the setting of \cref{local_global}, it holds that,
	\begin{align}
	f(\x\vee \x^*) + f(\x \wedge \x^*) +  f(\z\vee \z^*) + f(\z \wedge \z^*) \geq f(\x^*).
	\end{align}
\end{restatable}
Note that \citet{chekuri2014submodular,gillenwater2012near} propose a
similar relation for the special cases of the multilinear/softmax
extensions by mainly proving the same conclusion as in
\cref{claim_key}. Their relation does not incorporate the properties
of non-stationarity or strong DR-submodularity.  They both use the
proof idea of constructing a specialized auxiliary set function
tailored to specific DR-submodular functions (the considered extensions).  We present a different
proof method by directly utilizing the DR property on carefully
constructed auxiliary points (e.g., $(\x+\z)\vee \x^*$ in the proof of
\cref{claim_key}), which is arguably more succinct and straightforward
than that of \citet{chekuri2014submodular,gillenwater2012near}.

\section{Exemplary Applications   of Continuous Submodular Optimization}\label{sec_app}

Continuous submodularity naturally finds applications in various
domains, ranging from
influence and revenue maximization, to DPP MAP inference and mean
field inference of probabilistic graphical models. We
discuss several concrete problem instances in this section.

\subsection{Submodular Quadratic Programming (SQP)}

Non-convex/non-concave QP problems of the form
$f(\x) = \frac{1}{2} \x^\trans \bmH \x + \bh^\trans \x + c$ under
convex constraints naturally arise in many applications, including
scheduling \citep{DBLP:journals/jacm/Skutella01}, inventory theory,
and free boundary problems.  A special class of QP is the submodular
QP (the minimization of which was studied in \citet{kim2003exact}), in
which all off-diagonal entries of $\bmH$ are required to be
non-positive.  {Price optimization} with continuous prices is a
DR-submodular quadratic program \citep{ito2016large}.

Another representative  class of  DR-submodular quadratic objectives
arise when computing the {stability number}  $s(G)$ of a graph $G= (V, E)$,
${s(G)}^{-1} = \min_{\x\in \Delta}\x^\trans (\bmA + \bmI)\x$,
where $\bmA$ is the adjacency matrix of the graph $G$, $\Delta$ is the
standard simplex \citep{motzkin1965maxima}.  This instance is a
convex-constrained monotone DR-submodular maximization problem.

\subsection{Continuous Extensions of Submodular Set Functions}
\label{sec_app_multilinear}

The {\em Lov{\'a}sz  extension} \citep{lovasz1983submodular} used for
submodular set function minimization is both submodular and convex
(see Appendix A of \citet{bach2015submodular}).

The {\em multilinear extension} \citep{calinescu2007maximizing} is
extensively used for submodular set function maximization.  It is the
expected value of $F(S)$ under the fully factorized  surrogate distribution
$q(S|{\x}):= \prod_{i\in S}x_i \prod_{j\notin S}(1-x_j),
\x\in[0,1]^\groundset$:
\begin{align}\label{eq_multilinear_ext}
\multi(\x) := \E_{q(S\mid \x)} [F(S)] =\sum_{S\subseteq
	\groundset}F(S)\prod_{i\in S}x_i \prod_{j\notin S}(1-x_j).
\end{align}
$\multi(\x)$ is DR-submodular and coordinate-wise linear
\citep{bach2015submodular}.
The partial derivative of $\multi(\x)$ can be expressed as,
\begin{align}
\nabla_i \multi(\x)  &
= \E_{q(S\mid \x, x_i = 1)} [F(S)]  - \E_{q(S\mid \x, x_i=0)} [F(S)]\\\notag
& =\multi(\sete{x}{i}{1}) - \multi(\sete{x}{i}{0})
\\\notag
&  =  \sum_{S\subseteq \groundset, S\ni i } F(S)
\prod_{j \in S\backslash\{i\}}x_j
\prod_{j'\notin S}(1-x_{j'})\\\notag
& \quad  - \sum_{S\subseteq
	\groundset\backslash \{i\} }\ F(S) \prod_{j\in
	S} x_j \prod_{j'\notin S, j'\neq i}
(1-x_{j'}).
\end{align}

At the first glance, evaluating the multilinear extension
in \cref{eq_multilinear_ext} costs
an exponential number of
operations. However, when used in practice, one can often use sampling
techniques to estimate its value and gradient. Furthermore, it is
worth noting that for several classes of practical submodular set
functions, their multilinear extensions $\multi(\cdot)$ admit closed form
expressions.  We present details in the following.

\subsubsection{Gibbs Random Fields}\label{gibbs_multilinear}

Let us use $\v \in \{0,1\}^\groundset  $ to equivalently
denote the $n$ binary random variables in a Gibbs random field.
$F(\v)$ corresponds to the negative energy function
in Gibbs random fields. If the energy function is
parameterized with a finite order of interactions, i.e.,
$F(\v) = \sum_{s\in \groundset} \theta_s v_s + \sum_{(s,t)\in
	\groundset \times \groundset} \theta_{s, t}v_s v_t + ... +
\sum_{(s_1, s_2, ..., s_d)} \theta_{s_1, s_2, ..., s_d}v_{s_1} \cdots
v_{s_d},  \; d< \infty$, then one can verify that its
multilinear extension has the following closed form,
\begin{align}
\multi(\x)
= \sum_{s\in \groundset} \theta_s x_s + \sum_{(s,t)\in \groundset
	\times \groundset} \theta_{s, t}x_s x_t + ...  \\\notag
+  \sum_{(s_1, s_2,
	..., s_d)} \theta_{s_1, s_2, ..., s_d}x_{s_1} \cdots  x_{s_d}\,.
\end{align}

The gradient of this expression can also be easily derived.  Given
this observation, one can quickly  derive the multilinear extensions
of a large category of energy functions of Gibbs random fields, e.g.,
graph cut, hypergraph cut, Ising models, etc.
Specifically,

\paragraph{Undirected  \maxcut.}
For undirected \maxcut, its objective is
$F(\v) = \frac{1}{2}\sum_{(i,j)\in E} w_{ij} (v_i + v_j -2v_i v_j), \v
\in \{0,1\}^\groundset $.
One can verify that its multilinear extension is
$\multi(\x) = \frac{1}{2}\sum_{(i,j)\in E} w_{ij} (x_i + x_j -2x_i
x_j), \x \in [0,1]^\groundset $.

\paragraph{Directed \maxcut.} For directed \maxcut, its objective is
$F(\v) = \sum_{(i,j)\in E} w_{ij} v_i (1- v_j), \v \in
\{0,1\}^\groundset $.
Its multilinear extension is
$\multi(\x) = \sum_{(i,j)\in E} w_{ij} x_i (1- x_j), \x \in
[0,1]^\groundset $.

\paragraph{Ising models.}
For Ising models \citep{ising1925contribution} with non-positive pairwise interactions (antiferromagnetic interactions),
$F(\v) = \sum_{s\in \groundset} \theta_s v_s + \sum_{(s,t)\in E}
\theta_{st}v_s v_t$,
$\v\in \{0, 1 \}^\groundset$, this objective can be easily verified to
be submodular.
Its multilinear extension is:
\begin{align}
\multi(\x)= \sum_{s\in \groundset} \theta_s x_s + \sum_{(s,t)\in E}
\theta_{st}x_s x_t, \x \in [0,1]^\groundset.
\end{align}

\subsubsection{Facility Location and FLID (Facility Location Diversity)}

FLID is a diversity model \citep{Tschiatschek16diversity} that has
been designed as a computationally efficient alternative to DPPs
\citep{kulesza2012determinantal}.  It is based on
the facility location objective.  Let
$\BW\in \R_+^{|\groundset|\times D }$ be the weights, each row
corresponds to the latent representation of an item, with $D$ as the
dimensionality. Then
\begin{align}\notag
F(S) := & \sum\nolimits_ {i\in S} u_i + \sum\nolimits_{d=1}^{D} (
\max_{i\in S} W_{i,d} - \sum\nolimits_{i\in S} W_{i,d} ) \\
=& \sum\nolimits_ {i\in S} u'_i + \sum\nolimits_{d=1}^{D}
\max_{i\in S}W_{i,d},
\end{align}
which models both coverage and diversity, and
$u'_i = u_i - \sum_{d=1}^{D} W_{i,d}$. If $u'_i=0$, one recovers the
facility location objective.
The computational complexity of evaluating its partition function is
$\bigo{|\groundset|^{D+1}}$ \citep{Tschiatschek16diversity}, which is
exponential in terms of $D$.

We now show the technique such that $\multi(\x)$ and
$\nabla_i \multi(\x) $ can be evaluated in $\bigo{Dn^2}$ time.
Firstly, for one $d\in [D]$, let us sort $W_{i,d}$ such that
$W_{i_d(1), d} \leq W_{i_d(2), d} \leq \cdots \leq W_{i_d(n), d} $.
After this sorting, there are $D$ permutations to record:
$i_d(l), l=1,...,n, \forall d\in [D]$.  Now, one can verify that
\begin{align}
\multi(\x)   \notag
& =  \sum_ {i \in [n]} u'_i x_i +  \sum_d
\sum_{S\subseteq \groundset }  \max_{i \in S} W_{i, d}
\prod_{m\in S}x_m \prod_{m'\notin S}(1-x_{m'}) \\\notag
& = \sum_ {i\in [n]} u'_i x_i + \sum_{d} \sum_{l=1}^n
W_{i_d(l), d} x_{i_d(l)} \prod_{m=l+1}^n [1-
x_{i_d(m)}].
\end{align}
Sorting costs $\bigo{Dn\log n}$, and from the above expression, one
can see that the cost of evaluating $\multi(\x)$ is $\bigo{Dn^2}$. By the
relation that
$\nabla_i \multi(\x) = \multi(\sete{x}{i}{1}) -
\multi(\sete{x}{i}{0})$,
the cost is also $\bigo{Dn^2}$.

\subsubsection{Set Cover Functions}
\label{supp_setcover}
Suppose there are $|C| = \{c_1, ...,c_{|C|}\}$ concepts, and $n$ items
in $\groundset$. Give a set $S\subseteq \groundset$, $\Gamma (S)$
denotes the set of concepts covered by $S$. Given a modular function
$\m: 2^C \mapsto \R_+ $, the set cover function is defined as
$F(S) = \m (\Gamma(S))$.
This function models coverage in maximization,
and also the notion of complexity in minimization problems \citep{lin2011optimal}.
Let us define an inverse map $\Gamma^{-1}$, such that for
each concept $c$, $\Gamma^{-1}(c)$ denotes the set
of items $v$ such that $\Gamma^{-1}(c) \ni v$. So the
multilinear extension is,
\begin{align}\notag
\multi(\x)  & =  \sum\nolimits_{i \in \groundset}  \m (\Gamma(S))  \prod\nolimits_{m\in S}x_m \prod\nolimits_{m'\notin S }(1-x_{m'}) \\
& =  \sum\nolimits_ {c\in C}  m_c \left[  1- \prod\nolimits_{ i\in \Gamma^{-1}(c) }  (1- x_i) \right].
\end{align}
The last equality is achieved by considering the situations
where a concept $c$ is covered.
One can observe that both $\multi(\x)$ and $\nabla_i \multi(\x) $ can
be evaluated in $\bigo{n|C|}$ time.

\subsubsection{General Case: Approximation by Sampling}

In the most general case, one may only have access to the function values of $F(S)$.
In this scenario, one can use a polynomial number of sample steps to estimate
$\multi(\x)$ and its gradients.

Specifically: 1) Sample $k$ times
$S \sim q(S|\x)$ and evaluate function values for them, resulting in
$F(S_1), ...,F(S_k)$.  2) Return the average
$\frac{1}{k}\sum_{i=1}^{k} F(S_i)$. According to the Hoeffding bound
\citep{hoeffding1963probability}, one can easily derive that
$\frac{1}{k}\sum_{i=1}^{k} F(S_i)$ is arbitrarily close to
$\multi(\x)$ with increasingly more samples: With probability at least
$1- \exp(-k\epsilon^2/2)$, it holds that
$|\frac{1}{k}\sum_{i=1}^{k} F(S_i) - \multi(\x)| \leq \epsilon \max_S
|F(S)| $, for all $\epsilon > 0$.

\subsection{Influence Maximization with Marketing Strategies}
\label{app_influence_max_marketing_strategies}

\citet{kempe2003maximizing} propose a general marketing strategy
for influence maximization. They
assume that there exists a number $m$ of different marketing actions
$M_i$, each of which may affect some subset of nodes by increasing
their probabilities of being activated.  A natural requirement would be
that the more we spend on any one action, the stronger should be its
effect.
Formally, one chooses $x_i$ investments to marketing action $M_i$, so
a marketing strategy is an $m$-dimensional vector $\x\in \R^m$.  Then
the probability that node $i$ will become activated is described by
the activation function: $a^i(\x): \R^m \rightarrow [0, 1]$. This
function should satisfy the DR property by assuming that any
marketing strategy is more effective when the targeted individual is
less ``marketing-saturated'' at that point.

Now we search for the expected size of the final active set, which
is the expected influence. We know that given a marketing strategy
$\x$, a node $i$ becomes active with probability $a^i(\x)$, so the
expected influence is:
\begin{align}\label{influence_general_marketing}
f(\x) = \sum_{S\subseteq V} F(S) \prod_{i\in S} a^i(\x)
\prod_{j\notin S} (1 - a^j(\x)).
\end{align}
$F(S)$ is the influence with the seeding set as $S$. It is submodular
for many influence models, such as the Linear Threshold model and
Independent Cascade model of \citet{kempe2003maximizing}.  One can
easily see that \cref{influence_general_marketing} is DR-submodular by
viewing it as a composition of the multilinear extension of $F(S)$ and
the activation function $a(\x)$.

\subsubsection{Realizations of the Activation Function $a(\x)$}
\label{app_activations_influence_max}

For the activation function $a^i(\x)$, we consider two realizations:

\begin{enumerate}
	\item Independent marketing action.

	Here we provide one action for each customer, and different
	actions are independent. So we have $m = |V|$ actions, and for customer
	$i$, there exists an activation function $a^i(x_i)$, which is a one
	dimensional nondecreasing DR-submodular function. A specific
	instance is that $a^i(x_i) = 1 - (1 - p_i)^{x_i}$, $p_i \in [0, 1]$
	is the probability of customer $i$ becoming activated with one unit of
	investment.

	\item Bipartite marketing actions.

	Suppose there are $m$ marketing actions and $|V|$ customers.  The
	influence relationship among actions and customers are modeled as a
	bipartite graph $(M, V; W)$, where $M$ and $V$ are collections of
	marketing actions and customers, respectively, and $W$ is the collection
	of weights.  The edge weight, $p_{st}\in W$, represents the
	influence probability of action $s$ to customers $t$ by providing one
	unit of investment to action $s$.  So with a marketing strategy as
	$\x$, the probability of a customer $t$ being activated is
	$a^t(\x) = 1- \prod_{(s, t)\in W} \left(1-p_{st} \right)^{x_s}$.
	This is a nondecreasing DR-submodular function.

\end{enumerate}

One may notice that the independent marketing action is a special
case of bipartite marketing action.

\subsection{Optimal Budget Allocation with Continuous Assignments}

Optimal budget allocation is a special case of the influence
maximization problem.  It can be modeled as a bipartite graph
$(S,T; W)$, where $S$ and $T$ are collections of advertising channels
and customers, respectively.
The edge weight, $p_{st}\in W$, represents the influence probability
of channel $s$ to customer $t$.
The goal is to distribute the budget (e.g., time for a TV
advertisement, or space of an inline ad) among the source nodes, and
to maximize the expected influence on the potential customers
\citep{soma2014optimal,DBLP:conf/aaai/HatanoFMK15}.

The total influence of customer $t$ from all channels can be modeled
by a proper monotone DR-submodular function $I_t(\x)$, e.g.,
$I_t(\x) = 1- \prod_{(s, t)\in W} \left(1-p_{st} \right)^{x_s}$ where
$\x\in \R^S_+$ is the budget assignment among the advertising
channels.  For a set of $k$ advertisers, let $\x^i\in \R^S_+$ be
the budget assignment for advertiser $i$, and
$\x:= [\x^1,\cdots, \x^k]$ denote the assignments for all the
advertisers.  The overall objective is,
\begin{flalign}
& g(\x)= \sum\nolimits_{i=1}^k \alpha_i f(\x^i) ~\text{ with }~\\
& f(\x^i) :=\sum\nolimits_{t\in T} I_t(\x^i), \; \zero \leqco
\x^i\leqco \bar \bu^i , \forall i = 1,..., k,
\end{flalign}
which is monotone DR-submodular.

A concrete application arises when advertisers bid for search
marketing, i.e., where vendors bid for the right to appear alongside
the results of different search keywords.
Here, $x^i_s$ is the volume of advertisement space allocated to the
advertiser $i$ to show his ad alongside query keyword $s$.  The search
engine company needs to distribute the budget (advertising space) to
all vendors to maximize their influence on the
customers, %
while respecting various constraints. For example, each vendor has a
specified budget limit for advertising, and the ad space associated
with each search keyword can not be too large.
All such constraints can be formulated as a down-closed polytope $\P$,
hence the \submodularfw algorithm (\cref{alg_sfmax_GradientAscend} in \cref{sec_mono_dr_fun})  can be used to find an
approximate solution for the problem $\max_{\x\in \P} g(\x)$.

Note that one can flexibly add regularizers in designing
$I_t(\x^i)$ as long as it remains monotone DR-submodular.
For example, adding separable regularizers of the form
$\sum_s \phi(x^i_s)$ does not change off-diagonal entries of the
Hessian, and hence maintains submodularity. Alternatively, bounding
the second-order derivative of $\phi(x^i_s)$ ensures DR-submodularity.

\subsection{Softmax Extension for DPPs}

Determinantal point processes (DPPs) are probabilistic models of
repulsion, which have been used to model diversity in machine learning
\citep{kulesza2012determinantal}. The constrained MAP (maximum a
posteriori) inference problem of a DPP is an NP-hard combinatorial
problem in general. Currently, the methods with the best approximation
guarantees are based on either maximizing the multilinear extension
\citep{calinescu2007maximizing} or the softmax extension
\citep{gillenwater2012near}, both of which are continuous DR-submodular
functions.

The multilinear extension is given as an expectation over the original
set function values, thus evaluating the objective of this extension
requires expensive sampling in general.  In contrast, the {\em softmax extension} has a
closed form expression, which is more appealing from a
computational perspective.
Let $\bmL$ be the positive semidefinite kernel matrix of a DPP, its
softmax extension is:
\begin{flalign}\label{eq_softmax}
f(\x) = \log \text{det\;}[{\diag(\x)(\bmL-\bmI) +\bmI}], \x\in [0,1]^n,
\end{flalign}
where $\bmI$ is the identity matrix, $\diag(\x)$ is the diagonal
matrix with diagonal elements set as $\x$.
Its DR-submodularity can be established by directly applying
Lemma 3 of \citet{gillenwater2012near},  which immediately implies
that all  entries of  $\nabla^2  f$ are non-positive, so $f(\x)$
is continuous DR-submodular.

The problem of MAP
inference in DPPs corresponds to the problem $\max_{\x\in \P} f(\x)$,
where $\P$ is a down-closed convex constraint, e.g., a matroid
polytope or a matching polytope.

\subsection{Mean Field Inference for Probabilistic Log-Submodular Models}

Probabilistic log-submodular models \citep{djolonga14variational} are a class of
probabilistic models over subsets of a ground set $\groundset = [n]$,
where the log-densities are submodular set functions $F(S)$:
$p(S) = \frac{1}{\parti}\exp(F(S))$. The partition function
$ \parti = \sum_{S\subseteq \groundset}\exp(F(S))$ is typically hard to
evaluate.  One can use mean field inference to approximate $p(S)$ by
some factorized distribution
$q(S|\x):= \prod_{i\in S}x_i \prod_{j\notin S}(1-x_j), \x\in
[0,1]^n$,
by minimizing the distance measured w.r.t. the Kullback-Leibler
divergence between $q$ and $p$, i.e.,
$ \sum_{S\subseteq \groundset} q(S|\x)
\log\frac{q(S|\x)}{p(S)}$. It is,

\begin{align}
\text{KL}(\x) & =
-\sum_{S\subseteq \groundset}F(S)  \prod_{i\in S}x_i \prod_{j\notin S}(1-x_j) +\\\notag
&  \sum\nolimits_{i=1}^{n} [x_i\log x_i + (1-x_i)\log(1-x_i)] + \log \parti.
\end{align}

$ \text{KL}(\x)$ is \irsuper w.r.t. $\x$. To see this: The first term
is the negative of a multilinear extension, so it is \irsuper. The
second term is separable, and coordinate-wise convex, so it will not
affect the off-diagonal entries of $\nabla^2 \text{KL}(\x)$, it will
only contribute to the diagonal entries.  Now, one can see that all
entries of $\nabla^2 \text{KL}(\x)$
are non-negative, so $\text{KL}(\x)$ is \irsuper w.r.t. $\x$.
Minimizing the Kullback-Leibler divergence $\text{KL}(\x)$ amounts to
maximizing a DR-submodular function.

\subsection{Revenue Maximization with Continuous Assignments}
\label{sec_revenue_max}

Given a social connection graph with nodes denoting $n$ users and
edges encoding their connection strength, the viral marketing suggests to choose a small subset of buyers to
give them some product for free, to trigger a cascade of further
adoptions through ``word-of-mouth'' effects, in order to maximize the
total revenue \citep{hartline2008optimal}.  For some products (e.g.,
software), the seller usually gives away the product in the form of a
trial, to be used for free for a limited time period.  In this task,
except for deciding whether to choose a user or not, the sellers also
need to decide how much the free assignment should be, in which the
assignments should be modeled as continuous variables.
We call this problem \emph{revenue maximization with continuous
	assignments}.

We use a directed graph $G = (\groundset, E; \BW)$ to represent the
social connection graph. $\groundset$ contains all the $n$ users, $E$
is the edge set, and $\BW$ is the adjacency matrix. We treat the
undirected social connection graph as a special case of the directed
graph, by taking one undirected edge as two directed edge with the
same weight.

\subsubsection{A Variant of the Influence-and-Exploit (IE) Strategy}
\label{app_special_case_of_ie}

One model with ``discrete'' product assignments is considered by
\citet{soma2017non} and \citet{durr2019non}, motivated by the
observation that giving a user more free products increases the
likelihood that the user will advocate this product.
It can be treated as a
simplified variant of the Influence-and-Exploit (IE) strategy of
\citet{hartline2008optimal}.
Specifically:

\begin{itemize}
	\item \emph{Influence} stage: Each user $i$ that is given
	$x_i$ units of products for free becomes an advocate of
	the product with probability $1 - q^{x_i}$ (independently from other
	users), where $q\in (0, 1)$ is a parameter. This is consistent with
	the intuition that with more free assignment, the user is more
	likely to advocate the product.

	\item \emph{Exploit} stage: suppose that a set $S$ of users advocate
	the product while the complement set $\groundset \setminus S$ of
	users do not. Now the revenue comes from the users in
	$\groundset \setminus S$, since they will be influenced by the
	advocates with probability proportional to the edge weights. We use
	a simplified concave graph model \citep{hartline2008optimal} for the
	value function, i.e.,
	$v_j(S) = \sum_{i\in S} W_{ij}, j\in \groundset \setminus S$.
	Assume for simplicity that the users of $\groundset \setminus S$ are
	visited independently with each other.  Then the revenue is:
	\begin{align}
	R(S) =\sum_{j\in\groundset \setminus S} v_j(S) = \sum_{j\in
		\groundset\setminus S} \sum_{i\in S} W_{ij}.
	\end{align}
	Notice that $S$ is a random set drawn according to the distribution
	specified by the continuous assignment $\x$.
\end{itemize}

With this Influence-and-Exploit (IE) strategy, the expected revenue is
a function $f: \R_+^\groundset \rightarrow \R_+$, as shown below:

\begin{align}\notag
f(\x) & = \epe[S]{R(S)}= \epe[S]{\sum_{i\in S} \sum_{j\in \groundset\setminus
		S}W_{ij} } \\\label{revenue_max_maxcut}
& = \sum_{i\in \groundset} \sum_{j\in \groundset\setminus
	\{i\}} W_{ij} (1- q^{x_i})q^{x_j}.
\end{align}
According to \cref{lemma_separable_reparameterization}, one can  see that the above objective is submodular, since it is composed by the multilinear extension of $R(S)$ (which is continuous submodular) and the separable function $h: \R^\groundset \rightarrow \R^\groundset$, where $h^i(x_i) = 1 - q^{x_i}$.

\subsubsection{An Alternative Model}
\label{app_revenue_max_alternative}

In addition to the Influence-and-Exploit (IE) model, we also consider  an alternative model.  Assume there are $q$ products
and $n$ buyers/users, let $\x^i \in \R_+^n$  be the assignments of
product $i$ to the $n$ users, let $\x:= [\x^1,\cdots, \x^q]$ denote
the assignments for the $q$ products.  The revenue can be modeled as
$g(\x) = \sum_{i=1}^q f(\x^i)$ with
\begin{flalign}\label{eq_re}
f(\x^i) := \alpha_i \sum\nolimits_{s: x^i_s =0} R_s(\x^i) + \beta_i
\sum\nolimits_{t: x^i_t \neq 0} \phi (x^i_t) +\gamma_i
\sum\nolimits_{t: x^i_t \neq 0} \bar R_t(\x^i),\\\notag \zero \leqco
\x^i \leqco \bar \bu^i,
\end{flalign}
where $x^i_t$ is the assignment of product $i$ to user $t$ for free,
e.g., the amount of free trial time or the amount of the product
itself.  $R_s(\x^i)$ models revenue gain from user $s$ who did not
receive the free assignment. It can be some non-negative,
non-decreasing submodular function. $\phi (x^i_t)$ models revenue gain
from user $t$ who received the free assignment, since the more one
user tries the product, the more likely he/she will buy it after the
trial period.  $\bar R_t(\x^i)$ models the revenue loss from user $t$
(in the free trial time period the seller cannot get profits), which
can be some non-positive, non-increasing submodular function.
For products with continuous assignments, usually the cost of the
product does not increase with its amount, e.g., the product as a
software, so we only have the box constraint on each assignment. The
objective in \cref{eq_re} is generally
\emph{non-concave/non-convex}, and non-monotone submodular (see
\cref{supp_revenue} for more details).
\begin{lemma}\label{revenue}
	If $R_s(\x^i)$ is non-decreasing submodular and $\bar R_t(\x^i)$ is
	non-increasing submodular, then $f(\x^i)$ in \cref{eq_re} is
	submodular.
\end{lemma}

\subsection{Applications Generalized from the Discrete Setting}

Many discrete submodular problems can be naturally generalized to the
continuous setting with continuous submodular objectives.  The maximum
coverage problem and the problem of text summarization with submodular
objectives are among the examples \citep{lin2010multi}. We put details
in the sequel.

\subsubsection{Text Summarization}

Submodularity-based objective functions for text summarization perform
well in practice \citep{lin2010multi}.  Let $C$ be the set of all
concepts, and $\groundset$ be the set of all sentences.  As a
typical example, the concept-based summarization aims to find a subset
$S$ of the sentences to maximize the total credit of concepts covered
by $S$. \citet{soma2014optimal} considered extending the submodular
text summarization model to one that incorporates ``confidence''
of a sentence, which has a discrete value, and modeling the objective to
be an integer submodular function.  It is perhaps even more natural to consider continuous confidence values
$x_i\in [0, 1]$. Let us use $p_i(x_i)$ to denote the set of covered
concepts when selecting sentence $i$ with confidence level $x_i$, it
can be a monotone covering function
$p_i: \R_+ \rightarrow 2^C, \forall i\in \groundset$.  Then the
objective function of the extended model is
$f(\x) = \sum_{j\in \cup_i p_i(x_i) } c_j$, where $c_j\in \R_+$ is the
credit of concept $j$. It can be verified that this objective is a
monotone continuous submodular function.

\subsubsection{Maximum Coverage}

In the maximum coverage  problem, there are
$n$ subsets $C_1,..., C_n$ from the ground
set $V$.  One subset $C_i$ can be chosen
with ``confidence" level $x_i\in [0,1]$, the set of covered elements
when choosing subset $C_i$ with confidence $x_i$
can be modeled with the following monotone normalized covering function: $p_i: \R_+ \rightarrow 2^V, i=1,..., n$.
The target is to choose
subsets from $C_1,..., C_n$ with confidence level
to maximize the number of covered elements  $|\cup_{i=1}^n p_i(x_i)|$,
at the same time respecting the budget
constraint $\sum_i c_i x_i \leq b$ (where $c_i$ is the cost of choosing subset $C_i$).
This problem generalizes the classical maximum coverage problem.
It is easy to see that the objective
function is monotone submodular, and
the constraint is a down-closed polytope.

\subsubsection{Sensor Energy Management}
For cost-sensitive outbreak detection in sensor networks
\citep{leskovec2007cost}, one needs to place sensors in a subset of
locations selected from all the possible locations $\groundset$, to
quickly detect a set of contamination events $E$, while respecting the
cost constraints of the sensors.  For each location $v\in \groundset$
and each event $e\in E$, a value $t (v, e)$ is provided as the time it
takes for the placed sensor in $v$ to detect event
$e$. \citet{DBLP:conf/nips/SomaY15} considered the sensors with
discrete energy levels. It is natural to model the energy levels of
sensors to be a \emph{continuous} variable $\x\in \R_+^\groundset$.
For a sensor with energy level $x_v$, the success probability it
detects the event is $1-(1-p)^{x_v}$, which models that by spending
one unit of energy one has an extra chance of detecting the event with
probability $p$.  In this model, beyond deciding whether to place a
sensor or not, one also needs to decide the optimal energy levels. Let
$t_{\infty} = \max_{e\in E, v\in \groundset}t(v,e)$, let $v_e$ be the
first sensor that detects event $e$ ($v_e$ is a random variable).  One
can define the objective as the expected detection time that could be
\textit{saved},
\begin{flalign}
f(\x) := \mathbb{E}_{e\in E} \mathbb{E}_{v_e} [t_{\infty} - t(v_e,
e)],
\end{flalign}
which is a monotone DR-submodular function.  Maximizing $f(\x)$ w.r.t.
the cost constraints pursues the goal of finding the optimal energy
levels of the sensors, to maximize the expected detection time that
could be saved.

\subsubsection{Multi-Resolution Summarization}
Suppose we have a collection of items, e.g., images
$\groundset = \{\ele_1, ..., \ele_n\}$.
We follow the strategy to extract a representative summary, where
representativeness is defined w.r.t.~a submodular set function
$F:2^\groundset\to \mathbb{R}$. However, instead of returning a single
set, our goal is to obtain summaries at multiple levels of detail or
resolution. One way to achieve this goal is to assign each item
$\ele_i$ a nonnegative score $x_i$. Given a user-tunable threshold
$\tau$, the resulting summary $S_\tau=\{\ele_i | x_i \geq \tau\}$ is
the set of items with scores exceeding $\tau$. Thus, instead of
solving the discrete problem of selecting a fixed set $S$, we pursue
the goal to optimize over the scores, e.g., to use the following
continuous submodular function,
\begin{flalign}
f(\x) =  \sum\nolimits_{i \in \groundset} \sum\nolimits_{j\in \groundset}  \phi(x_j) s_{i,j}
- \sum\nolimits_{i \in \groundset} \sum\nolimits_{j \in \groundset} x_i x_j s_{i,j},
\end{flalign}
where $s_{i,j}\geq 0$ is the similarity between items $i,j$, and
$\phi(\cdot)$ is a non-decreasing concave function.

\subsubsection{Facility Location with Scales}

The classical discrete facility location problem can be
generalized to the continuous case where the scale of a facility is
determined by a continuous value in interval $[\zero, \bar \bu]$.  For a
set of facilities $\groundset$, let $\x\in \R_+^\groundset$ be the
scale of all facilities.  The goal is to decide how large each
facility should be in order to optimally serve a set $T$ of
customers. For a facility $s$ of scale $x_s$, let $p_{st}(x_s)$ be the
value of service it can provide to customer $t\in T$, where
$p_{st}(x_s)$ is a normalized monotone function ($p_{st}(0) =
0$).
Assuming each customer chooses the facility with highest value, the
total service provided to all customers is
$f(\x) = \sum_{t\in T} \max_{s\in \groundset} p_{st}(x_s)$.  It can be
shown that $f$ is monotone submodular.

\section{Algorithms for  Monotone \text{DR}-Submodular Maximization}
\label{sec_mono_dr_fun}

In this section,  we present
two classes of algorithms for maximizing a {\em monotone} continuous  DR-submodular  function subject to a  down-closed convex constraint.
The detailed proofs can be found in \cref{supp_proof_monotone_max}.
Even despite the monotonicity assumption, solving
the problem to optimality  is still a very challenging
task.  In fact, we prove the following hardness result:

\begin{proposition}[Hardness and Inapproximability]\label{prop_np}
	The problem of maximizing a monotone nondecreasing continuous
	DR-submodular function subject to a general down-closed
	\emph{polytope} constraint is NP-hard.  For any $\epsilon >0$, it
	cannot be approximated in polynomial time within a ratio of
	$(1-1/e+\epsilon)$ (up to low-order terms), unless RP = NP.
\end{proposition}
\cref{prop_np} can be proved  by
the reduction from the problem of maximizing
a monotone submodular set function subject to cardinality constraints.  The proof relies on   the techniques of multilinear extension \citep{calinescu2007maximizing,DBLP:journals/siamcomp/CalinescuCPV11} and pipage rounding \citep{ageev2004pipage},    and also the hardness results  of  \citet{feige1998threshold,calinescu2007maximizing}.

\begin{remark}
	Due to the NP-hardness of converging to the global optimum for
	Problem \labelcref{setup}, in the following by
	``convergence'' we mean converging to a solution point which has a
	constant factor approximation guarantee with respect to the global
	optimum.
\end{remark}

\subsection{Algorithms based on the Local-Global Relation: \nonconvexfw and \pga}

The first class of algorithms directly utilize the local-global
relation of \cref{coro_1half}. We know that any stationary point is a
1/2 approximate solution. Thus {\em any} solver that obtains a
stationary point yields a solution with a 1/2 approximation
guarantee. We give two concrete examples below.

\subsubsection{The \nonconvexfw Algorithm}
\label{sec_nonconvex_fw}

For sake of completeness, we summarize the \nonconvexfw
algorithm in \cref{nonconvex_fw}.

\begin{algorithm}[htbp]
	\caption{\nonconvexfw
		$(f, \P, K, \epsilon, \x^\pare{0})$\citep{lacoste2016convergence}
		for maximizing a smooth objective}\label{nonconvex_fw}
	\KwIn{$\max_{\x \in \P} f(\x)$, $f$: a smooth function, $\P$:
		{convex} set, $K$: number of iterations, $\epsilon$: stopping
		tolerance}
	\For{$k = 0, ... , K$}{ {find
			$\v^\pare{k} \text{ s.t. } \dtp{\v^\pare{k}}{\nabla
				f(\x^\pare{k})} \geq \max_{\v\in\P} \dtp{\v}{ \nabla
				f(\x^\pare{k})}$\tcp*{\emph{LMO}}}
		{$\d^\pare{k} \leftarrow \v^\pare{k} - \x^\pare{k}$,
			$g_k := \dtp{\d_k}{\nabla f(\x^\pare{k})}$ \tcp*{$g_k$:
				non-stationarity measure}} {\bfseries{if $g_k \leq \epsilon$
				then return $\x^\pare{k}$}\;} {Option I:
			$\gamma_k \in \argmin_{\gamma\in [0, 1]}f(\x^\pare{k} + \gamma
			\d^\pare{k})$,

			Option II: $\gamma_k \leftarrow \min \{\frac{g_k}{C}, 1 \}$ for
			$C\geq C_f(\P)$ \;}
		{$\x^\pare{k+1}\leftarrow \x^\pare{k} + \gamma_k \d^\pare{k}$ \;}
	} \KwOut{$\x^\pare{k'}$ and $g_{k'} = \min_{0\leq k\leq K} g_k$
		\tcp*{modified output solution compared to that of
			\citet{lacoste2016convergence}}}
\end{algorithm}
\cref{nonconvex_fw} is modified from \citet{lacoste2016convergence}.
The only difference lies in the output: we return the solution
$\x^\pare{k'}$ with the minimum non-stationarity, which is needed to
invoke the local-global relation. In contrast, \citet{lacoste2016convergence}
outputs the solution from the last iteration.
Since $C_f(\P)$ is generally hard to evaluate, we use the
classical oblivious \stepsize rule ($\frac{2}{k+2}$) and the Lipschitz
\stepsize rule
($\gamma_k = \min\{1, \frac{g_\pare{k}}{L \|\d^\pare{k}\|} \}$, where
$g_\pare{k}$ is the so-called Frank-Wolfe gap) in the experiments (\cref{sec_exp}).

\citet{hassani2017gradient} show that the {Projected Gradient
Ascent} algorithm (\pga) with constant \stepsize ($1/L$) can
converge to a stationary point, so it has a 1/2 approximation
guarantee.
We can also show that the \nonconvexfw of
\citet{lacoste2016convergence} has a 1/2 approximation guarantee
according to the local-global relation:

\begin{corollary}\label{coro_nonconvex_fw}
	The {non-convex Frank-Wolfe} algorithm (abbreviated as \nonconvexfw) of
	\citet{lacoste2016convergence} has a 1/2 approximation guarantee,
	and $1/\sqrt{k}$ rate of convergence for solving Problem \labelcref{setup} when the objective is monotone \nondec.
\end{corollary}

\subsubsection{The \pga Algorithm}

\begin{algorithm}[htbp]
	\caption{\pga for maximizing a monotone DR-submodular objective
		\citep{hassani2017gradient}}\label{alg_pga}

	\KwIn{$\max_{\x \in \P} f(\x)$, $f$: a smooth DR-Submodular
		function, $\P$: {convex} set, $K$: number of iterations,
		$\x^\pare{0}\in \P$}
	\For{$k = 0, ... , K-1$}{ {Set \stepsize $\gamma_k$ \tcp*{i):
				``Lipschitz'' rule $\frac{1}{L}$; ii): adaptive rule:
				$C/\sqrt{k}$}}
		{$\y^\pare{k+1} \leftarrow \x^\pare{k} + \gamma_k \nabla
			f(\x^\pare{k})$\;}
		{$\x^\pare{k+1}\leftarrow \argmin_{\x\in \P} \|\x -
			\y^\pare{k+1}\| $
			\tcp*{Projection}} } \KwOut{$\x^\pare{k'}$ with
		${k'} = \argmax_{0\leq k\leq K} f(\x^\pare{k})$ \tcp*{modified
			output compared to that of \citet{hassani2017gradient}}}
\end{algorithm}

\cref{alg_pga} is reproduced from \citet{hassani2017gradient} for completeness.
It takes a smooth DR-submodular function $f$, and a convex
constraint $\P$. Then it runs for $K$ iterations. In each iteration,
we firstly choose a \stepsize $\gamma_k$, then we update the current
solution using the current gradient to get a point $\y^\pare{k+1}$.
Lastly, we projects $\y^\pare{k+1}$ onto the convex set $\P$, which
amounts to solving a constrained quadratic program. After $K$
iterations, we output the solution with the maximal function value,
which is slightly different from that of \citet{hassani2017gradient}.

The resulting algorithm has a 1/2 approximation guarantee and sublinear rate of
convergence:

\begin{theorem}[\cite{hassani2017gradient}]
	For \cref{alg_pga}, if one chooses $\gamma_k = 1/L$, then after $K$
	iterations,
	\begin{align}
	f(\x^\pare{K}) \geq \frac{f(\optcont)}{2} - \frac{D^2 L}{2K}.
	\end{align}
\end{theorem}
It is worth noting that, in general the smoothness parameter $L$ is
difficult to estimate, so the ``Lipschitz'' \stepsize rule
$\gamma_k = 1/L$ poses a challenge for implementation. In experiments,
\citet{hassani2017gradient} also suggest the adaptive \stepsize rule
$\gamma_k = C/\sqrt{k}$, where $C$ is a constant.

\subsection{\texttt{Submodular FW}\xspace: Follow Concave Directions}

For DR-submodular maximization, one key property is that
while being non-convex/non-concave in general,
they are concave along any non-negative directions (c.f., \cref{prop_concave}).  Thus, if we design an algorithm such that it
follows a non-negative direction in each update step, we ensure that
it achieves progress in a concave direction. As a consequence, its
function value is guaranteed to grow by a certain increment.  Based on
this intuition, we present the \submodularfw algorithm, which is a
generalization of the continuous greedy algorithm of
\citet{DBLP:conf/stoc/Vondrak08}, and the classical Frank-Wolfe
algorithm \citep{frank1956algorithm,DBLP:conf/icml/Jaggi13}.

\begin{algorithm}[ t]
	\caption{\submodularfw for monotone {DR}-submodular
		maximization \citep{bian2017guaranteed}}\label{alg_sfmax_GradientAscend}

	\KwIn{$\max_{\x\in \P} f(\x)$, $\P$ is a down-closed {convex} set in
		the positive orthant with lower bound $\zero$; prespecified
		step size $\gamma \in (0, 1]$; Error tolerances $\alpha$ and
		$\delta$. \# of iterations $K$.}

	{$\x^0 \leftarrow \zero$, $t\leftarrow 0$, $k\leftarrow 0$\tcp*{$k:$
			iteration index, $t$: cumulative \stepsize}} \While{$t < 1$}{

		{find \stepsize $\gamma_k\in (0, 1]$, e.g.,
			$\gamma_k \leftarrow \gamma $;
			set $\gamma_k \leftarrow \min\{\gamma_k, 1-t\}$\;}

		{find
			$\v^k \text{ s.t. }  \dtp{\v^k}{\nabla f(\x^k)} \geq
			\alpha\max_{\v\in\P} \dtp{\v}{ \nabla f(\x^k)} -
			\frac{1}{2}\delta \gamma_k L D^2$
			\tcp*{$\alpha\in(0, 1]$ is the mulplicative error level,
				$\delta\in [0, \bar \delta]$ is the additive error
				level}\label{fw_sub}}

		{$\x^{k+1}\leftarrow \x^k + \gamma_k \v^k$,
			$t \leftarrow t + \gamma_k$,
			$k\leftarrow k+1$\;\label{step_update}} }

	\KwOut{$\x^K$\;
	}
\end{algorithm}

\cref{alg_sfmax_GradientAscend} summarizes the details. Since it is a
variant of the convex Frank-Wolfe algorithm for DR-submodular
maximization, we call it \submodularfw.
In iteration $k$, it uses the linearization of $f(\cdot)$
as a surrogate, and moves in the direction of the maximizer of this
surrogate function, i.e.,
$\v^k=\arg\max_{\v \in \P} \dtp{\v}{ \nabla f(\x^k)}$.
Intuitively, it searches for the direction in which one can maximize the
improvement in the function value and still remain feasible.  Finding
such a direction requires maximizing a linear objective at each
iteration.
Meanwhile, it eliminates the need for projecting back to the feasible
set in each iteration, which is an essential step for methods such as
projected gradient ascent (\pga).
The \submodularfw algorithm updates the solution in each iteration by using \stepsize
$\gamma_k$, which can simply be set to a prespecified constant
$\gamma$.

Note that \submodularfw can tolerate both multiplicative error
$\alpha$ and additive error $\delta$ when solving the LMO subproblem
(Step \ref{fw_sub} of \cref{alg_sfmax_GradientAscend}). Setting
$\alpha = 1$ and $\delta = 0$ would recover the error-free case.

\begin{remark}
	The main difference of  \submodularfw in
	\cref{alg_sfmax_GradientAscend} and  the classical
	Frank-Wolfe algorithm in \cref{alg_classical_fw} lies in the  update direction being used:
	For \cref{alg_sfmax_GradientAscend}, the update direction (in Step
	\ref{step_update}) is $\v^k$, while for classical Frank-Wolfe it is
	$\v^k - \x^k$, i.e.,
	$\x^{k+1}\leftarrow \x^k + \gamma_k(\v^k - \x^k)$.
\end{remark}

To prove the approximation guarantee, we first derive the following
lemma.
\begin{lemma}\label{lemma_31}
	The output solution $\x^K$ lies in $\P$. Assuming $\x^*$ to be the
	optimal solution, one has,
	\begin{flalign}\label{eq26}
	\dtp{\v^k}{\nabla f(\x^k)}\geq \alpha [f(\x^*) -f(\x^k)] -
	\frac{1}{2}\delta \gamma_k L D^2 , \ \ \forall k = 0,..., K-1.
	\end{flalign}
\end{lemma}
\begin{theorem}[Approximation guarantee]\label{thm_fw}
	For error levels $\alpha \in (0, 1], \delta\in [0, \bar \delta]$,
	with $K$ iterations, \cref{alg_sfmax_GradientAscend} outputs
	$\x^K \in \P$ such that,
	\begin{equation}\label{eq8}
	f(\x^K)   \geq  (1-e^{-\alpha})f(\x^*)
	-\frac{LD^2 (1+\delta)}{2} \sum_{k=0}^{K-1}\gamma_k^2 + e^{-\alpha}f(\zero).
	\end{equation}
\end{theorem}
\cref{thm_fw} gives the approximation guarantee for
any step size $\gamma_k$.  By observing that
$\sum_{k=0}^{K-1}\gamma_k =1$ and
$\sum_{k=0}^{K-1}\gamma_k^2 \geq K^{-1}$ (see the proof in
\cref{app_proof_c9}), with constant step size, we obtain the following
``tightest'' approximation bound,

\begin{corollary}\label{cor_9}
	For a fixed number of iterations $K$, and constant step size
	$\gamma_k =\gamma = K^{-1}$, \cref{alg_sfmax_GradientAscend}
	provides the following approximation guarantee:
	\begin{equation}
	f(\x^K) \geq (1-e^{-\alpha})f(\x^*) -\frac{LD^2 (1+\delta)}{2K}+
	e^{-\alpha}f(\zero).
	\end{equation}
\end{corollary}

Corollary \ref{cor_9} implies that with a constant step size $\gamma$,
1) when $\gamma \rightarrow 0$ ($K\rightarrow \infty$),
\cref{alg_sfmax_GradientAscend} will output the solution with the
worst-case guarantee $(1-1/e)f(\x^*)$ in the error-free case if
$f(\zero) = 0$; and 2) The \submodularfw has a sub-linear convergence
rate for monotone {DR}-submodular maximization over any down-closed
convex constraint.

\paragraph{Remarks on computational cost.}  It can be seen that
when using a constant step size, \cref{alg_sfmax_GradientAscend} needs
$O(\frac{1}{\epsilon})$ iterations to get $\epsilon$-close to the
best-possible function value $(1-e^{-1})f(\x^*)$ in the error-free
case.  When $\P$ is a polytope in the positive orthant, one iteration
of \cref{alg_sfmax_GradientAscend} costs approximately the same as
solving a positive LP, for which a nearly-linear time solver exists
\citep{allen2015nearly}.

\section{Algorithms for  Non-Monotone DR-Submodular Maximization}
\label{sec_algs}

In this section we present algorithms for the problem of non-monotone DR-submodular maximization, all omitted proofs can be found in \cref{proofs_nonmonotone_max}.
Non-monotone DR-submodular maximization is strictly harder
than the monotone setting. For the simple situation with only one
hyperrectangle constraint ($\P = [0, 1]^n$), we have the following hardness result:

\begin{proposition}[Hardness and Inapproximability]\label{obs_dr_submodular_max}
	The problem of maximizing a generally non-monotone continuous DR-submodular
	function subject to a \emph{hyperrectangle} constraint is NP-hard. Furthermore, there is no $(1/2+\epsilon)$-approximation for any $\epsilon>0$, unless RP = NP.
\end{proposition}

\noindent The above results can be proved through  the reduction from the problem of maximizing an unconstrained non-monotone submodular set function.
The proof depends on the techniques of \citet{calinescu2007maximizing,buchbinder2012tight}
and the hardness results of \citet{feige2011maximizing,dobzinski2012query}.

We propose two algorithms:
The first  is based on the local-global relation,  and the second  is a \algname{Frank-Wolfe} variant adapted for the non-monotone
setting.  All  the omitted proofs are deferred to \cref{proofs_nonmonotone_max}.

\subsection{\twophase  Algorithm: Applying the Local-Global Relation}
\label{subsec_local_alg}

\begin{algorithm}[htbp]
	\caption{The \algname{Two-Phase} Algorithm \citep{biannips2017nonmonotone}
	}\label{lg_fw}

	\KwIn{$\max_{\x \in \P} f(\x)$,
		stopping tolerances $\epsilon_1, \epsilon_2$, \#iterations
		$K_1, K_2$}

	{$\x \leftarrow \algname{Non-convex Frank-Wolfe}(f, \P, K_1,
		\epsilon_1, \x^\pare{0}) $ \tcp*{$\x^\pare{0}\in \P$}}

	{$\Q \leftarrow \P\cap \{\y\in \R_+^n\;|\;\y\leqco \bar \u -\x \}$\;}
	{
		$\z \leftarrow \algname{Non-convex Frank-Wolfe}(f, \Q, K_2,
		\epsilon_2, \z^\pare{0}) $ \tcp*{$\z^\pare{0}\in \Q$}}

	\KwOut{$\argmax \{f(\x), f(\z)\}$ \;}
\end{algorithm}

By directly applying the local-global relation in
\cref{subsec_local_global}, we present the \twophase algorithm in
\cref{lg_fw}. It  generalizes the ``two-phase'' method of
\citet{chekuri2014submodular,gillenwater2012near}. It invokes a
non-convex solver (we use the \nonconvexfw by
\citet{lacoste2016convergence}; pseudocode is included in
\cref{nonconvex_fw} of \cref{sec_nonconvex_fw}) to find approximately
stationary points in $\P$ and $\Q$, respectively, then returns the
solution with the larger function value.

Though we use \nonconvexfw as a subroutine here, it is noteworthy that
any algorithm that is guaranteed to find an approximately stationary
point can be plugged into \cref{lg_fw} as a subroutine.
We give an improved approximation bound by considering more properties
of DR-submodular functions.  Building on the results of
\citet{lacoste2016convergence}, we obtain the following

\begin{theorem}\label{rate_local_fw}
	The output of \cref{lg_fw} satisfies,
	\begin{align}\label{eq_local_rates}
	&\max \{f(\x), f(\z)\}  \geq
	\frac{\mu}{8}\left(\|\x -\x^*\|^2 + \|\z - \z^*\|^2\right )\\\notag
	&    + \frac{1}{4}\left[f(\x^*)  - \min \left\{\frac{\max \{2h_1,
		C_f(\P)\}}{\sqrt{K_1+1}} , \epsilon_1\right\}   -
	\min\left\{\frac{\max \{2h_2, C_f(\Q)\}}{\sqrt{K_2+1}} ,
	\epsilon_2\right\} \right],
	\end{align}
	where $h_1 := \max_{\x\in\P}f(\x) - f(\x^\pare{0})$,
	$h_2 := \max_{\z\in\Q}f(\z) - f(\z^\pare{0})$ are the initial
	suboptimalities,
	$C_f(\P) : = \sup_{\x, \v\in \P, \gamma\in (0, 1], \y = \x + \gamma
		(\v - \x )}\frac{2}{\gamma^2}(f(\y) - f(\x) - {(\y -
		\x)^\trans}{\nabla f(\x)})$
	is the curvature of $f$ w.r.t.  $\P$, and $\z^*= \x\vee \x^* -\x$.
\end{theorem}
\cref{rate_local_fw} indicates that \cref{lg_fw} has a $1/4$
approximation guarantee and $1/\sqrt{k}$ rate of convergence.
However, it has good empirical performance as demonstrated by the
practical experiments.  Informally, this can be partially explained by
the term $\frac{\mu}{8}\left(\|\x -\x^*\|^2 + \|\z - \z^*\|^2\right )$
in
\labelcref{eq_local_rates}: if $\x$ strongly deviates from $\x^*$,
then this term will augment the bound; if $\x$ is close to $\x^*$, by
the smoothness of $f$, it should be close to optimal.

\subsection{\shrunkenfw: Follow Concavity and Shrink
	Constraint}\label{sec_fw_variant}

\begin{algorithm}[htbp]
	\caption{The \shrunkenfw Algorithm
		for Non-monotone
		{DR}-submodular
		Maximization \citep{biannips2017nonmonotone}}\label{fw-non-monotone}
	\KwIn{$\max_{\x \in \P} f(\x)$
		; \#iterations $K$; \stepsize $\gamma= 1/K$.}
	{$\x^\pare{0} \leftarrow \zero$, $t^\pare{0}\leftarrow 0$, $k\leftarrow 0$\tcp*{$k:$ iteration index, $t^\pare{k}:$ cumulative \stepsize}}
	\While{$t^\pare{k} <  1$}{
		{$\v^\pare{k} \leftarrow   \argmax_{\v\in\P, \textcolor{blue}{\v\leqco {\bar \u}-\x^\pare{k}}} \dtp{\v}{ \nabla f(\x^\pare{k})}$\tcp*{\emph{ \color{blue} shrunken LMO}} \label{new_lmo}}
		{use uniform \stepsize $\gamma_k = \gamma$;  set $\gamma_k \leftarrow \min\{\gamma_k, 1-t^\pare{k} \}$\;}
		{$\x^\pare{k+1}\leftarrow \x^\pare{k} + \gamma_k \v^\pare{k}$, $t^\pare{k+1} \leftarrow t^\pare{k} + \gamma_k$,  $k\leftarrow k+1$\;}
	}
	\KwOut{$\x^\pare{K}$
		\tcp*{suppose there are $K$ iterations in total}
	}
\end{algorithm}

\cref{fw-non-monotone} summarizes the \shrunkenfw variant, which  is inspired by the  unified continuous
greedy algorithm in \citet{feldman2011unified} for
maximizing the multilinear extension of a submodular
set function.

It initializes the solution $\x^\pare{0}$ to be $\zero$, and maintains
$t^\pare{k}$ as the cumulative \stepsize. At iteration $k$, it
maximizes the linearization of $f$ over a ``shrunken'' constraint set
$\{\v \mid \v\in \P, \v\leqco \bar \u-\x^\pare{k}\}$, which is
different from the classical LMO
of Frank-Wolfe-style algorithms (hence we refer to it as the
``shrunken LMO''). Then it employs an update step in the direction
$\v^\pare{k}$ chosen by the LMO with a uniform \stepsize
$\gamma_k = \gamma$.
The cumulative \stepsize $t^\pare{k}$ is used to ensure that the
overall \stepsizes sum to one, thus the output solution $\x^\pare{K}$
is a convex combination of the LMO outputs, hence also lies in $\P$.

The shrunken LMO (Step \labelcref{new_lmo}) is the key difference
compared to the \submodularfw variant in
\citet{bian2017guaranteed} (detailed in
\cref{alg_sfmax_GradientAscend}). Therefore, we call
\cref{fw-non-monotone} \shrunkenfw.  The extra constraint
$\v\leqco {\bar \u}-\x^\pare{k}$ is added to prevent too rapid
growth of the solution, since in the non-monotone setting such
fast increase may hurt the overall performance.

The next theorem states the guarantees of \shrunkenfw in
\cref{fw-non-monotone}.
\begin{theorem}\label{thm-e}
	Consider \cref{fw-non-monotone} with uniform step size $\gamma$.
	For $k = 1,..., K$ it holds that,
	\begin{flalign}
	f(\x^\pare{k}) \geq t^\pare{k} e^{-t^\pare{k}}f(\x^*) - \frac{L
		D^2}{2}k\gamma^2 - O(\gamma^2)f(\x^*).
	\end{flalign}
\end{theorem}
By observing that $t^\pare{K} = 1$ and applying \cref{thm-e}, we get
the following \namecref{coro_e}:
\begin{corollary}\label{coro_e}
	The output of \cref{fw-non-monotone} satisfies
	\begin{flalign}
	f(\x^\pare{K}) \geq \frac{1}{e} f(\x^*) - \frac{L D^2}{2K} -
	O\left(\frac{1}{K^2}\right)f(\x^*).
	\end{flalign}
\end{corollary}

\cref{coro_e} shows that \cref{fw-non-monotone} enjoys a sublinear
convergence rate towards some point $\x^\pare{K}$ inside $\P$, with a
$1/e$ approximation guarantee.

\paragraph{Proof sketch of \cref{thm-e}: }
The proof is by induction. To prepare the building blocks, we first of
all show that the growth of $\x^\pare{k}$ is indeed bounded,
\begin{restatable}[Bounding the growth of $\x^\pare{k}$]{lemma}{restalemmatwo}
	\label{prop_non_fw}
	Assume $\x^\pare{0} = \zero$. For $k=0,..., K-1$, it holds,
	\begin{align}
	x_i^\pare{k}\leq \bar u_i[1-(1-\gamma)^{t^\pare{k}/\gamma}],
	\forall i\in [n].
	\end{align}
\end{restatable}

Then the following \namecref{lem_nonmonotone_fw} provides a lower
bound, which depends on the global optimum,

\begin{restatable}[Generalized from Lemma 7 of
	\citet{chekuri2015multiplicative}]{lemma}{restalemmathree}
	\label{lem_nonmonotone_fw}
	Given $\bmtheta\in (\zero, \bar \u]$, let
	$\lambda' = \min_{i\in [n]} \frac{\bar u_i}{\theta_i}$. Then for all
	$\x\in [\zero, \bmtheta]$, it holds,
	\begin{align}
	f(\x\vee \x^*) \geq (1-\frac{1}{\lambda'})f(\x^*).
	\end{align}
\end{restatable}

Then the key ingredient for induction  is the relation between  $f(\x^{\pare{k+1}})$
and $f(\x^{\pare{k}})$ indicated by:
\begin{restatable}{claim}{restaclaimthree}
	\label{claim3_1}
	For $k = 0,...,K-1$ it holds,
	\begin{align}
	f(\x^{\pare{k+1}}) \geq (1-\gamma) f(\x^{\pare{k}}) +
	\gamma(1-\gamma)^{t^\pare{k}/\gamma} f(\x^*) -\frac{L
		D^2}{2}\gamma^2.
	\end{align}
\end{restatable}
It is derived by a combination of the quadratic lower
bound in \cref{eq_quad_lower_bound}, \cref{prop_non_fw} and
\cref{lem_nonmonotone_fw}.

\subsection{Remarks on the Two Algorithms.}
Notice that though the \twophase algorithm has an inferior guarantee
compared to \shrunkenfw, it is still of interest: i) It preserves
flexibility in using a wide range of existing solvers for finding an
(approximately) stationary point. ii) The guarantees that we present
rely on a worst-case analysis. The empirical performance of the
\twophase algorithm is often comparable or better than that of
\shrunkenfw. This suggests to explore more properties in concrete
problems that may favor the \twophase algorithm.

\section{Experimental Evaluation}
\label{sec_exp}

\subsection{Influence Maximization with Marketing Strategies}

Follow the application in \cref{app_influence_max_marketing_strategies},
we consider the following simplified influence  model for experiments. The resulted problem is an instance of the monotone DR-submodular maximization problem.

\subsubsection{Experimental Setup}

\paragraph{Simplified Influence Model for Experiments.}

For general influence models, it is hard to evaluate
\cref{influence_general_marketing}.  To simplify the experiments, we
consider $F(S)$ to be a facility location objective, for which the
expected influence has a closed-form expression, as shown by \citet[Section
4.2]{bian2019optimalmeanfield}.
Here each customer may represent an ``opinion leader'' in social networks,
and there is a bipartite graph describing the influence strength of
each opinion leader to the population.

\paragraph{Dataset.}

We used the UC Irvine forum
dataset\footnote{\url{http://konect.uni-koblenz.de/networks/opsahl-ucforum}} as the real-world bipartite graph.
It is a bipartite network containing user posts to forums. The users
are students at the University of California, Irvine. An edge
represents a forum message on a specific forum.  It has in total 899
users, 522 forums and 33,720 edges (posts on the forum).

For a specific (user, forum) pair, we determine the edge weight as the
number of posts from that user on the forum. This weighting indicates
that the more one user has posted on a forum, the more he has
influenced that particular forum. With this processing, we have
7,089 unique edges between users and forums.

We experimented with the independent marketing actions in \cref{app_activations_influence_max} for
simplicity. For a customer $i$, we set the parameter $p_i \in [0, 1]$
based on the following heuristic: Firstly, we calculate the
``degree'' of customer $i$ as the number of forums he has posted on:
$d_i = \|W_{i:}\|_0$. Then we set $p_i = \sigma(- d_i)$, $\sigma(\cdot)$ is
the logistic sigmoid function.  Remember that $p_i$ is the probability
of customer $i$ becoming activated with one unit of investment, so this
heuristic means that the more influence power a user has, the more
difficult it is to activate him, because he might charge more than other
users with less influence power.
Since it is too time consuming to experiment on the whole bipartite
graph, we experimented on different subgraphs of the original
bipartite graph.

\subsubsection{Experimental Results}

\setkeys{Gin}{width=0.53\textwidth}
\begin{figure}[htbp]
	\center \subfloat [50 users, 10 forums] {
		\includegraphics[]{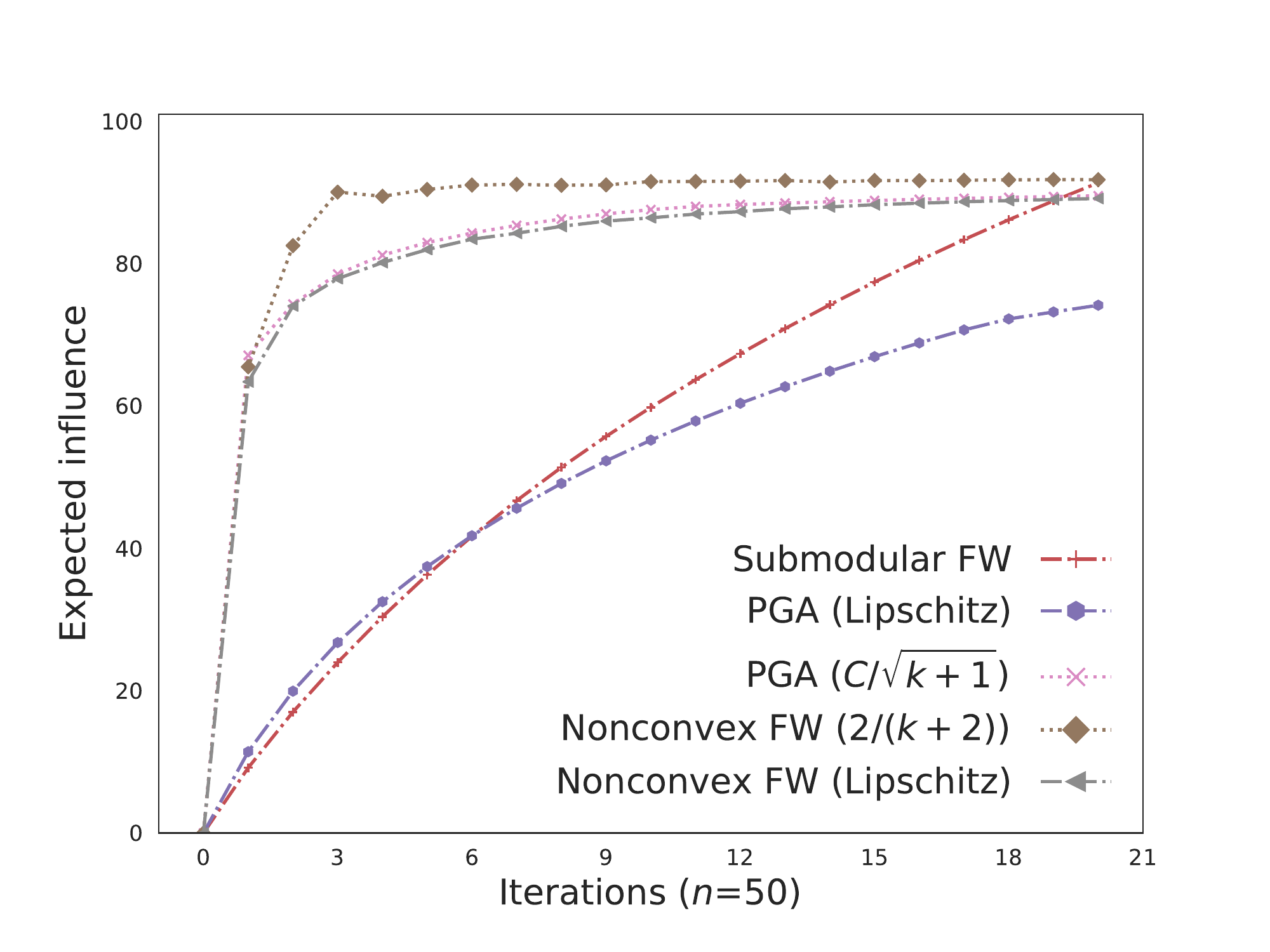}}
	\subfloat [100 users, 10 forums] {
		\includegraphics[]{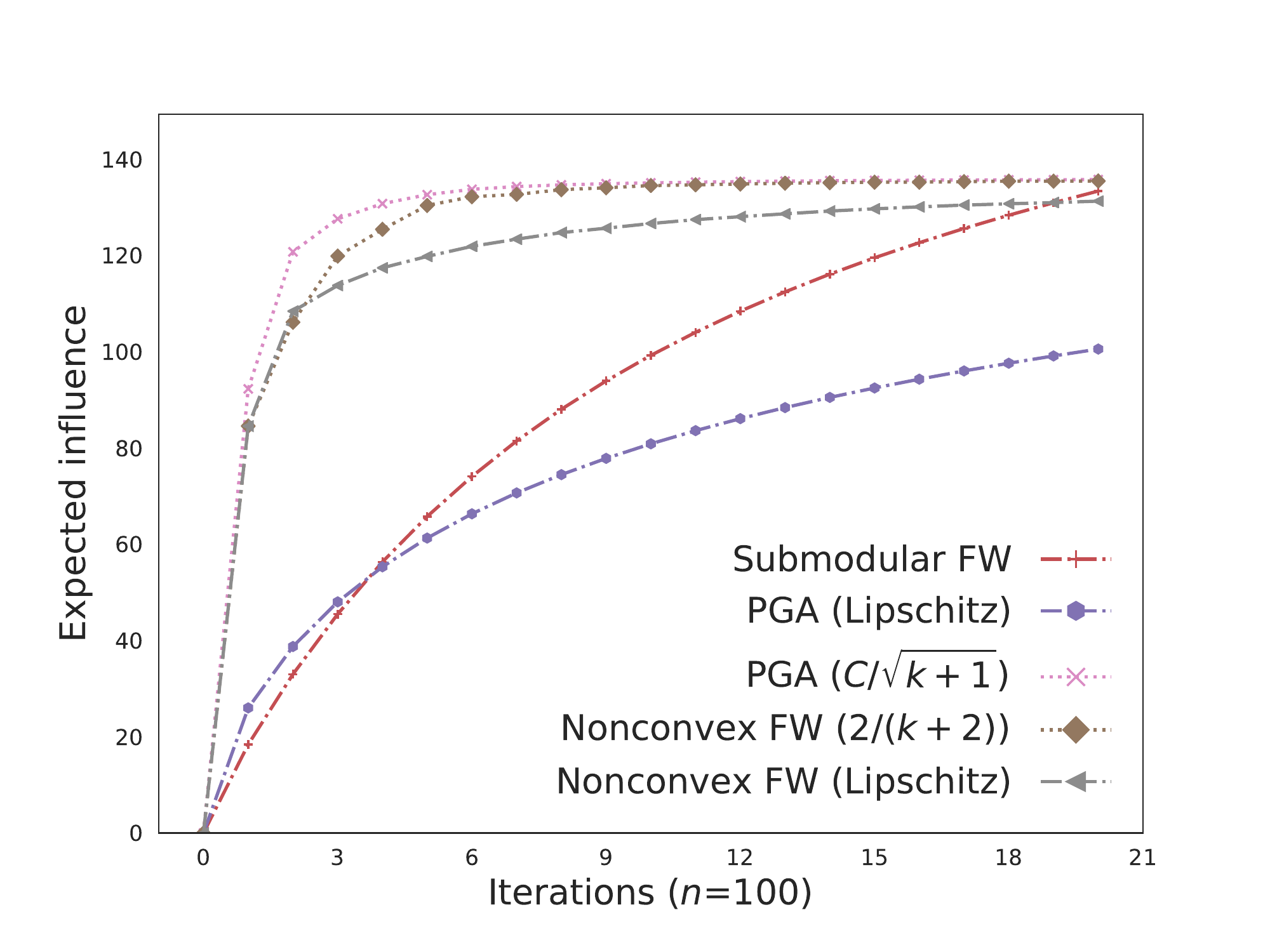}}\\
\subfloat [150 users, 20 forums] {
	\includegraphics[]{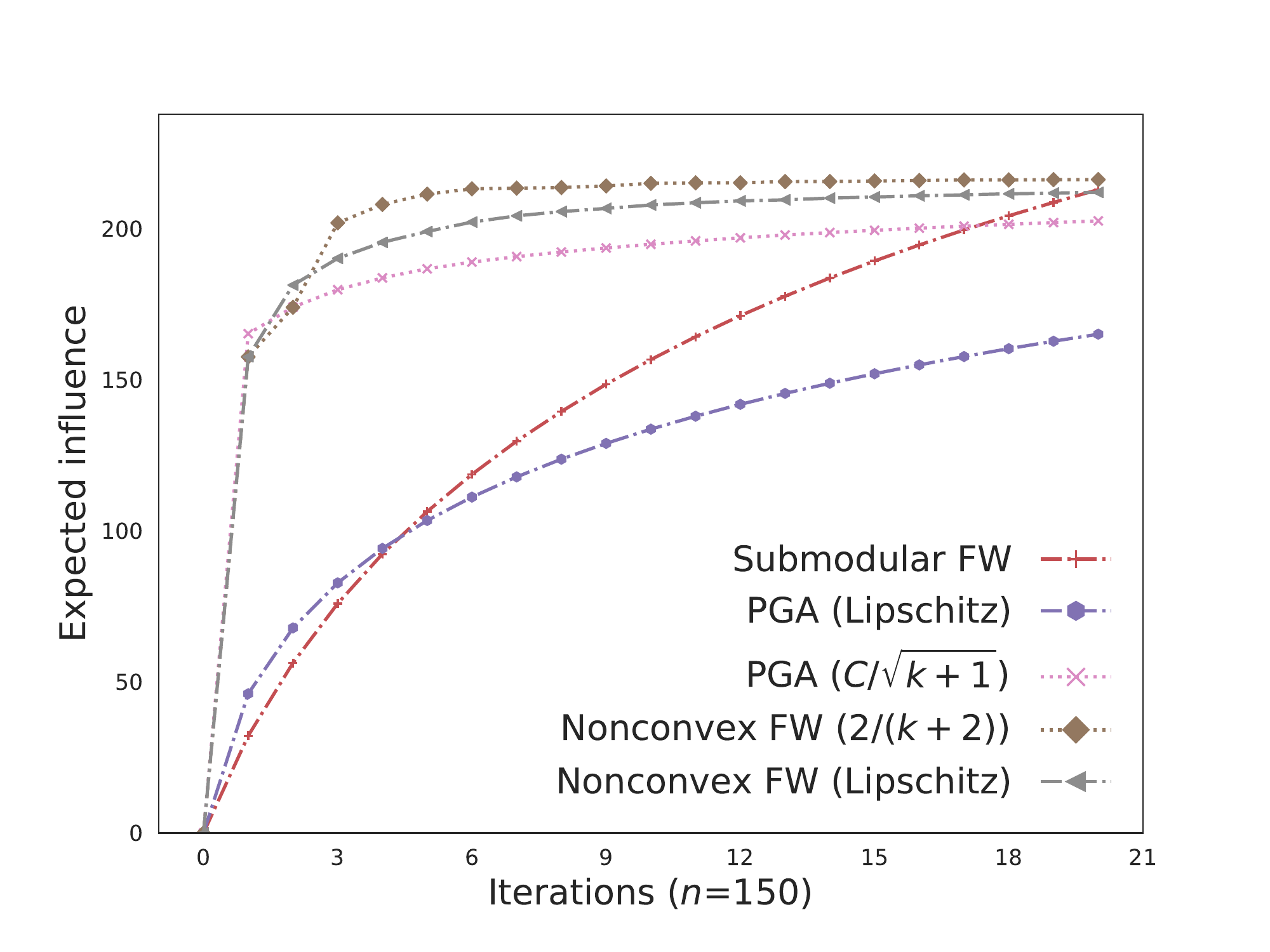}}
 \subfloat [200 users, 20 forums] {
	\includegraphics[]{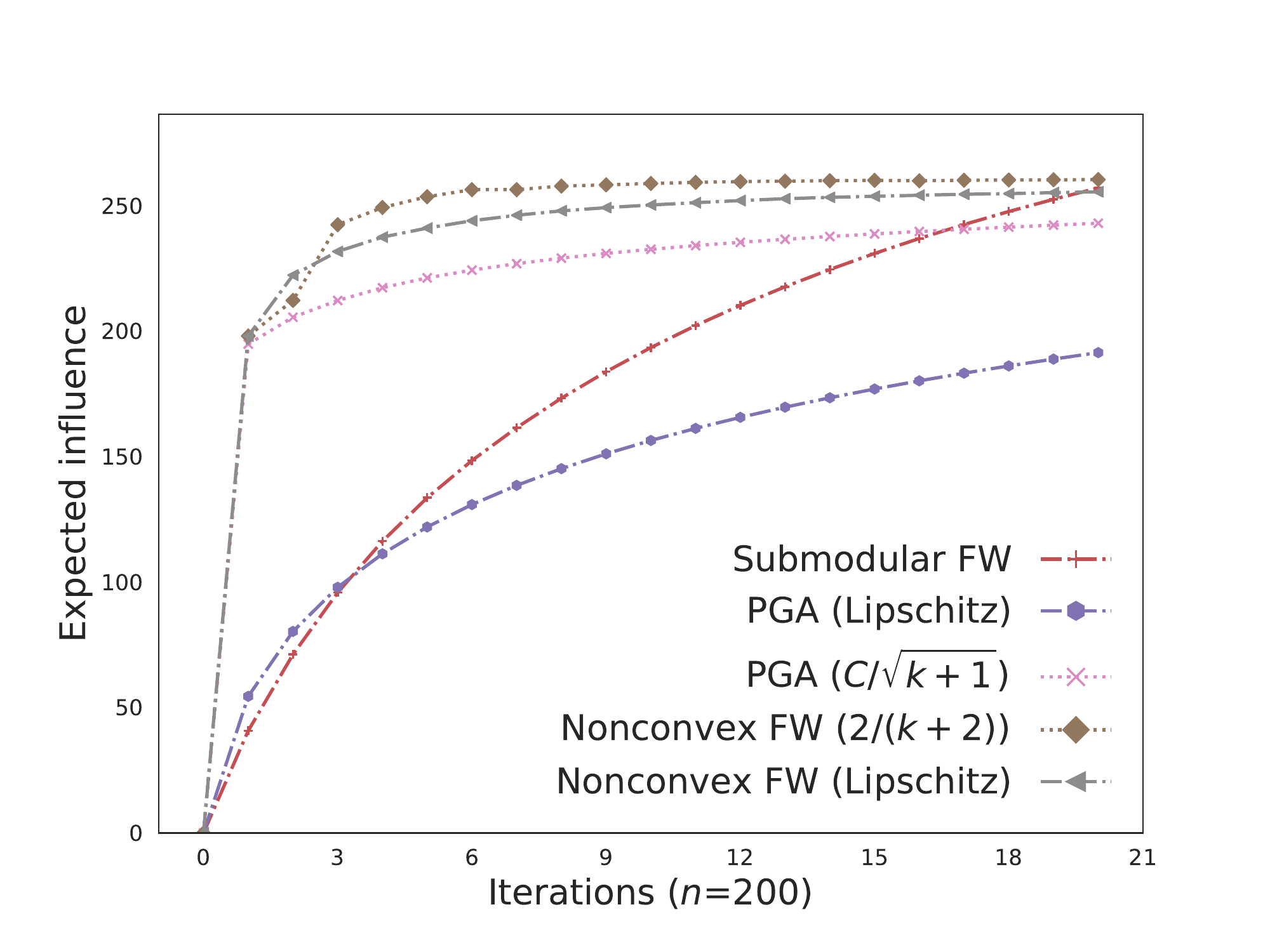}}
	\caption{Expected influence w.r.t. iterations of different
		algorithms on real-world subgraphs with (a) 50 users; (b) 100 users; (c) 150 users; (d) 200 users. \submodularfw has a  stable performance. It does not need to tune the \stepsizes or any hyperparameters. \pga algorithms are sensitive to quality of tuned \stepsizes. \nonconvexfw
		with the Lipschitz \stepsize rule also needs a careful tuning of the Lipschitz parameter.}
	\label{fig_traj_influence_50_100}
\end{figure}

\cref{fig_traj_influence_50_100} documents
the trajectories of expected influence of different algorithms.  We
can see that \submodularfw has a very stable performance: It can
always reach a fairly good solution, no matter what kind of setting
you have. And it does not need to tune the \stepsizes or any
hyperparameters. One drawback is that it converges relatively slowly in the
beginning.

For \pga algorithms, we tested with two \stepsize rules: the Lipschitz
rule ($1/L$) which has the 1/2 approximation guarantee; the diminishing
\stepsize rule ($C/\sqrt{k+1}$), which does not have a formal
theoretical guarantee.
One general observation is that both \stepsize rules need a careful
tuning of hyperparameters, and the performance crucially depends on
the quality of hyperparameters. For example, for \pga, if the
\stepsize is too small, it may converge too slowly; if the \stepsizes
are too large, it tends to fluctuate.

For \nonconvexfw algorithms, we also tested two \stepsize rules: the
``oblivious'' rule ($2/(k+2)$)) and the Lipschitz rule. Apparently the
Lipschitz \stepsize rule needs a careful tuning of the Lipschitz
parameter $L$, while the oblivious rule does not.  With a careful
tuning of $L$, both \nonconvexfw variants converge very fast and
converge to the highest function value.

\subsection{Maximizing  Softmax Extensions}

Maximizing Softmax extensions of DPP MAP
inference is an important instance of non-monotone
DR-submodular maximization problem.
One can obtain  the derivative of the softmax
extension in \cref{eq_softmax} as:
\begin{align}
\nabla_i f(\x) = \tr{ \{ [{\diag(\x)(\bmL-\bmI) +\bmI }]^{-1} [(\bmL -
	\bmI)_i] \}}, \forall i\in [n],
\end{align}
where $(\bmL - \bmI)_i$ denotes the matrix obtained by zeroing all
entries except for  the $i^\text{th}$ row of $(\BL - \BI)$. Let
$\BC:= ({\diag(\x)(\bmL-\bmI) +\bmI })^{-1}, \BD:=(\BL - \BI)$, one
can see that $\nabla_i f(\x) = \BD_{i\cdot}^\trans \BC_{\cdot i}$\footnote{where $\BD_{i\cdot}$ means the $i$-th row of $\BD$ and $\BC_{\cdot i}$ indicates the $i$-th column of $\BC$.},
which gives an efficient way to calculate the gradient $\nabla f(\x)$.

\setkeys{Gin}{width=0.34\textwidth}
\begin{figure}[htbp]
	\center
	\subfloat
	{
		\includegraphics[]{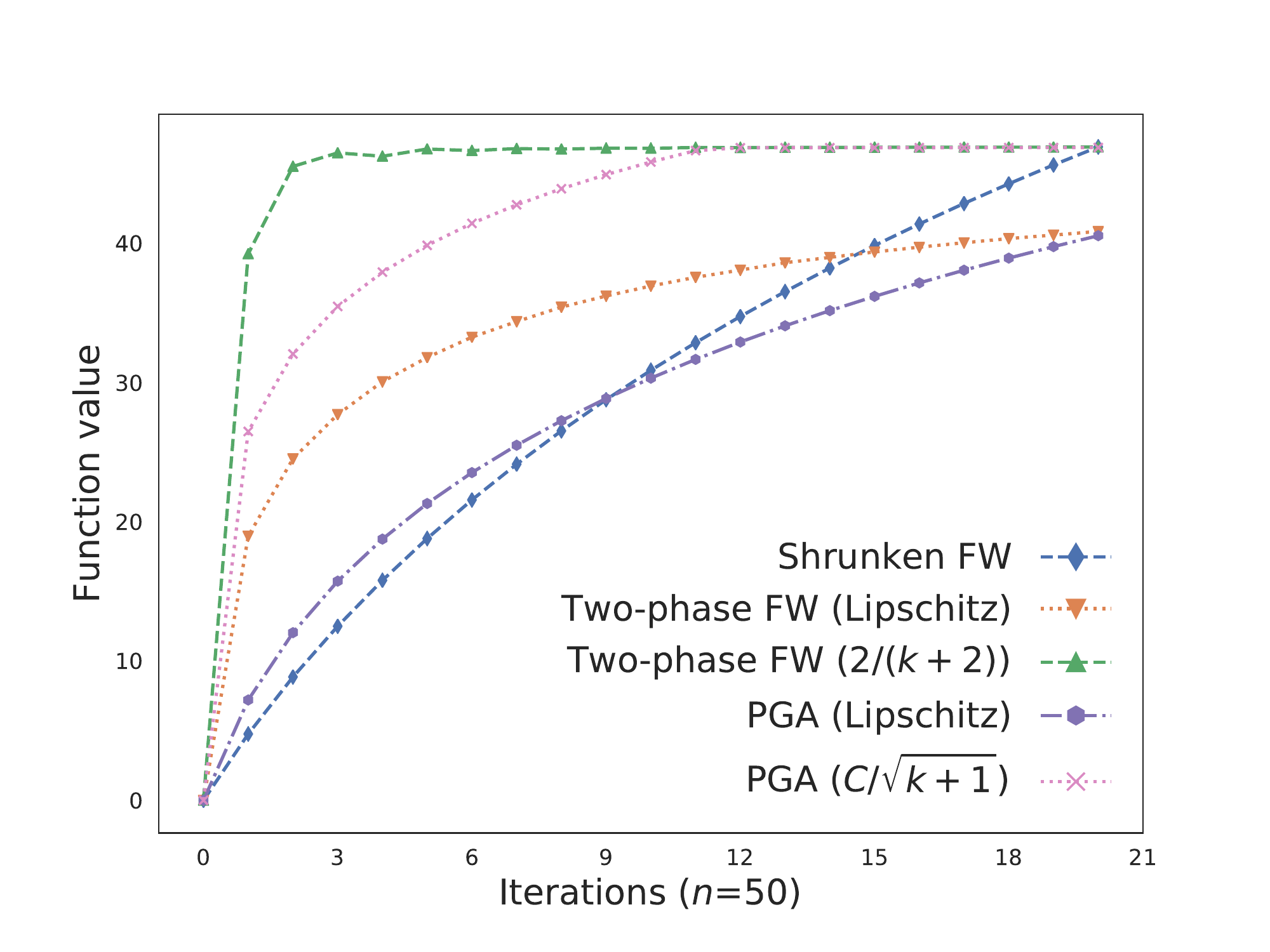}}
	\subfloat
	{
		\includegraphics[]{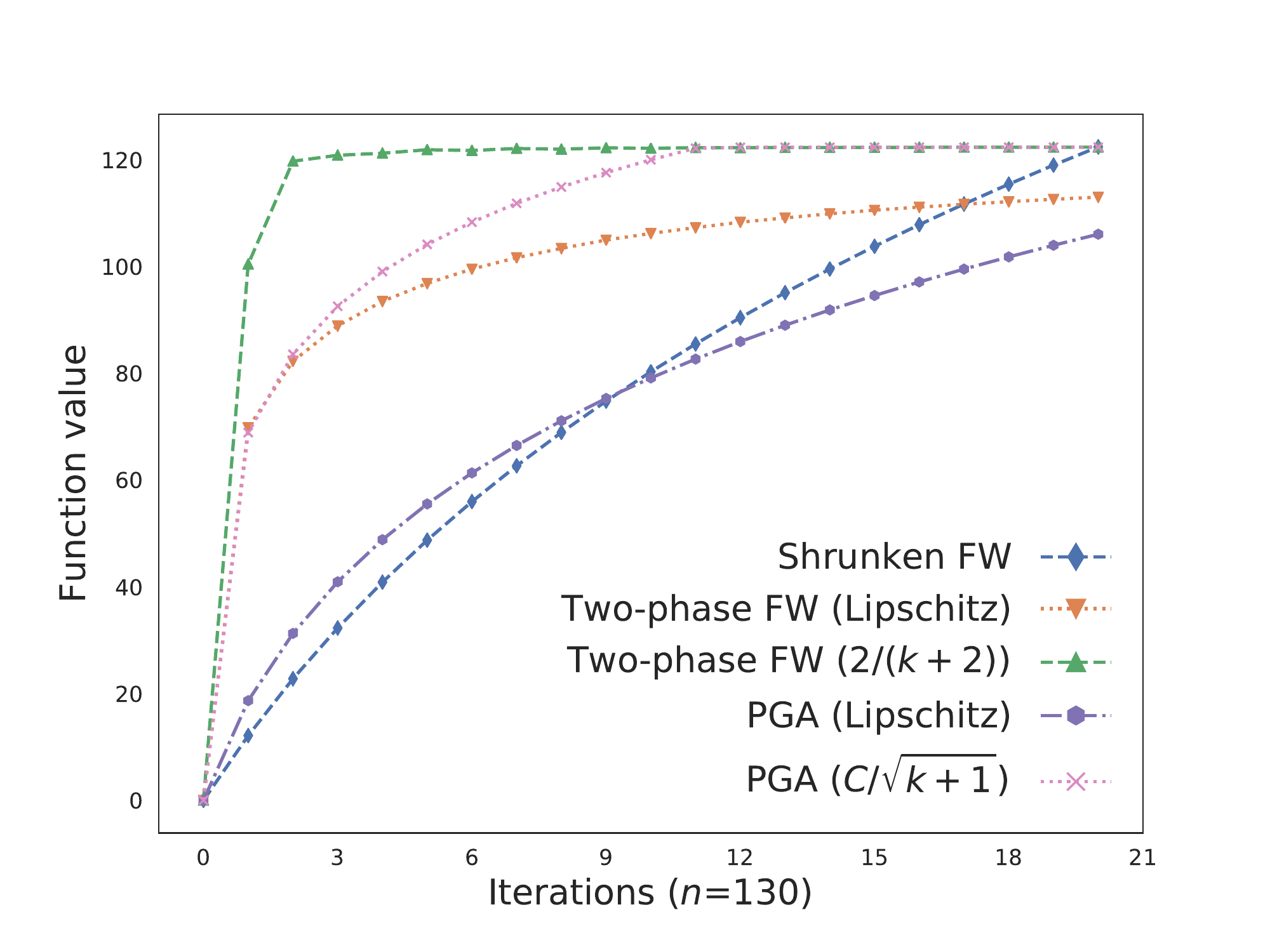}}
	\subfloat
	{
	\includegraphics[]{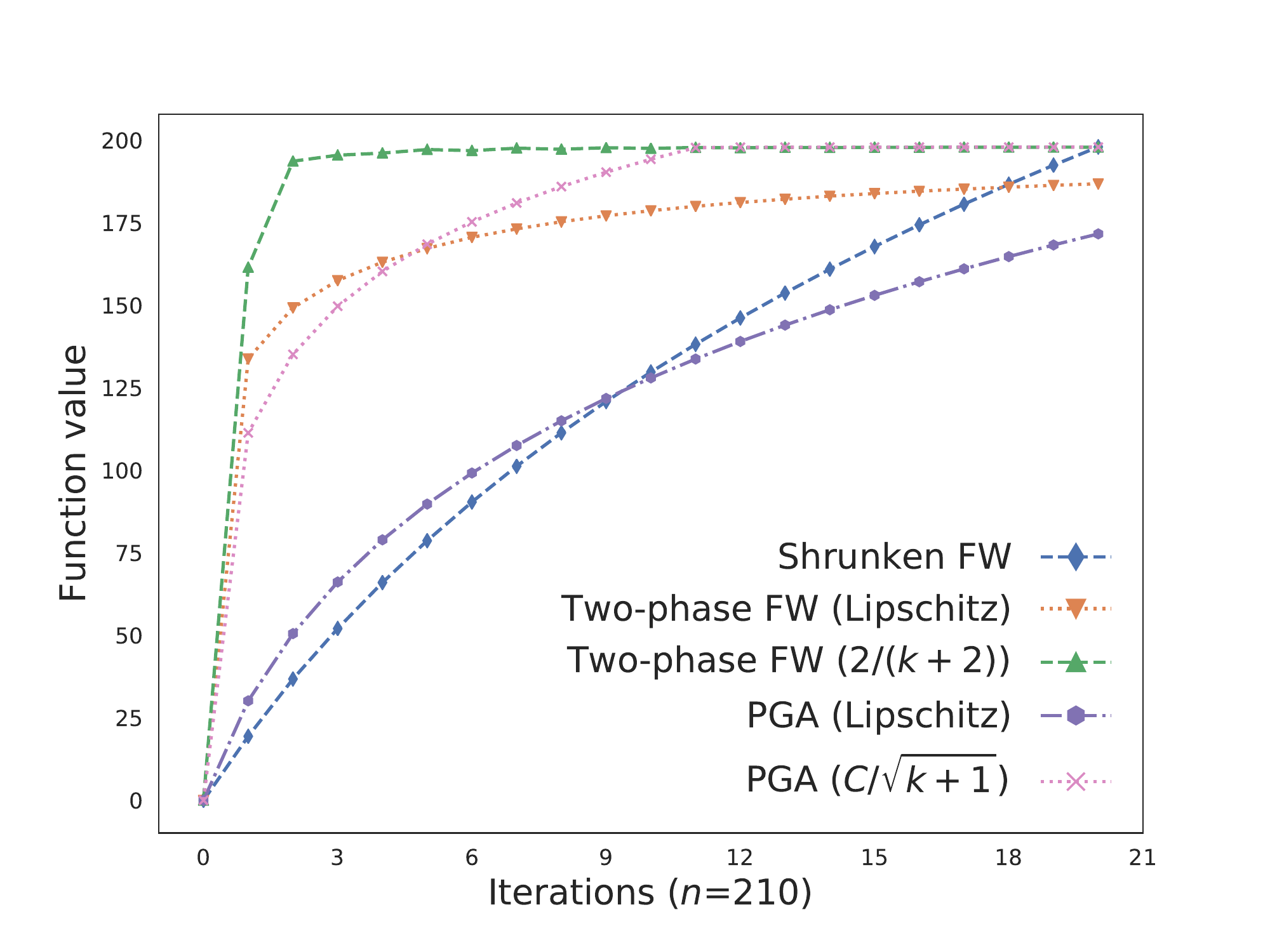}}
	\caption{Trajectories of different solvers  on  Softmax instances with one cardinality  constraint. Left: $n=50$; Middle: $n=130$; Right: $n=210$. \twophasefw has the fastest convergence.
\shrunkenfw converges slower, yet it always eventually returns a high
function value.  The performance of \pga highly depends on the
hyperparameters of the \stepsizes}
	\label{fig_softmax_syn}
\end{figure}

\paragraph{Results on Synthetic Data.}
We generate the softmax objectives (see \labelcref{eq_softmax}) in
the following way: first generate the $n$ eigenvalues $\d\in \R_+^n$,
each evenly distributed in $[0, 10]$, and set $\BD =
\diag(\d)$.
After generating a random unitary matrix $\BU$, we set
$\BL = \BU \BD\BU^\trans$.  One can verify that $\BL$ is positive
semidefinite and has eigenvalues as the entries of $\d$.
Then we generate one cardinality constraint in the form of
$\BA \x \leqco \b$, where $\BA =  \one^{1\times n}$ and
$\b = 0.5n$.

Function value trajectories returned by different solvers on problem instances with different dimensionalities are shown in
\cref{fig_softmax_syn}.
One can observe that \twophasefw has the fastest convergence.
\shrunkenfw converges slower, however it always eventually returns a high
function value.  The performance of \pga highly depends on the
hyperparameters of the \stepsizes.

\subsection{Revenue Maximization with Continuous Assignments}

We experiment with the model from \cref{app_special_case_of_ie} on several real-world graphs. Note that the objective of the simplified revenue maximization model is in general continuous submodular, and resulting optimization problem is a continuous submodular maximization problem with down-closed convex constraint.
For this problem setting,  the studied algorithms might not have a formal theoretical approximation guarantee. Yet, due to the practical usage of this application, it is worthwhile to use it as a robustness test of the algorithms in this section.

\subsubsection{Experimental Setting}

The  real-world graphs are from the Konect network
collection
\citep{kunegis2013konect}\footnote{\url{http://konect.uni-koblenz.de/networks}}
and the SNAP\footnote{\url{http://snap.stanford.edu/}} dataset.
The graph datasets and corresponding experimental parameters are
recorded in \cref{tab_dataset}.
We tested with the constraint that is the interaction of  one box constraint ($0 \leq x_i \leq u$) and one cardinality constraint $\one ^\trans \x \leq b$.

\begin{table}[htbp]
	\begin{center}
		\caption{Graph datasets and the corresponding experimental parameters. $n$ is the number of nodes, $q$ is the parameter of the model in \cref{app_special_case_of_ie}, $u$ is the upper bound of the box constraint, and $b$ is the budget.}
		\label{tab_dataset}
		\begin{tabularx}{\textwidth}{|r|X|X|X|l|X|}
			\hline
			Dataset name&   $n$  & \#edges & $q$ & $u$ &  budget $b$  \\
			\hline
			\hline
			``Reality Mining''  & 96 & 1,086,404 (multiedge) &   0.75 & 10  &   $0.2nu$  \\
			\hline
			``Residence hall'' & 217 & 2,672 & 0.75  & 10 & $0.4nu$  \\
			\hline
			``Infectious'' & 410 & 17,298 & 0.7 & 20 & $0.2nu$  \\
			\hline
			``U. Rovira i Virgili'' & 1,133  & 5,451& 0.8 & 20 & 		$0.2nu$  \\
			\hline
			``ego Facebook'' & 4,039  & 88,234& 0.9 & 40 & 		$0.1nu$  \\
			\hline
		\end{tabularx}
	\end{center}
\end{table}

For a specific example, the ``Reality Mining''
\citep{eagle2006reality}
dataset\footnote{\url{http://konect.uni-koblenz.de/networks/mit}, and\\
	\url{http://realitycommons.media.mit.edu/realitymining.html}}
contains the contact data of 96 persons through tracking 100 mobile
phones.  The dataset was collected in 2004 over the course of nine
months and represents approximately 500,000 hours of data on users'
location, communication and device usage behavior.
Here one contact could mean a phone call, Bluetooth sensor proximity
or physical location proximity.  We use the number of contacts as the
weight of an edge, by assuming that the more contacts happen between
two persons, the stronger the connection strength should be.

\subsubsection{Experimental Results}

\paragraph{Results on a Small Graph for Visualization.}

\looseness -1 First, we tested on a small graph, in order to clearly  visualize the results. We select a subgraph from the ``Reality Mining'' dataset
by taking the first five users/nodes, the nodes and number of contacts
amongst nodes are shown in \cref{fig_reality_subgraph}. For
illustration, we label the five users as ``A, B, C, D,
E''. One can see that there are different level of contacts between
different users, for example, there are 22,194 contacts between A and
B, while there are only 82 contacts between E and C.

\cref{fig_sub_Trajectories_reality} traces the trajectories of
different algorithms when maximizing the revenue objective. They were
all run for 20 iterations. One can see that \shrunkenfw and
\twophasefw reach higher revenue than \pga algorithms. Notice that
\shrunkenfw and \twophasefw with oblivious \stepsizes do not need to
tune any hyperparameters, while the others need to adapt the Lipschitz
parameter $L$ and the constant $C$ to determine the \stepsizes.

\setkeys{Gin}{width=0.52\textwidth}
\begin{figure}[htbp]
	\center
	\subfloat[The ``Reality Mining'' subgraph. \label{fig_reality_subgraph}]{
		\includegraphics[]{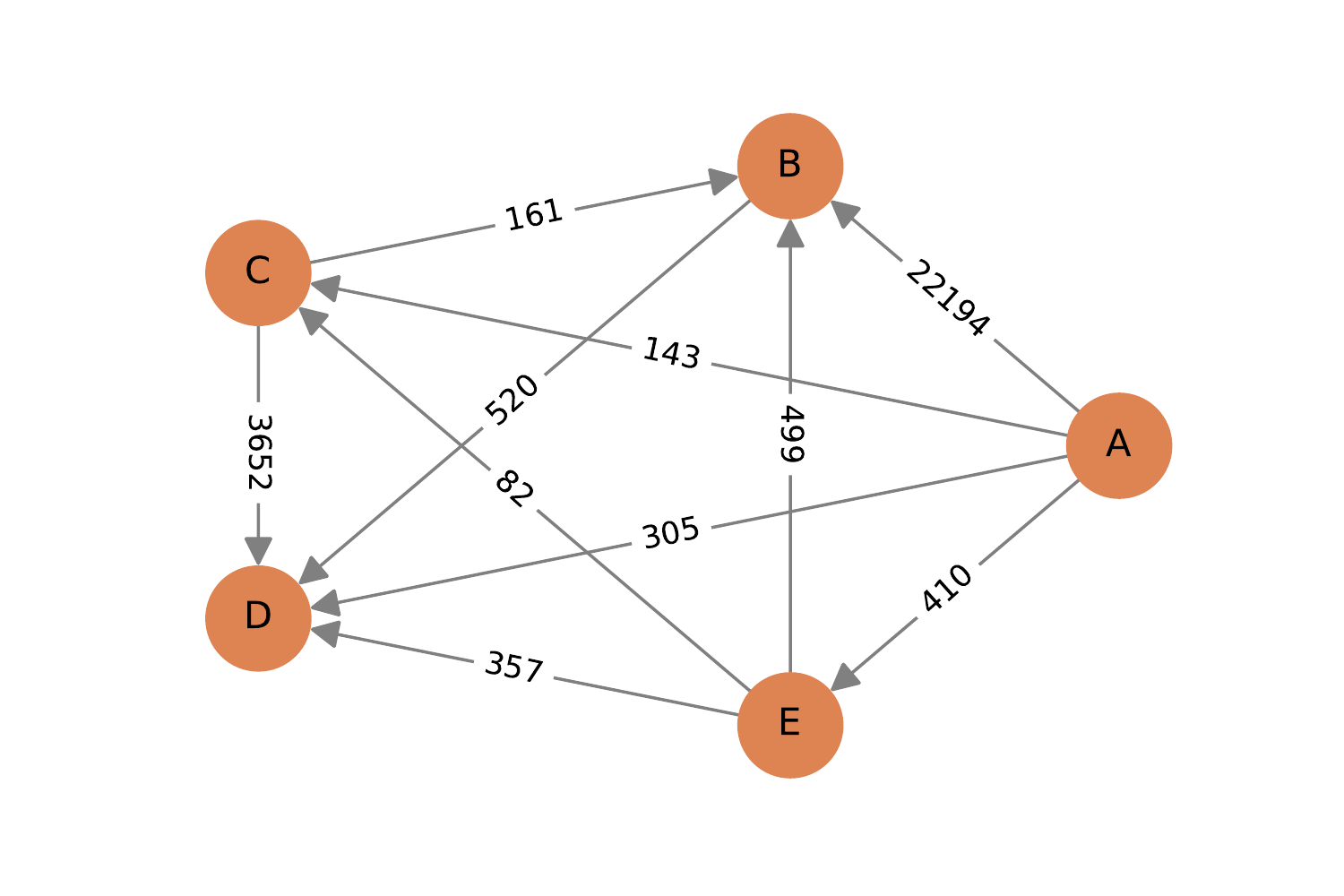}}
	\hspace{-0.99cm}
	\subfloat[Trajectories of algorithms with 20 iterations \label{fig_sub_Trajectories_reality}]{
		\includegraphics[]{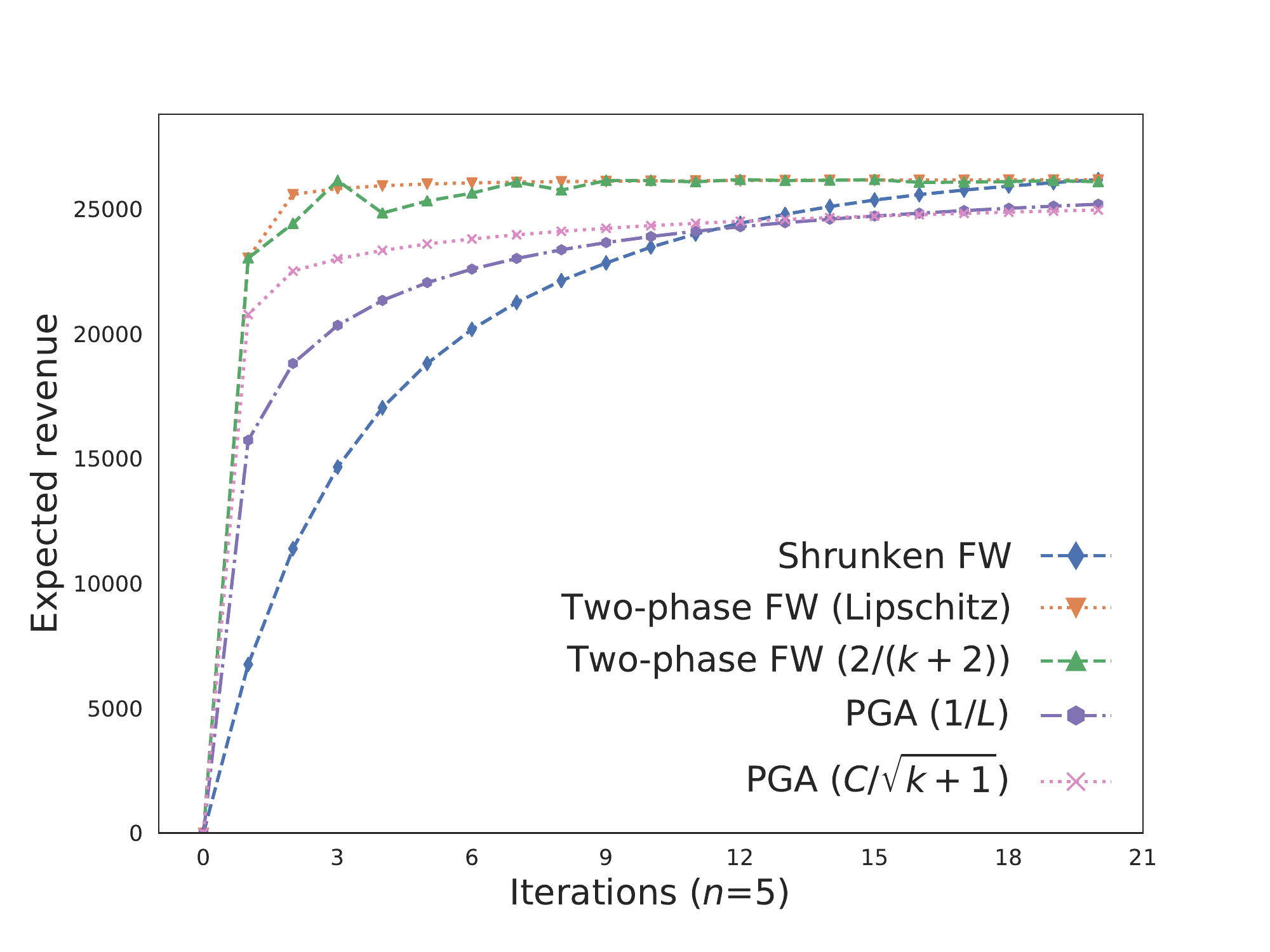}}
	\caption{Results  on  the ``Reality Mining'' subgraph with one cardinality constraint, where $u=10, b=0.2*n*u$.}
	\label{fig_reality}
\end{figure}

One may ask the question: \emph{How does the assignment look like for
different algorithms?} In order to show this behavior, we visualize the
assignments in \cref{fig_reality_assignments}.
One can see that \shrunkenfw assigns user A the most free products
(6.1), followed by user C (3.3), then user E (0.6). All other users
get $0$ assignment.  This is consistent with the intuition: one can
observe that user A most strongly influences others users (with total contacts as 22,194+
410 + 143), while user D exerts zero influence on others. \twophasefw
provides similar result, while \pga is conservative in assigning free
products to users.

\setkeys{Gin}{width=0.35\textwidth}
\begin{figure}[htbp]
	\center
	\subfloat{
		\includegraphics[]{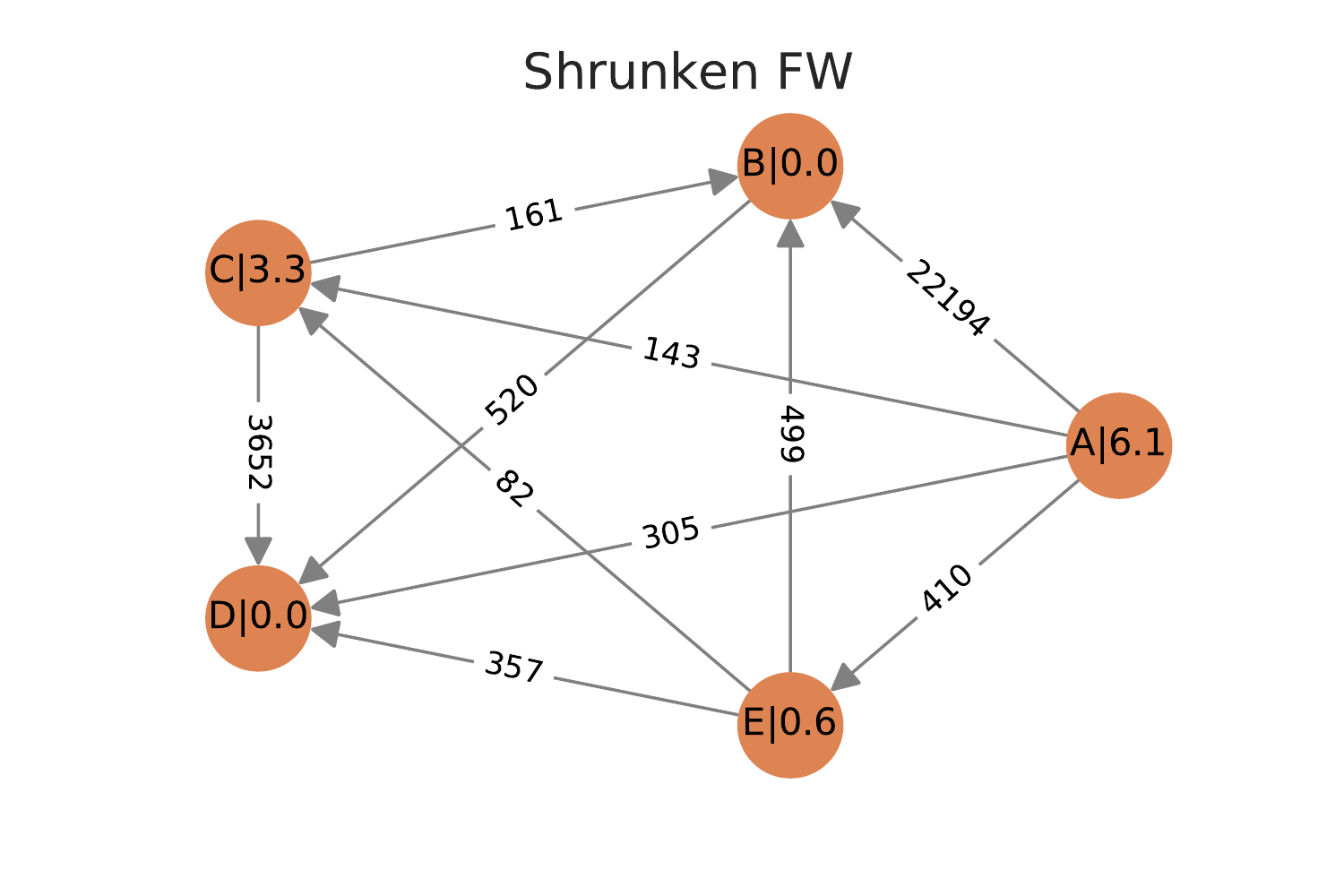}}
	\hspace{-0.99cm}
	\subfloat{
		\includegraphics[]{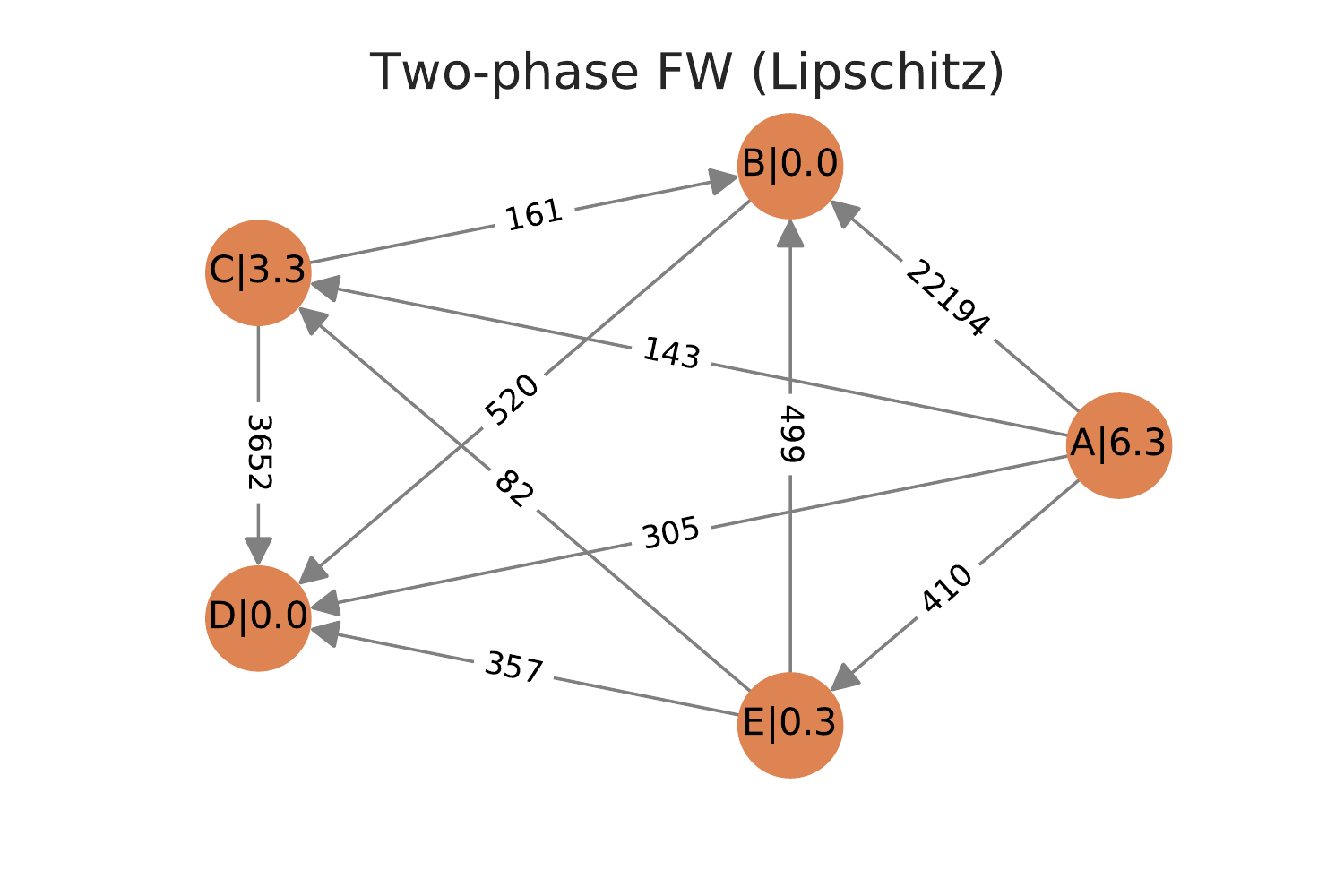}}
	\hspace{-0.99cm}
	\subfloat
	{
		\includegraphics[]{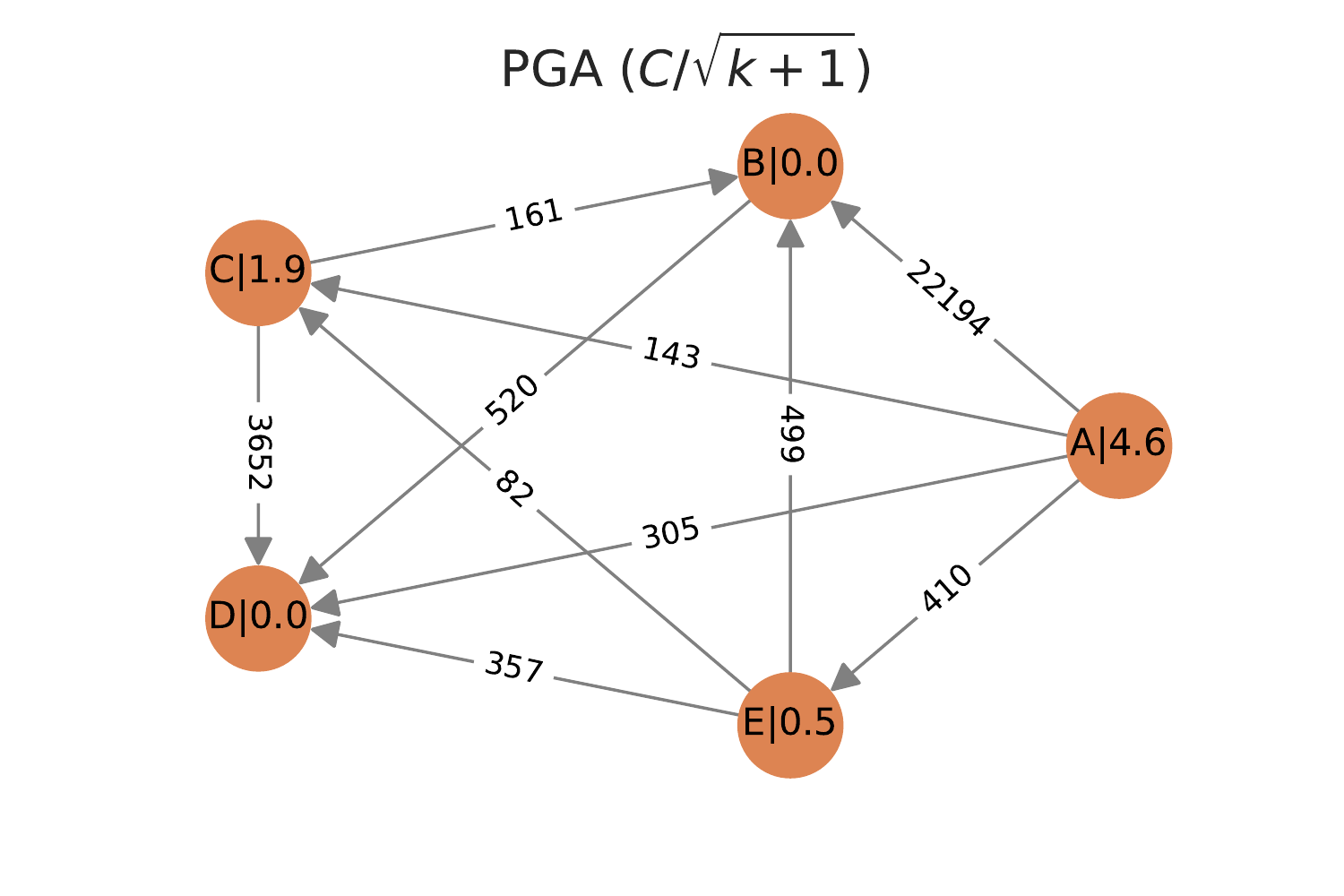}}
	\caption{Assignments to the users returned by different algorithms. \pga is more conservative in terms of assigning free
products to users than the other two algorithms: \shrunkenfw and \twophasefw.}
	\label{fig_reality_assignments}
\end{figure}

\paragraph{Results on Big Graphs.}

Then we looked at the behavior of the algorithms on the original big
graph, which is plotted in \cref{fig_traj_big_graph}, for real-world graphs with at most
$n = 4,039$ nodes.

One can observe that usually \twophasefw algorithm achieves the highest
objective value, and also converges with the fastest rate. \shrunkenfw converges
slower than \twophasefw, but it always reaches competitive function
value.  \pga algorithms
need to tune parameters for the \stepsize, and converges to lower
objective values.

\setkeys{Gin}{width=0.52\textwidth}
\begin{figure}[htbp]
	\center
	\subfloat
	[``Residence hall'' dataset]
	{
		\includegraphics[]{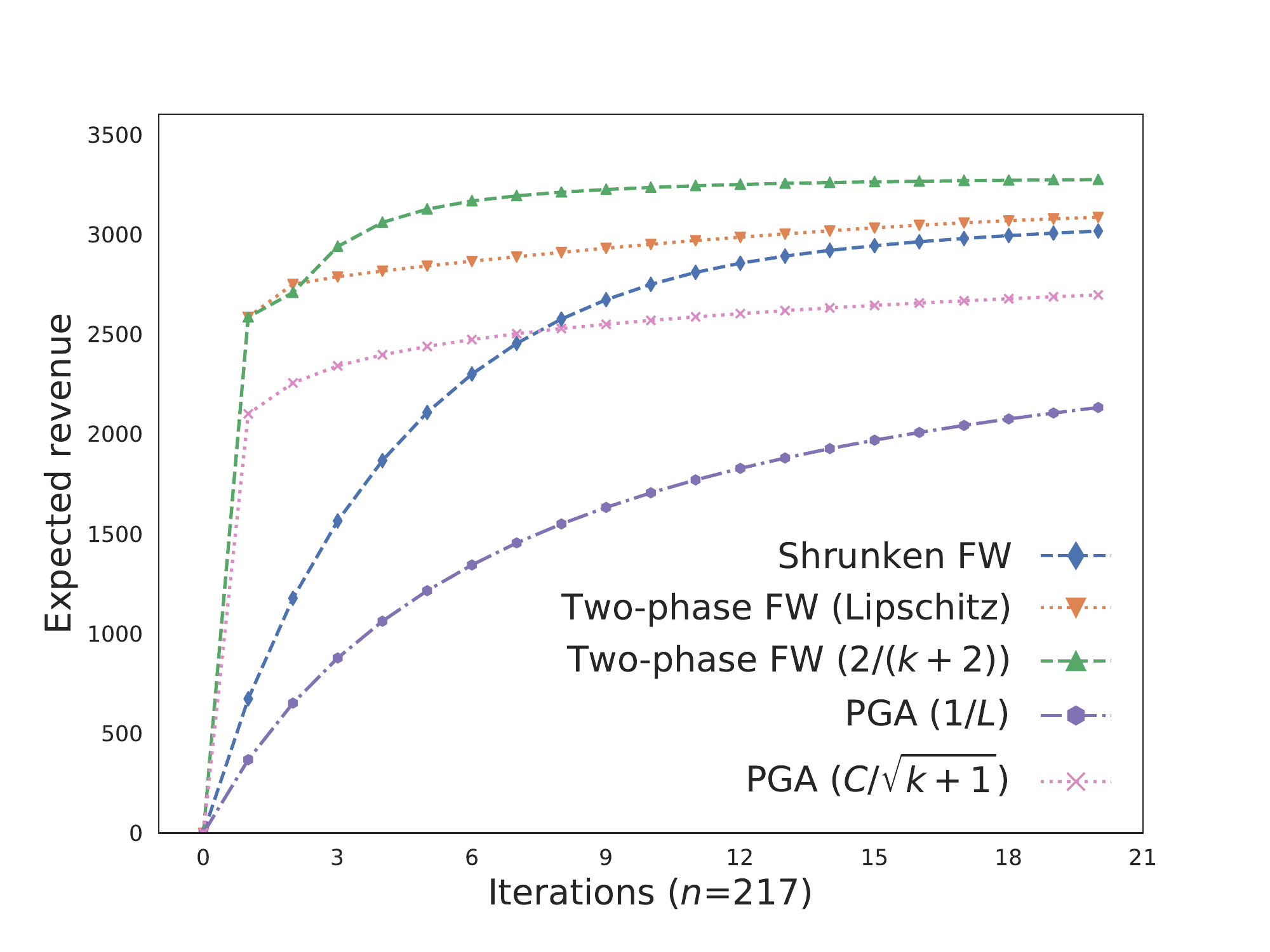}}
	\subfloat
	[``Infectious'' dataset]
	{
		\includegraphics[]{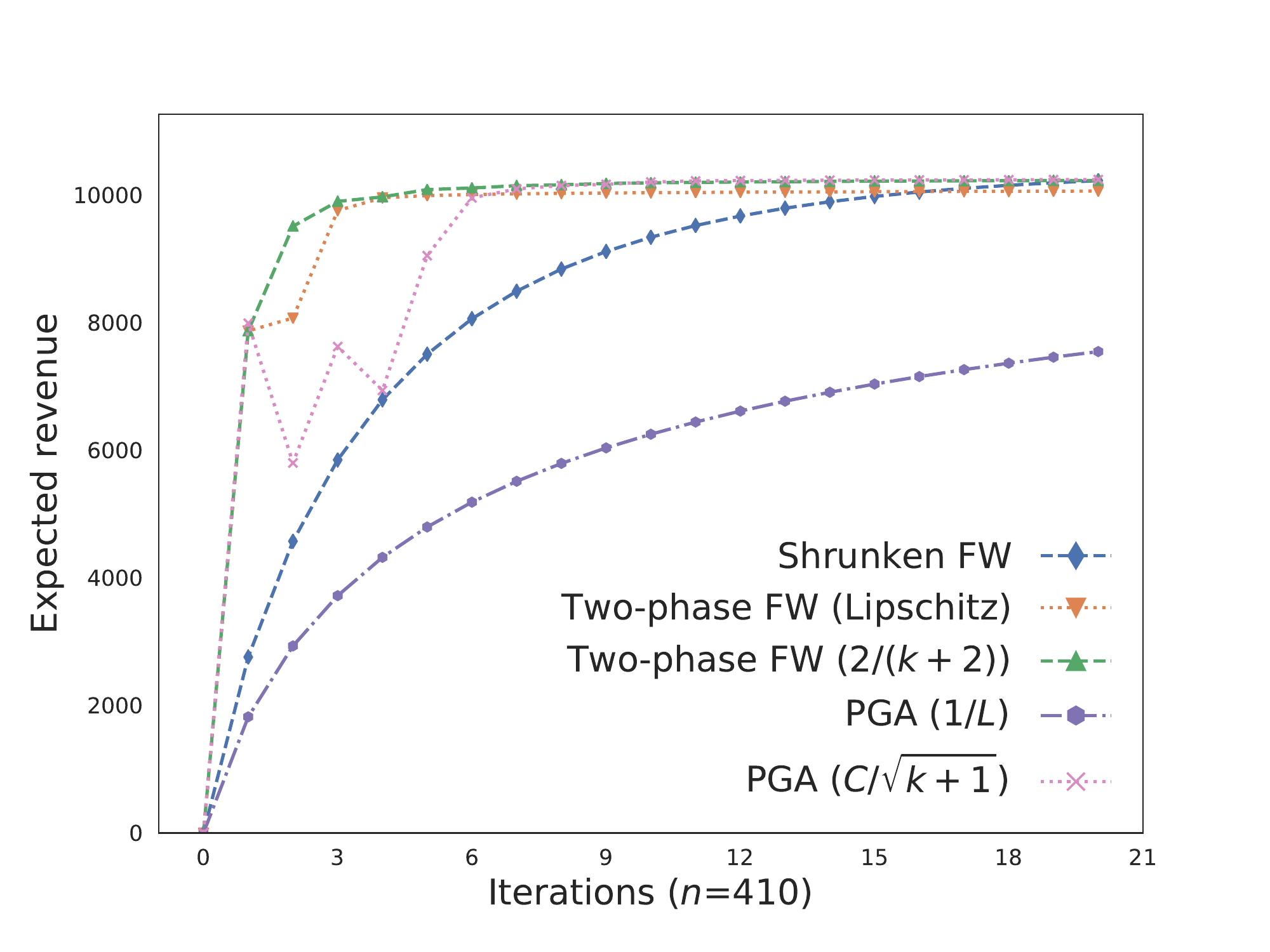}}\\
	\subfloat
	[``U. Rovira i Virgili'' dataset]
	{
		\includegraphics[]{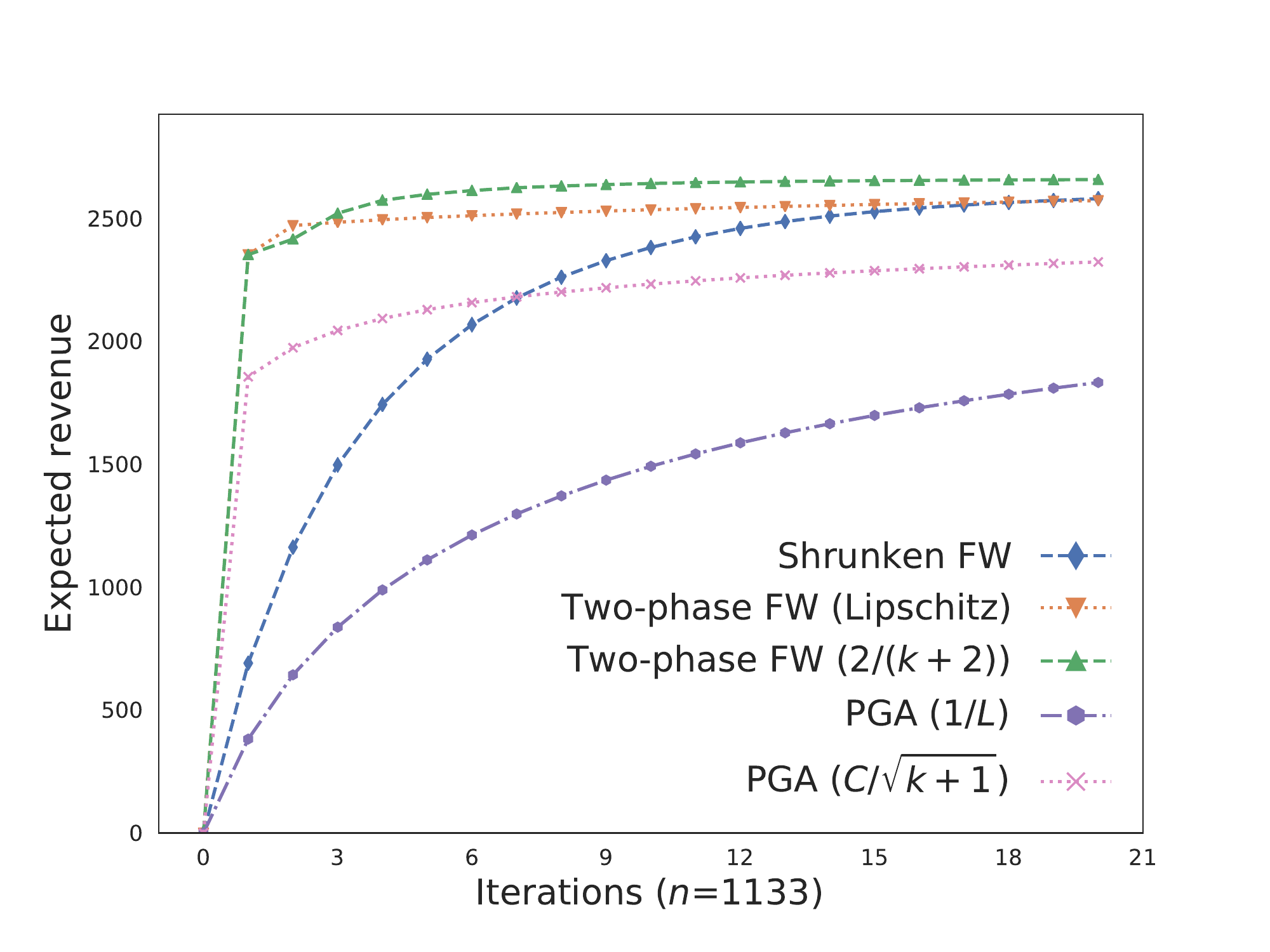}}
	\subfloat
	[``ego Facebook'' dataset]
	{
		\includegraphics[]{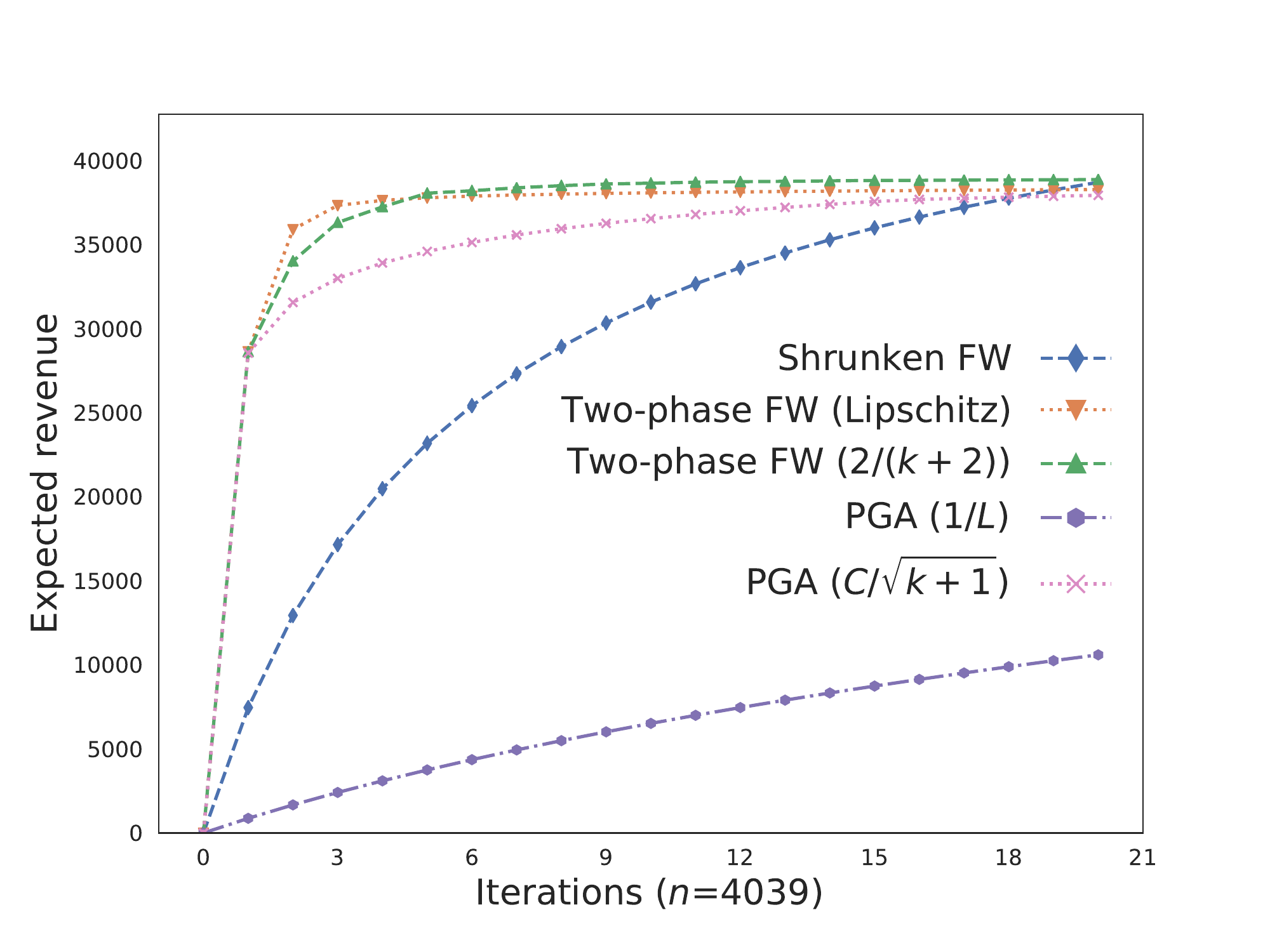}}
	\caption{Trajectory of different algorithms on real-world graphs. Usually \twophasefw  achieves the highest
objective value, and also converges with the fastest rate. \shrunkenfw converges
slower than \twophasefw, but it always reaches competitive function
value.  \pga algorithms
need to tune parameters for the \stepsize, and converges to lower objective values.}
	\label{fig_traj_big_graph}
\end{figure}

\section{Conclusion} %
\label{sec_discu}

In this work, we have systematically studied continuous submodularity and the problem of continuous (DR)-submodular maximization. With  rigorous characterizations and study of composition rules, we established important properties of this class of functions.
Based on of geometric properties of  continuous DR-submodular maximization, we proposed provable algorithms for both the monotone and non-monotone settings.
We also identified representative applications and demonstrated the effectiveness of the proposed algorithms on both synthetic and real-world experiments.

\clearpage

{
\bibliography{\bibpath}
}

\newpage
\appendix
\appendixtitle{Appendix}

\section{Proofs of Characterizations of Continuous Submodular  Functions}\label{app_proof}

Since $\X_i$ is a compact subset
of $\R$, we denote its lower bound and upper bound to be $\underline{u}_i$ and
$\bar u_i$, respectively.

\subsection{Proofs of \cref{lemma_dr_antitone} and
	\cref{lemma_weak_antitone}}

\begin{proof}[Proof of \cref{lemma_dr_antitone}]
	\emph{Sufficiency}: For any dimension $i$,
	\begin{align}
	\nabla_i f(\a) = \lim_{k\rightarrow 0} \frac{f(k\bas_i + \a)
		- f(\a)}{k}  \geq \lim_{k\rightarrow 0} \frac{f(k\bas_i +
		\b) - f(\b)}{k} = \nabla_i f(\a).
	\end{align}

	\emph{Necessity}:

	Firstly, we show that for any $\c \geqco \zero$, the function
	$g(\x) := f(\c + \x) - f(\x)$ is monotonically non-increasing.
	\begin{align}
	\nabla g(\x) =  \nabla f(\c + \x) - \nabla f(\x) \leqco \zero.
	\end{align}

	Taking $\c = k\bas_i$, since $g(\a) \leq g(\b)$, we reach the
	DR-submodularity definition.
\end{proof}

\begin{proof}[Proof of \cref{lemma_weak_antitone}]
	Similar as the proof of \cref{lemma_dr_antitone}, we have the
	following:

	\emph{Sufficiency}: For any dimension $i$
	$\text{ s.t. } a_{i} = b_{i}$,
	\begin{align}
	\nabla_i f(\a) = \lim_{k\rightarrow 0} \frac{f(k\bas_i + \a) -
		f(\a)}{k}  \geq \lim_{k\rightarrow 0} \frac{f(k\bas_i + \b) -
		f(\b)}{k} = \nabla_i f(\a).
	\end{align}

	\emph{Necessity}:

	We show that for any $k \geq 0$, the function
	$g(\x) := f(k\bas_i + \x) - f(\x)$ is monotonically non-increasing.
	\begin{align}
	\nabla g(\x) =  \nabla f(k\bas_i + \x) - \nabla f(\x) \leqco \zero.
	\end{align}

	Since $g(\a) \leq g(\b)$, we reach the weak DR definition.
\end{proof}

\subsection{Alternative Formulation of the \NEWDR\ Property}

First of all, we will prove that \NEWDR\ has the following alternative
formulation, which will be used to prove Proposition \ref{lemma_support_dr}.
\begin{lemma}[Alternative formulation of \NEWDR]
	The \NEWDR\ property (\cref{def_supp_dr2}, denoted as
	\emph{Formulation I}) has the following equilvalent formulation
	(\cref{def_supp_dr}, denoted as \emph{Formulation II}):
	$\forall \a\leqco \b\in \X$,
	$\forall i\in \{i'|a_{i'} = b_{i'}=\underline{u}_{i'} \}, \forall
	k'\geq l'\geq 0$
	s.t. $(k'\chara_i + \a)$, $(l'\chara_i + \a)$, $(k'\chara_i + \b)$
	and $(l'\chara_i + \b)$ are still in $\X$, the following inequality
	is satisfied,
	\begin{equation} \label{def_supp_dr} f(k'\chara_i + \a) -
	f(l'\chara_i + \a) \geq f(k'\chara_i+ \b) - f(l'\chara_i+ \b).
	\quad \emph{(Formulation II)}
	\end{equation}
\end{lemma}
\begin{proof}

	Let $D_1 = \{i| a_i = b_i = \underline{u}_i \}$,
	$D_2 = \{i|\underline{u}_i < a_i = b_i < \bar u_i \}$, and
	$D_3 = \{i| a_i = b_i = \bar u_i \}$.

	1) \texttt{Formulation II} $\Rightarrow$ \texttt{Formulation I}

	When $i\in D_1$, set $l' = 0$ in \texttt{Formulation II} one can get
	$f(k'\chara_i+ \a) - f(\a) \geq f(k'\chara_i+ \b) - f(\b)$.

	When $i\in D_2$, $\forall k\geq 0$, let
	$l' = a_i - \underline{u}_i = b_i- \underline{u}_i >0,$
	$k' = k + l' = k +(a_i - \underline{u}_i)$, and let
	$\bar \a = (\sete{a}{i}{\underline{u}_i}), \bar \b =
	(\sete{b}{i}{\underline{u}_i})$.
	It is easy to see that $\bar \a \leqco \bar \b$, and
	$\bar a_i = \bar{b}_i = \underline{u}_i$. Then from
	\texttt{Formulation II},
	\begin{flalign}
	& f(k'\chara_i + \bar \a) - f(l'\chara_i + \bar \a) = f(k\chara_i
	+ \a) - f(\a) \\\notag
	& \geq f(k'\chara_i + \bar \b) -
	f(l'\chara_i + \bar \b) = f(k\chara_i + \b) - f(\b).
	\end{flalign}
	When $i\in D_3$, \cref{def_supp_dr2} holds trivially.

	The above three situations proves the \texttt{Formulation I}.

	2) \texttt{Formulation II} $\Leftarrow$ \texttt{Formulation I}

	$\forall \a\leqco \b$, $\forall i\in D_1$, one has
	$a_i = b_i = \underline{u}_i$.
	$\forall k'\geq l' \geq 0$, let
	$\hat \a = l'\chara_i + \a, \hat \b = l'\chara_i + \b$, let
	$k = k'-l' \geq 0$, it can be verified that $\hat \a\leqco \hat \b$
	and $\hat a_i = \hat b_i$, from \texttt{Formulation I},
	\begin{flalign}
	&f(k\chara_i + \hat \a) - f(\hat \a) = f(k'\chara_i + \a) -
	f(l'\chara_i + \a)\\\notag \geq
	& f(k\chara_i + \hat \b) - f(\hat
	\b) = f(k'\chara_i + \b) - f(l'\chara_i + \b).
	\end{flalign}
	which proves \texttt{Formulation II}.
\end{proof}

\subsection{Proof of Proposition \ref{lemma_support_dr}}

\begin{proof}

	1) \texttt{submodularity} $\Rightarrow$ \texttt{weak DR}:

	Let us prove the \texttt{Formulation II} (\cref{def_supp_dr}) of
	\texttt{weak DR}, which is,

	$\forall \a\leqco \b\in \X$,
	$\forall i\in \{i'|a_{i'} = b_{i'}=\underline{u}_{i'} \}, \forall
	k'\geq l'\geq 0$, the following inequality holds,
	\begin{equation}
	f(k'\chara_i+ \a) - f(l'\chara_i + \a) \geq f(k'\chara_i+ \b) - f(l' \chara_i + \b).
	\end{equation}

	And $f$ is a submodular function iff $\forall \x, \y\in \X$,
	$f(\x)+f(\y) \geq f( \x \vee \y) + f(\x \wedge \y)$, so
	$f(\y) - f(\x\wedge \y) \geq f(\x\vee \y) - f(\x)$.

	Now $\forall \a \leqco \b \in\X$, one can set $\x = l'\chara_i + \b$
	and $\y = k'\chara_i + \a$. It can be easily verified that
	$\x\wedge \y =l'\chara_i + \a$ and $\x\vee \y = k'\chara_i + \b$.
	Substituting all the above equalities into
	$f(\y) - f(\x\wedge \y) \geq f(\x\vee \y) - f(\x)$ one can get
	$f(k'\chara_i+ \a) - f(l'\chara_i + \a) \geq f(k'\chara_i+ \b) - f(l' \chara_i + \b)$.

	2) \texttt{submodularity}
	$\Leftarrow$  \texttt{weak DR}:

	Let us use \texttt{Formulation I} (\cref{def_supp_dr2}) of
	\texttt{weak DR} to prove the \texttt{submodularity} property.

	$\forall \x, \y\in \X$, let $D := \{e_1, \cdots, e_d\}$  be the
	set of elements for which $y_e > x_e$, let
	$k_{e_i}: = y_{e_i} - x_{e_i}$.  Now set
	$\a^0 := \x\wedge \y, \b^0 := \x$ and
	$\a^i = (\sete{a^{i-1}}{e_i}{y_{e_i}}) = k_{e_i}\chara_i +
	\a^{i-1}, \b^i = (\sete{b^{i-1}}{e_i}{y_{e_i}}) = k_{e_i}\chara_i
	+ \b^{i-1}$,
	for $i = 1, \cdots, d$.

	One can verify that
	$\a^i\leqco \b^i, a^i_{e_{i'}} = b^i_{e_{i'}}$ for all
	$i'\in D, i=0, \cdots, d$, and that
	$\a^d = \y, \b^d = \x\vee \y$.

	Applying \cref{def_supp_dr2} of the \texttt{weak DR} property for
	$i = 1,\cdots, d$ one can get
	\begin{flalign}
	&f(k_{e_1}\chara_{e_1} + \a^0) - f(\a^0) \geq
	f(k_{e_1}\chara_{e_1} + \b^0) - f(\b^0) \\
	&f(k_{e_2}\chara_{e_2} + \a^1) - f(\a^1) \geq
	f(k_{e_2}\chara_{e_2} + \b^1) - f(\b^1) \\\notag
	&\cdots\\
	&f(k_{e_d}\chara_{e_d} + \a^{d-1}) -
	f(\a^{d-1}) \geq f(k_{e_d}\chara_{e_d} + \b^{d-1}) -
	f(\b^{d-1}).
	\end{flalign}
	Taking a sum over all the above $d$ inequalities, one can get
	\begin{flalign}
	& f(k_{e_d}\chara_{e_d} + \a^{d-1}) - f(\a^{0}) \geq
	f(k_{e_d}\chara_{e_d} + \b^{d-1}) - f(\b^{0})\\\notag
	&  \Leftrightarrow\\
	& f(\y) - f(\x\wedge \y) \geq f(\x\vee
	\y) - f(\x)\\\notag
	&  \Leftrightarrow\\
	& f(\x) + f(\y) \geq
	f(\x\vee \y) + f(\x\wedge \y),
	\end{flalign}
	which proves the submodularity property.
\end{proof}

\subsection{Proof of Proposition \ref{lemma_dr}}

\begin{proof}

	1) \texttt{submodular} + \texttt{coordinate-wise concave}
	$\Rightarrow$ \texttt{DR}:

	From coordinate-wise concavity we have $f(\a + k\chara_i) - f(\a) \geq f(\a+(b_i - a_i + k)\chara_i) - f(\a+(b_i - a_i)\chara_i)$. Therefore, to prove \text{DR} it suffices to show that
	\begin{flalign}\label{eq_12}
	f(\a+(b_i - a_i + k)\chara_i) - f(\a+(b_i - a_i)\chara_i) \geq f(\b + k\chara_i) - f(\b).
	\end{flalign}
	Let $\x:=\b, \y:=(\a+(b_i - a_i + k)\chara_i)$, so $\x\wedge\y  = (\a+(b_i - a_i)\chara_i), \x\vee \y = (\b + k\chara_i)$.
	From submodularity, one can see that inequality \labelcref{eq_12} holds.

	2) \texttt{DR} $\Rightarrow$ \texttt{submodular} + \texttt{coordinate-wise concave}:

	From \texttt{DR} property, the  \texttt{weak DR} (\cref{def_supp_dr2}) property is implied, which
	equivalently proves the \textit{submodularity} property.

	To prove \textit{coordinate-wise concavity}, one just need to set $\b:=\a+l\chara_i$, then we have  $f(\a + k\chara_i) - f(\a) \geq f(\a + (k+l)\chara_i) - f(\a + l\chara_i)$.
\end{proof}

\section{Proofs for Properties of Continuous DR-Submodular Maximization}
\label{app_proofs_struc_algs}

\subsection{Proof of \cref{lemma_separable_reparameterization}}
\label{app_proof_separable_reparameterization}

\begin{proof}[Proof of \cref{lemma_separable_reparameterization}]

Supppose for simplicity that $f$ and $h$ are both twice differentiable. Note that when  $f$ and $h$ are not differentiable, one can similarly prove the conclusion using \ith{\text{zero}} order definition of continuous submodularity.

Without loss of generality,  let us prove that $f(h(\x))$ maintains submodularity of $f$.  One just need to show that the term $\fracppartial{g(\x)}{x_i}{x_j}$
in  \cref{eq_ijth} is non-positive when $i \neq j$.

Firstly, let us consider the term $\sum_{k=1}^n \fracpartial{f(\y)}{y_k}  \fracppartial{h^k(\x)}{x_i}{x_j}$. Since $h$ is separable as stated above, $\fracppartial{h^k(\x)}{x_i}{x_j}$ is always zero, so $\sum_{k=1}^n \fracpartial{f(\y)}{y_k}  \fracppartial{h^k(\x)}{x_i}{x_j}$ is always zero.

Then it remains to show that the term $\sum_{s,  t = 1}^n  \fracppartial{f(\y)}{y_s}{y_t}  \fracpartial{h^s(\x)}{x_i} \fracpartial{h^t(\x)}{x_j}$ is non-positive. There are two situations: 1) $s=t$. Since $i\neq j$, there must be one term out of $\fracpartial{h^s(\x)}{x_i}$ and $\fracpartial{h^t(\x)}{x_j}$ that are zero (because $h$ is separable). 2) $s\neq t$. Since $f$ is submodular, it holds that $ \fracppartial{f(\y)}{y_s}{y_t}\leq 0$. Because $h$ is monotone, it also holds that  $\fracpartial{h^s(\x)}{x_i} \fracpartial{h^t(\x)}{x_j} \geq 0$. So the term $\sum_{s,  t = 1}^n  \fracppartial{f(\y)}{y_s}{y_t}  \fracpartial{h^s(\x)}{x_i} \fracpartial{h^t(\x)}{x_j}$ is non-positive in the above two situations.

Now we reach the conclusion that $f(h(\x))$ maintains submodularity of $f$.
\end{proof}

\subsection{Proof of  \cref{prop_concave}}\label{supp_prop_concave}

\begin{proof}[Proof of  \cref{prop_concave}]
	Consider a univariate  function
	\begin{align}
	g(\xi):= f(\x+\xi \v^*), \xi\geq 0, \v^* \geqco \zero.
	\end{align}

	We know that
	\begin{align}
	\frac{d g(\xi)}{d \xi} = \dtp{\v^*}{\nabla f(\x+\xi \v^*)}.
	\end{align}

	It can be verified that:

	$g(\xi)$ is concave $\Leftrightarrow$
	\begin{flalign}
	\frac{d^2 g(\xi)}{d \xi^2} = (\v^*)^\trans \nabla^2 f(\x+\xi \v^*) \v^* = \sum_{i\neq j} v^*_i v^*_j \nabla^2_{ij} f + \sum_i (v_i^*)^2\nabla_{ii}^2f \leq 0.
	\end{flalign}
	The non-positiveness of $\nabla^2_{ij} f $ is ensured by submodularity of $f(\cdot)$, and the non-positiveness of $\nabla^2_{ii} f $ results from the coordinate-wise concavity of $f(\cdot)$.

	The proof of concavity along any non-positive direction is similar, which is omitted here.
\end{proof}

\subsection{Proof of \cref{lemma_3_1}}

\begin{proof}[Proof of \cref{lemma_3_1}]
	Since $f$ is DR-submodular, so it is concave along any direction $\v\in \pm \R^n_+$. We know that $\x\vee \y - \x \geqco \zero$
	and $\x\wedge \y - \x\leqco \zero$, so from the strong DR-submodularity in \labelcref{eq_strong_dr},
	\begin{align}
	& f(\x\vee\y)  - f(\x) \leq \dtp{\nabla f(\x)}{\x\vee \y - \x}  -\frac{\mu}{2}  \|\x\vee \y - \x\|^2,\\
	&  f(\x\wedge \y) - f(\x) \leq \dtp{\nabla f(\x)}{\x\wedge \y - \x}  -\frac{\mu}{2}  \|\x\wedge  \y - \x\|^2.
	\end{align}
	Summing the above two inequalities and notice that $\x\vee\y + \x\wedge \y = \x+\y$, we arrive,
	\begin{flalign}
	& (\y-\x)^{\trans}\nabla f(\x)\\\notag
	& \geq f(\x\vee\y) + f(\x\wedge \y) - 2f(\x) + \frac{\mu}{2}  (\|\x\vee \y - \x\|^2 + \|\x\wedge  \y - \x\|^2)\\
	& = f(\x\vee\y) + f(\x\wedge \y) - 2f(\x) + \frac{\mu}{2}  \| \y - \x\|^2,
	\end{flalign}
	the last equality holds since $\|\x\vee \y - \x\|^2 + \|\x\wedge  \y - \x\|^2 =  \| \y - \x\|^2$.
\end{proof}

\subsection{Proof of  \cref{local_global}}\label{app_claim_proof}

\begin{proof}[Proof of \cref{local_global}]
	Consider the point $\z^*:= \x\vee \x^* -\x = (\x^* - \x)\vee \zero$. One can see that: 1) $\zero\leqco \z^* \leqco \x^*$; 2) $\z^* \in \P$ (down-closedness); 3) $\z^*\in \Q$ (because of   $\z^*\leqco \bar \u - \x$).
	From \cref{lemma_3_1},
	\begin{align}\label{eq_1718}
	& \dtp{\x^*-\x}{\nabla f(\x)} +  2f(\x) \geq f(\x\vee \x^*) + f(\x \wedge \x^*) +  \frac{\mu}{2}\|\x -\x^*\|^2, \\\label{eq12}
	& \dtp{\z^*-\z}{\nabla f(\z)} +  2f(\z) \geq f(\z\vee \z^*) + f(\z \wedge \z^*) +  \frac{\mu}{2}\|\z -\z^*\|^2.
	\end{align}
	Let us first of all prove the following  key \namecref{claim_key}.

	\keyclaim*

	\begin{proof}[Proof of \cref{claim_key}]
		Firstly, we are going to prove that
		\begin{align}\label{proof_part1}
		f(\x \vee \x^*) + f(\z\vee \z^*) \geq f(\z^*) + f((\x+\z)\vee \x^*),
		\end{align}
		which is equivalent to
		$f(\x \vee \x^*) - f(\z^*) \geq f((\x+\z)\vee \x^*) - f(\z\vee \z^*)$.
		It can be shown that  $\x \vee \x^*  - \z^* = (\x+\z)\vee \x^* - \z\vee \z^* $. Combining this with
		the fact that $\z^* \leqco \z\vee \z^*$, and using the DR property (see \cref{def_dr}) implies
		\labelcref{proof_part1}.
		Then we establish,
		\begin{align}\label{eq_EqaulityPoints}
		\x \vee \x^*  - \z^* = (\x+\z)\vee \x^* - \z\vee \z^* ~.
		\end{align}
		We will show that both the RHS and LHS of the above equation are equal to $\x$:  for the LHS of \labelcref{eq_EqaulityPoints} we can write
		$\x \vee \x^*  - \z^* =  \x \vee \x^*  - \left(  \x \vee \x^* - \x\right) = \x$.
		For the RHS of \labelcref{eq_EqaulityPoints} let us consider any coordinate $i\in [n]$,
		\begin{align}\notag
		&(x_i+z_i)\vee x_i^* - z_i\vee z_i^* =\\
		& (x_i+z_i)\vee x_i^* - \left((x_i+z_i)-x_i\right)\vee  \left((x_i \vee x_i^*) - x_i\right) =x_i,
		\end{align}
		where the last equality holds easily for the two situations: $(x_i+z_i) \geq  x_i^*$ and $(x_i+z_i) < x_i^*$.

		Next, we are going to prove that,
		\begin{align}\label{proof_part2}
		f(\z^*) + f(\x\wedge \x^*)\geq f(\x^*) + f(\zero).
		\end{align}
		It is equivalent to
		$f(\z^*)   - f(\zero) \geq  f(\x^*) - f(\x\wedge \x^*)$,
		which can be done similarly by the DR property: Notice that
		\begin{align}
		\x^* - \x\wedge \x^* = \x\vee \x^* - \x = \z^* - \zero \text{ and }
		\zero \leqco  \x\wedge \x^*.
		\end{align}
		Thus \labelcref{proof_part2} holds from the DR property.
		Combining \labelcref{proof_part1,proof_part2} one can get,
		\begin{align}\notag
		& f(\x \vee \x^*) + f(\z\vee \z^*) + f(\x\wedge \x^*) + f(\z\wedge \z^*) \\
		& \geq  f(\x^*) + f(\zero) +  f((\x+\z)\vee \x^*)+ f(\z\wedge \z^*)\\\notag
		& \geq f(\x^*).    \quad \text{(non-negativity of $f$) }
		\end{align}
	\end{proof}

	Combining \labelcref{eq_1718,eq12} and \cref{claim_key} it reads,
	\begin{align}\label{eq16}
	& \dtp{\x^* -\x}{\nabla f(\x)} +  \dtp{\z^*-\z}{\nabla f(\z)}  +   2(f(\x) + f(\z) ) \\ & \geq f(\x^*) +
	\frac{\mu}{2}(\|\x -\x^*\|^2 + \|\z -\z^*\|^2).
	\end{align}

	From the definition of non-stationarity in \labelcref{non_stationary} one can get,
	\begin{align}\label{eq17}
	&  g_{\P}(\x) := \max_{\v\in\P}\dtp{\v - \x}{\nabla f(\x)} \overset{\x^*\in \P}{\geq}  \dtp{\x^*-\x}{\nabla f(\x)},\\\label{eq18}
	& g_{\Q}(\z) := \max_{\v\in\Q}\dtp{\v - \z}{\nabla f(\z)}  \overset{\z^*\in \Q}{\geq} \dtp{\z^*-\z}{\nabla f(\z)}.
	\end{align}
	Putting together \cref{eq16,eq17,eq18} we can get,
	\begin{align}
	2(f(\x) + f(\z) ) \geq f(\x^*) -g_{\P}(\x) -g_{\Q}(\z) +   \frac{\mu}{2}(\|\x -\x^*\|^2 + \|\z -\z^*\|^2).
	\end{align}
	So it arrives
	\begin{flalign}
	& \max\{f(\x), f(\z) \} \geq \\
	&\frac{1}{4}[f(\x^*) -g_{\P}(\x) -g_{\Q}(\z)]  +   \frac{\mu}{8}(\|\x -\x^*\|^2 + \|\z -\z^*\|^2).
	\end{flalign}
\end{proof}

\section{Additional Details for  Monotone DR-Submodular  Maximization}
\label{supp_proof_monotone_max}

\subsection{Proof of \cref{prop_np}}

\begin{proof}[Proof of \cref{prop_np}]

	On a high level, the proof idea follows from the reduction from the
	problem of maximizing a monotone submodular set function subject to
	cardinality constraints.

	Let us denote $\Pi_1$ as the problem of maximizing a monotone
	submodular set function subject to cardinality constraints, and
	$\Pi_2$ as the problem of maximizing a monotone continuous
	DR-submodular function under general down-closed polytope
	constraints.
	Following \citet{DBLP:journals/siamcomp/CalinescuCPV11}, there exist
	an algorithm $\A$ for $\Pi_1$ that consists of a polynomial time
	computation in addition to polynomial number of subroutine calls to
	an algorithm for $\Pi_2$. For details on $\A$ see the following.

	First of all, the multilinear extension
	\citep{calinescu2007maximizing} of a monotone submodular set
	function is a monotone continuous submodular function, and it is
	coordinate-wise linear, thus falls into a special case of monotone
	continuous DR-submodular functions. Evaluating the multilinear extension and its gradients can be done using sampling methods, thus resulted in a randomized algorithm.

	So the algorithm $\A$ shall be: 1) Maximize the multilinear
	extension of the submodular set function over the matroid polytope
	associated with the cardinality constraint, which can be achieved by
	solving an instance of $\Pi_2$.  We call the solution obtained the
	fractional solution; 2) Round the fractional solution to a feasible
	integeral solution using polynomial time rounding technique in
	\citet{ageev2004pipage,calinescu2007maximizing} (called the pipage
	rounding). Thus we prove the reduction from $\Pi_1$ to $\Pi_2$.

	Our reduction algorithm $\A$ implies the NP-hardness and
	inapproximability of problem $\Pi_2$.

	For the NP-hardness, because $\Pi_1$ is well-known to be NP-hard
	\citep{calinescu2007maximizing,feige1998threshold}, so $\Pi_2$ is
	NP-hard as well.

	For the inapproximability: Assume there exists a polynomial
	algorithm ${\mathscr B}$ that can solve $\Pi_2$ better than $1-1/e$,
	then we can use ${\mathscr B}$ as the subroutine algorithm in the
	reduction, which implies that one can solve $\Pi_1$ better than
	$1 - 1/e$. Now we slightly adapt the proof of inapproximability on
	max-k-cover of \citet{feige1998threshold}, since max-k-cover is a
	special case of $\Pi_1$.  According to the proof of  Theorem 5.3 in
	\citet{feige1998threshold} and our reduction $\A$, we have a
	reduction from approximating 3SAT--5 to problem $\Pi_2$. Using the
	rest proof of Theorem 5.3 in \citet{feige1998threshold}, we reach
	the result that one cannot solve $\Pi_2$ better than $1 - 1/e$,
	unless RP = NP.
\end{proof}

\subsection{Proof of \cref{coro_nonconvex_fw}}

\begin{proof}[Proof of \cref{coro_nonconvex_fw}]
	Firstly, according to Theorem  1 of \citet{lacoste2016convergence},
	\nonconvexfw is known to converge to a  stationary point with a rate of
	$1/\sqrt{k}$.

	Then according to \cref{coro_1half}, any stationary point is a
	1/2 approximate solution.
\end{proof}

\subsection{Proof of \cref{lemma_31}}
\begin{proof}
	It is easy to see that $\x^K$ is a convex combination of points in
	$\P$, so $\x^K\in\P$.

	Consider the point
	$\v^*:=(\x^*\vee \x) - \x = (\x^* - \x)\vee \zero\geqco \zero$.  Because
	$\v^*\leqco \x^*$ and $\P$ is down-closed, we get $\v^*\in \P$.

	By monotonicity, $f(\x+\v^*) = f(\x^*\vee \x) \geq f(\x^*)$.

	Consider the function $g(\xi):= f(\x+\xi \v^*), \xi\geq 0$.
	$\frac{d g(\xi)}{d \xi} = \dtp{\v^*}{\nabla f(\x+\xi \v^*)}$.  From
	Proposition \ref{prop_concave}, $g(\xi)$ is concave, hence
	\begin{flalign}
	g(1) - g(0) = f(\x+\v^*) - f(\x) \leq \frac{d g(\xi)}{d \xi}
	\Bigr|_{\xi = 0} \times 1 = \dtp{\v^*}{ \nabla f(\x)}.
	\end{flalign}
	Then one can get
	\begin{flalign}
	&\dtp{\v}{\nabla f(\x)} \overset{(a)}{\geq} \alpha \dtp{\v^*}{
		\nabla f(\x)} -\frac{1}{2}\delta \gamma L D^2 \geq \\
	&\alpha (f(\x+\v^*) - f(\x)) -\frac{1}{2}\delta \gamma L D^2 \geq
	\alpha (f(\x^*) -f(\x)) -\frac{1}{2}\delta \gamma L D^2,
	\end{flalign}
	where $(a)$ is resulted from the LMO step of \cref{alg_sfmax_GradientAscend}.
\end{proof}

\subsection{Proof of \cref{thm_fw}}
\begin{proof}[Proof of \cref{thm_fw}]
	From the Lipschitz assumption of $f$ (\cref{eq_smooth}):
	\begin{flalign}
	f(\x^{k+1}) - f(\x^k) & = f(\x^k + \gamma_k \v^k) - f(\x^k)
	\\\notag
	&\geq \gamma_k \dtp{\v^k}{\nabla f(\x^k)} -
	\frac{\cg}{2}\gamma_k^2 \|\v^k\|^2 \quad (\text{Lipschitz
		smoothness}) \\\notag &\geq \gamma_k \alpha [f(\x^*) - f(\x^k)]
	- \frac{1}{2}\gamma_k^2 \delta LD^2 - \frac{\cg}{2}\gamma_k^2 D^2.
	\quad (\text{Lemma \ref{lemma_31}})
	\end{flalign}
	After rearrangement,
	\begin{flalign}
	f(\x^{k+1}) - f(\x^*) \geq (1-\alpha\gamma_k) [f(\x^k) - f(\x^*)]-
	\frac{LD^2\gamma_k^2 (1+\delta)}{2}.
	\end{flalign}
	Therefore,
	\begin{flalign}
	f(\x^K) - f(\x^*) \geq \prod_{k=0}^{K-1}
	(1-\alpha\gamma_k)[f(\zero) - f(\x^*)] - \frac{LD^2 (1+\delta)}{2}
	\sum_{k=0}^{K-1}\gamma_k^2 .
	\end{flalign}
	One can observe that $\sum_{k=0}^{K-1}\gamma_k = 1$, and since
	$1-y \leq e^{-y}$ when $y\geq 0$,
	\begin{flalign}
	f(\x^*) - f(\x^K) &\leq [f(\x^*) - f(\zero)]e^{-\alpha
		\sum_{k=0}^{K-1}\gamma_k} + \frac{LD^2 (1+\delta)}{2}
	\sum_{k=0}^{K-1}\gamma_k^2 \\
	& = [f(\x^*) -
	f(\zero)]e^{-\alpha} +\frac{LD^2 (1+\delta)}{2}
	\sum_{k=0}^{K-1}\gamma_k^2.
	\end{flalign}
	After rearrangement, we get,
	\begin{align}
	f(\x^K) \geq (1-1/e^{\alpha})f(\x^*) -\frac{LD^2 (1+\delta)}{2}
	\sum_{k=0}^{K-1}\gamma_k^2 + e^{-\alpha}f(\zero).
	\end{align}
\end{proof}

\subsection{Proof of \cref{cor_9}}\label{app_proof_c9}

\begin{proof}[Proof of \cref{cor_9}]
	Fixing $K$, to reach the tightest bound in \cref{eq8} amounts to
	solving the following problem:
	\begin{flalign}
	&\min \sum_{k=0}^{K-1}\gamma_k^2\\\notag &\text{ s.t. }
	\sum_{k=0}^{K-1}\gamma_k = 1, \gamma_k \geq 0.
	\end{flalign}
	Using Lagrangian method, let $\lambda$ be the Lagrangian multiplier,
	then
	\begin{align}
	L(\gamma_0, \cdots, \gamma_{K-1}, \lambda) =
	\sum_{k=0}^{K-1}\gamma_k^2 + \lambda \left[\sum_{k=0}^{K-1}\gamma_k
	- 1\right].
	\end{align}
	It can be easily verified that when
	$\gamma_0 = \cdots =\gamma_{K-1} = K^{-1}$,
	$\sum_{k=0}^{K-1}\gamma_k^2$ reaches the minimum (which is
	$K^{-1}$). Therefore we obtain the tightest worst-case bound in
	Corollary \ref{cor_9}.
\end{proof}

\section{Details of Revenue Maximization with Continuous Assignments}
\label{supp_revenue}

\subsection{More Details About the Model}
As discussed in the main text, $R_s(\x^i)$ should be some
non-negative, non-decreasing, submodular function; therefore, we set
$R_s(\x^i) := \allowbreak \sqrt{\sum_{t: x^i_t \neq 0}x^i_t w_{st}}$,
where $w_{st}$ is the weight of edge connecting users $s$ and $t$.
The first part in R.H.S. of \cref{eq_re} models the revenue from users
who have not received free assignments, while the second and third
parts model the revenue from users who have gotten the free
assignments.
We use $w_{tt}$ to denote the ``self-activation rate" of user $t$:
Given certain amount of free trail to user $t$, how probable is it
that he/she will buy after the trial.  The intuition of modeling the
second part in R.H.S. of \cref{eq_re} is: Given the users more free
assignments, they are more likely to buy the product after using it.
Therefore, we model the expected revenue in this part by
$\phi(x^i_t) = w_{tt}x^i_t$; The intuition of modeling the third part
in R.H.S. of \cref{eq_re} is: Giving the users more free assignments,
the revenue could decrease, since the users use the product for free
for a longer period.  As a simple example, the decrease in the revenue
can be modeled as $\gamma \sum_{t:x^i_t\neq 0} -x^i_t$.

\subsection{Proof of Lemma \ref{revenue}}

\begin{proof}

	First of all, we prove that $g(\x) : = \sum_{s: x_s =0} R_s(\x)$
	is a non-negative submodular function.

	It is easy to see that $g(\x)$ is non-negative.
	To prove that $g(\x)$ is submodular, one just need,
	\begin{flalign}\label{eq_f}
	g(\a) + g(\b) \geq g(\a\vee \b) + g(\a\wedge \b), \quad  \forall \a, \b \in [\zero, \bar \bu].
	\end{flalign}
	Let $A:= \spt{\a}, B := \spt{\b}$, where $\spt{\x}:=\{i|x_i\neq 0 \}$ is  the  support of the vector $\x$.
	First of all, because $R_s(\x)$ is non-decreasing,  and $\b\geqco \a\wedge \b$, $\a\geqco \a\wedge \b$,

	\begin{flalign}\label{eq_1}
	\sum_{s\in A\backslash B} R_s(\b) + \sum_{s\in B\backslash A} R_s(\a) \geq \sum_{s\in A\backslash B} R_s(\a\wedge \b)  + \sum_{s\in B\backslash A} R_s(\a\wedge \b).
	\end{flalign}
	By submodularity of $R_s(\x)$, and  summing over $s\in \groundset \backslash(A\cup B)$,
	\begin{flalign}\label{eq_2}
	\sum_{s\in \groundset \backslash(A\cup B)}R_s(\a) + \sum_{s\in \groundset \backslash(A\cup B)}R_s(\b) \geq \sum_{s\in \groundset \backslash(A\cup B)}R_s(\a\vee \b) + \sum_{s\in \groundset \backslash(A\cup B)}R_s(\a\wedge \b).
	\end{flalign}
	Summing Equations  \ref{eq_1} and \ref{eq_2} one can get
	\begin{flalign}\notag
	\sum_{s\in \groundset \backslash A}R_s(\a) + \sum_{s\in \groundset \backslash B}R_s(\b) \geq \sum_{s\in \groundset \backslash(A\cup B)}R_s(\a\vee \b) + \sum_{s\in \groundset \backslash(A\cap  B)}R_s(\a\wedge \b)
	\end{flalign}
	which is equivalent to  \cref{eq_f}.

	Then we prove that $h(\x):=\sum_{t: x_t \neq 0} \bar R_t(\x)$ is submodular.
	Because $\bar R_t(\x)$ is non-increasing, and $\a\leqco \a\vee \b$,
	$\b \leqco \a\vee \b$,
	\begin{flalign}\label{37}
	\sum_{t\in A\backslash B} \bar R_t(\a) + \sum_{t\in B\backslash A} \bar R_t(\b) \geq \sum_{t\in A\backslash B} \bar R_t(\a\vee \b) + \sum_{t\in B\backslash A} \bar R_t(\a\vee \b).
	\end{flalign}
	By submodularity of $\bar R_t(\x)$, and summing over $t\in A\cap  B$,
	\begin{flalign}\label{38}
	\sum_{t\in A\cap  B} \bar R_t(\a) + \sum_{t\in A\cap  B} \bar R_t(\b) \geq \sum_{t\in A\cap  B} \bar R_t(\a\vee \b) + \sum_{t\in A\cap  B} \bar R_t(\a\wedge \b).
	\end{flalign}
	Summing Equations \ref{37}, \ref{38} we get,
	\begin{flalign}
	\sum_{t\in A} \bar R_t(\a) + \sum_{t\in  B} \bar R_t(\b) \geq \sum_{t\in A\cup  B} \bar R_t(\a\vee \b) + \sum_{t\in A\cap  B} \bar R_t(\a\wedge \b)
	\end{flalign}
	which is equivalent to $h(\a)+h(\b)\geq h(\a\vee \b)+h(\a\wedge \b)$, $\forall \a, \b \in [\zero, \bar \bu]$, thus proving the submodularity of
	$h(\x)$.

	Finally, because $f(\x)$ is the sum of two submodular functions and one
	modular function, so it is submodular.
\end{proof}

\section{Proofs for  Non-Monotone DR-Submodular Maximization}
\label{proofs_nonmonotone_max}

\subsection{Proof for Hardness and Inapproximability}

\begin{proof}[Proof of \cref{obs_dr_submodular_max}]
	The main proof  follows from the  reduction from the problem of
	maximizing an unconstrained non-monotone submodular set function.

	Let us denote $\Pi_1$ as the problem of  maximizing an unconstrained non-monotone submodular set function, and $\Pi_2$ as
	the problem  of maximizing a box constrained non-monotone   continuous DR-submodular  function.
	Following the Appendix A of  \cite{buchbinder2012tight}, there exist an  algorithm $\A$ for $\Pi_1$
	that consists of a polynomial time computation in addition to
	polynomial number of subroutine calls to an algorithm for $\Pi_2$. For details
	see the following.

	Given a submodular set  function $F: 2^{\groundset}\rightarrow \R_+$, its
	multilinear extension \citep{calinescu2007maximizing}
	is a function $f: [0,1]^\groundset \rightarrow \R_+$, whose value
	at a point $\x\in [0,1]^\groundset$ is the expected value of $F$ over
	a random subset $R(\x)\subseteq \groundset$, where $R(\x)$ contains
	each element $e\in \groundset$ independently with probability $x_e$.
	Formally, $f(\x):= \mathbb{E} [R(\x)] = \sum_{S\subseteq \groundset} F(S) \prod_{e\in S} x_e\prod_{e'\notin S}(1-x_{e'})$.
	It can be easily seen that $f(\x)$ is a non-monotone
	DR-submodular  function.

	Then the algorithm $\A$ can be: 1) Maximize the multilinear extension
	$f(\x)$ over the box constraint $[0, 1]^\groundset$, which can be achieved by
	solving an instance of $\Pi_2$. Obtain the fractional solution $\hat \x\in [0, 1]^n$; 2) Return the random set $R(\hat \x)$. According to the definition
	of multilinear extension, the expected value of $F(R(\hat \x))$ is
	$f(\hat \x)$.  Thus proving the reduction from $\Pi_1$ to $\Pi_2$.

	Given the reduction, the hardness result follows from the hardness
	of unconstrained non-monotone submodular set function maximization.

	The inapproximability result comes from that of the unconstrained non-monotone submodular set function maximization in \citet{feige2011maximizing}  and \citet{dobzinski2012query}.
\end{proof}

\subsection{Proof of \cref{rate_local_fw}}

\begin{proof}[Proof of \cref{rate_local_fw}]

	Let $g_{\P}(\x), g_{\Q}(\z)$ to the non-stationarity of $\x$ and
	$\z$, respectively. Since we are using
	the \nonconvexfw (\cref{nonconvex_fw}) as
	subroutine, according to \citet[Theorem 1]{lacoste2016convergence}, one can  get,
	\begin{align}
	&	g_{\P}(\x) \leq \min\left\{\frac{\max \{2h_1, C_f(\P)\}}{\sqrt{K_1+1}} , \epsilon_1\right\}, \\
	&	 g_{\Q}(\z) \leq  \min\left\{\frac{\max \{2h_2, C_f(\Q)\}}{\sqrt{K_2+1}} , \epsilon_2 \right\}.
	\end{align}
	Plugging the above into \cref{local_global} we reach the  conclusion in \labelcref{eq_local_rates}.
\end{proof}

\subsection{Detailed Proofs for \cref{thm-e}}

\subsubsection{Proof of \cref{prop_non_fw}}

\restalemmatwo*

\begin{proof}[Proof of \cref{prop_non_fw}]
	We prove  by induction.
	First of all, it holds when $k=0$, since $x_i^\pare{0}=0$,
	and $t^\pare{0}=0$ as well.
	Assume it holds for $k$. Then for $k+1$, we have
	\begin{align}
	x_i^\pare{k+1} & = x_i^\pare{k} + \gamma v_i^\pare{k}\\
	& \leq x_i^\pare{k} + \gamma ({\bar u_i} - x_i^\pare{k}) \quad \text{(constraint of shrunken LMO)}\\\notag
	& = (1-\gamma) x_i^\pare{k} + \gamma {\bar u_i}\\
	& \leq (1-\gamma){\bar u_i}[1- (1-\gamma)^{t^\pare{k}/\gamma} ]+ \gamma {\bar u_i} \quad \text{ (induction) } \\\notag
	& =  \bar u_i [1- (1-\gamma)^{t^\pare{k+1}/\gamma}].
	\end{align}
\end{proof}

\subsubsection{Proof of \cref{lem_nonmonotone_fw}}

\restalemmathree*

\begin{proof}[Proof of \cref{lem_nonmonotone_fw}]

	Consider $r(\lambda)= \x^* + \lambda(\x\vee \x^* - \x^*)$, it is easy to
	see that $r(\lambda)\geqco 0, \forall \lambda \geq 0$.

	Notice that $\lambda'\geq 1$.
	Let $\y = \r(\lambda') =  \x^* + \lambda'(\x\vee \x^* - \x^*)$, it is easy to see that $\y \geqco 0$, it also holds that $\y\leqco \bar u$: Consider one coordinate $i$, 1) if $x_i\geq x_i^*$, then $y_i = x_i^* + \lambda'(x_i - x_i^*)\leq \lambda'x_i \leq \lambda'\theta_i \leq \bar u_i$; 2)  if $x_i< x_i^*$, then $y_i = x_i^* \leq \bar u_i$. So $f(\y) \geq 0$.

	Note that
	\begin{align}
	\x\vee \x^* = (1-\frac{1}{\lambda'})\x^* + \frac{1}{\lambda'}\y = (1-\frac{1}{\lambda'})r(0) + \frac{1}{\lambda'}r(\lambda'),
	\end{align}

	since $f$ is concave along $r(\lambda)$, so it holds that,
	\begin{align}
	f(\x\vee \x^*) \geq  (1-\frac{1}{\lambda'})f(\x^*) +  \frac{1}{\lambda'}f(\y) \geq (1-\frac{1}{\lambda'})f(\x^*).
	\end{align}
\end{proof}

\subsubsection{Proof of \cref{thm-e}}
\label{app_subsec_thm2_proof}

\begin{proof}[Proof of \cref{thm-e}]

	First of all, let us prove the \namecref{claim3_1}:

	\restaclaimthree*

	\begin{proof}[Proof of \cref{claim3_1}]
		Consider a point
		$\z^\pare{k}:= \x^\pare{k}\vee \x^* - \x^\pare{k}$, one can
		observe that: 1) $\z^\pare{k}\leqco \bar \u -\x^\pare{k}$; 2)
		since $\x^\pare{k}\geqco \zero, \x^*\geqco \zero$, so
		$\z^\pare{k}\leqco \x^*$, which implies that $\z^\pare{k}\in \P$
		(from down-closedness of $\P$). So $\z^\pare{k}$ is a candidate
		solution for the shrunken LMO (Step \labelcref{new_lmo} in \cref{fw-non-monotone}). We have,
		\begin{flalign}
		f(\x^{\pare{k+1}}) - f(\x^{\pare{k}}) & \geq \gamma\dtp{\nabla
			f(\x^\pare{k})}{\v^\pare{k}} - \frac{L}{2}\gamma^2
		\|\v^\pare{k}\|^2  (\text{Quadratic lower bound
			of \labelcref{eq_quad_lower_bound}}) \\ & \geq
		\gamma\dtp{\nabla f(\x^\pare{k})}{\v^\pare{k}} -
		\frac{L}{2}\gamma^2 D^2 \quad (\text{diameter of } \P) \\
		& \geq \gamma \dtp{\nabla f(\x^\pare{k})}{\z^\pare{k}} -
		\frac{L}{2}\gamma^2 D^2\quad (\text{shrunken LMO})\\  &
		\geq \gamma(f(\x^\pare{k}+\z^\pare{k}) - f(\x^\pare{k})) -
		\frac{L}{2}\gamma^2 D^2 \quad (\text{concave along
			$\z^\pare{k}$})\\ & = \gamma [f(\x^\pare{k}\vee \x^*) -
		f(\x^\pare{k})] - \frac{L}{2}\gamma^2 D^2\\ & \geq \gamma
		[(1-\frac{1}{\lambda'})f(\x^*) - f(\x^\pare{k})] -
		\frac{L}{2}\gamma^2 D^2 \quad (\text{\cref{lem_nonmonotone_fw}})
		\\ & =\gamma [ (1-\gamma)^{t^\pare{k}/\gamma} f(\x^*) -
		f(\x^\pare{k})] - \frac{L}{2}\gamma^2 D^2,
		\end{flalign}
		where the last equality comes from {setting }
		$\bmtheta : = \bar \u(1-(1-\gamma)^{t^\pare{k}/\gamma})$ {
			according to \cref{prop_non_fw}}, thus
		$\lambda' = \min_i \frac{\bar u_i}{\theta_i} =
		(1-(1-\gamma)^{t^\pare{k}/\gamma})^{-1}$.

		After rearrangement, we reach the claim.
	\end{proof}
	Then, let us prove \cref{thm-e} by \emph{induction}.

	First of all, it holds when $k = 0$ (notice that
	$t^\pare{0}=0$). Assume that it holds for $k$.

	Then for $k+1$,
	considering the fact $e^{-t} - O(\gamma)\leq (1-\gamma)^{t/\gamma}$
	when $0< \gamma\leq t \leq 1$ and \cref{claim3_1} we get,
	\begin{align}
	& f(\x^{\pare{k+1}})\\
	&  \geq (1-\gamma)  f(\x^{\pare{k}})   +
	\gamma(1-\gamma)^{t^\pare{k}/\gamma} f(\x^*) -\frac{L
		D^2}{2}\gamma^2\\
	& \geq  (1-\gamma)  f(\x^{\pare{k}})   + \gamma [e^{-t^\pare{k}} -
	O(\gamma)] f(\x^*) -\frac{L D^2}{2}\gamma^2\\\notag
	& \geq  (1-\gamma) [ t^\pare{k} e^{-t^\pare{k}}f(\x^*) - \frac{L
		D^2}{2}k\gamma^2 - O(\gamma^2)f(\x^*)]+ \gamma [e^{-t^\pare{k}}
	- O(\gamma)] f(\x^*) -\frac{L D^2}{2}\gamma^2\\\notag
	& = [(1-\gamma) t^\pare{k} e^{-t^\pare{k}} + \gamma
	e^{-t^\pare{k}}   ]f(\x^*)  - \frac{L D^2}{2}\gamma^2
	[(1-\gamma)k + 1] - [(1-\gamma) O(\gamma^2) + \gamma
	O(\gamma)]f(\x^*)\\\label{eq_30}
	& \geq  [(1-\gamma) t^\pare{k} e^{-t^\pare{k}} + \gamma
	e^{-t^\pare{k}}   ]f(\x^*) -  \frac{L D^2}{2}\gamma^2(k+1) -
	O(\gamma^2)f(\x^*).
	\end{align}
	Let us consider the term
	$ [(1-\gamma) t^\pare{k} e^{-t^\pare{k}} + \gamma e^{-t^\pare{k}}
	]f(\x^*)$.
	We know that the function $g(t) = te^{-t}$ is concave in $[0, 2]$,
	so
	$g(t^\pare{k}+\gamma) - g(t^\pare{k}) \leq \gamma g'(t^\pare{k})$,
	which amounts to,
	\begin{align}
	[(1-\gamma) t^\pare{k} e^{-t^\pare{k}} + \gamma e^{-t^\pare{k}}
	]f(\x^*) & \geq (t^\pare{k} +\gamma) e^{-(t^\pare{k} +\gamma)}
	f(\x^*)\\\label{eq_34}
	&= t^{\pare{k+1}} e^{-t^{\pare{k+1}}} f(\x^*).
	\end{align}
	Plugging \cref{eq_34} into \cref{eq_30} we get,
	\begin{align}
	f(\x^{\pare{k+1}})    \geq t^{\pare{k+1}} e^{-t^{\pare{k+1}}}
	f(\x^*) -  \frac{L D^2}{2}\gamma^2(k+1) - O(\gamma^2)f(\x^*).
	\end{align}
	Thus proving the induction, and proving the theorem as well.
\end{proof}

\section{Miscellaneous Results}

\subsection{Verifying DR-Submodularity of the Objectives}\label{appe_dr_soft}

\paragraph{Softmax extension.}
For softmax extension, the objective is,
\begin{flalign}\notag
f(\x) = \log\de{\diag(\x)(\bmL-\bmI) +\bmI }, \x\in [0,1]^n.
\end{flalign}
Its DR-submodularity can be established by directly applying
Lemma 3 in \citep{gillenwater2012near}:  \citet[Lemma 3]{gillenwater2012near} immediately implies
that all  entries of  $\nabla^2  f$ are non-positive, so $f(\x)$
is DR-submodular.

\paragraph{Multilinear extension.}
The DR-submodularity of  multilinear extension can be directly
recognized by considering the conclusion in Appendix A.2 of \citet{bach2015submodular}
and the fact that multilinear extension is coordinate-wise linear.

\paragraph{$\text{KL}(\x)$.}
The Kullback-Leibler divergence between  $q_{\x}$ and $p$, i.e., $ \sum_{S\subseteq \groundset} q_{\x}(S)
\log\frac{q_{\x}(S)}{p(S)}$ is,
\begin{align}\notag
\text{KL}(\x) =
-\sum_{S\subseteq \groundset}\prod_{i\in S}x_i \prod_{j\notin S}(1-x_j) F(S) + \sum\nolimits_{i=1}^{n} [x_i\log x_i + (1-x_i)\log(1-x_i)] + \log Z.
\end{align}
The first term is the negative of a multilinear extension, so it is DR-supermodular. The second term
is separable, and coordinate-wise convex, so it will not
affect the off-diagonal entries of $\nabla^2 \text{KL}(\x)$,
it will only contribute to  the diagonal entries.
Now, one can see that all entries of $\nabla^2 \text{KL}(\x)$  are non-negative, so $\text{KL}(\x)$ is DR-supermodular w.r.t. $\x$.

\end{document}